\documentclass[11pt]{article}
\usepackage[letterpaper,margin=1in]{geometry}
\usepackage[T1]{fontenc}
\usepackage{amsmath,amsthm,mathtools}
\IfFileExists{newtxmath.sty}{\usepackage{newtxtext,newtxmath}}{\usepackage{txfonts}}
\usepackage{mathrsfs,microtype}
\usepackage{booktabs,array,tabularx,enumitem}
\usepackage{algorithm,placeins,needspace}
\usepackage[noend]{algpseudocode}
\usepackage[dvipsnames]{xcolor}
\usepackage{tikz}
\usepackage[titles]{tocloft}
\usetikzlibrary{arrows.meta,positioning,calc,shapes.geometric}
\usepackage[colorlinks=true,linkcolor=MidnightBlue,citecolor=MidnightBlue,urlcolor=MidnightBlue]{hyperref}
\hypersetup{pdftitle={Square-Root Regret for Adversarial Multiplayer Bandits without Collision Information or Shared Randomness},pdfauthor={Chenyu Gan},linktoc=all}
\setlist{topsep=4pt,itemsep=2pt,leftmargin=*}
\newtheorem{theorem}{Theorem}
\newtheorem{proposition}[theorem]{Proposition}
\newtheorem{lemma}[theorem]{Lemma}
\newtheorem*{informaltheorem}{Theorem (informal)}
\newtheorem{ilemma}{Lemma}

\newtheorem{iprop}[ilemma]{Proposition}
\theoremstyle{definition}
\theoremstyle{remark}
\newcommand{\E}{\mathbb E}
\newcommand{\ind}{\mathbf 1}
\newcommand{\cS}{\mathcal S}
\DeclareMathOperator{\Unif}{Unif}

\title{\bfseries Square-Root Regret for Adversarial Multiplayer Bandits without Collision Information or Shared Randomness\footnote{A Lean formalization of the mathematical results of this paper, excluding computational complexity claims, is available at \href{https://github.com/ChenyuGan/SquareRootBandits}{https://github.com/ChenyuGan/SquareRootBandits}.}}

\author{Chenyu Gan\\
Tsinghua University}
\date{}

\algrenewcommand\alglinenumber[1]{\normalsize #1:}

\begin{document}
\maketitle
\begin{abstract}
\normalsize
We study cooperative multiplayer bandits with adversarial rewards and no collision information. In each of $T$ rounds, each of $m$ labeled players picks one of $K>m$ arms. A player alone on its arm receives that arm's reward in $[0,1]$; players who share an arm receive~$0$. Each player observes only its own reward, so it cannot distinguish a collision from an arm with zero reward, and there is neither shared randomness nor a communication channel. Bubeck, Li, Peres and Sellke~\cite{bubeck2020} obtained regret $\widetilde O(KT^{3/4})$ for two players and $\widetilde O(mK^{3/2}T^{1-1/(2m)})$ for $m$ players, and conjectured that regret $\Omega(T^{2/3})$ is unavoidable already for two players.

We refute this conjecture. We give a strategy whose expected regret against every oblivious adversary is $O(m^{3/2}\sqrt{KT}\log^2(KmT))$, which is optimal in $K$ and $T$ up to logarithmic factors. Three ideas make this possible. (i)~Bits are sent through deliberate collisions as random parity checks, and the regret spent while a message is delayed is charged to the number of arms with positive reward rather than to elapsed time: an adversary that jams the channel with zero rewards forfeits the regret it could inflict. (ii)~An exponential-weights learner over sets of arms combines a frequently resampled coordinator arm with a follower assignment that changes only $O(\sqrt{mT/K})$ times in expectation; delayed updates are reconstructed exactly. (iii)~We teach the followers a privately chosen row index in a fixed public table using reward observations alone, allowing them to coordinate without shared randomness. The bound holds simultaneously for all reward matrices, even ones chosen with knowledge of the table, and a random table works with probability $1-O(1/(KmT))$.
\end{abstract}
\clearpage
\begingroup
\small
\setlength{\parskip}{0pt}
\tableofcontents
\endgroup
\clearpage
\section{Introduction}
\label{main:introduction}

Consider $m$ devices that share $K$ wireless channels. In each time slot every device transmits on one channel, and if two devices choose the same channel, both transmissions fail. Channel qualities change arbitrarily over time, and a device learns only how much it transmitted itself. In particular, a device that transmitted nothing cannot tell whether its channel was bad or whether another device was on it. Without communication or a common source of randomness, can the devices learn to share the channels almost as well as a central planner who knows all channel qualities in hindsight? The same question arises when devices offload jobs to shared servers, and in other decentralized allocation problems~\cite{liu2010,bande2019adversarial,wang2022shareable,busacca2024madfellows,cilic2026qos,liu2024qos,liu2020matching}.

This is the multiplayer bandit problem without collision information~\cite{liu2010,boursier2024survey}. Formally (Section~\ref{main:model}), $m$ labeled players and $K>m$ arms interact for $T$ rounds, and an oblivious adversary fixes rewards $r_t(a)\in[0,1]$. In round~$t$ each player picks an arm. A player alone on arm~$a$ receives and observes $r_t(a)$; players sharing an arm receive and observe~$0$. The regret $R_T$ is the expected shortfall of the players' total reward relative to the best fixed set of $m$ distinct arms. A central learner who picks all $m$ arms and sees their rewards has minimax regret $\Theta(\min\{mT,\sqrt{mKT}\})$ when $K\ge2m$~\cite{audibert2014}. How much does decentralization cost?

The difficulty is that collisions are the players' only means of coordination, and a collision looks exactly like a zero reward. Players can signal by colliding on purpose, but the adversary can keep the reward of any arm at zero for as long as it likes, so no signal is guaranteed to arrive within any fixed number of rounds. Bubeck, Li, Peres and Sellke~\cite[Theorems~2 and~21]{bubeck2020} nevertheless achieved sublinear regret: $\widetilde O(KT^{3/4})$ for two players and $\widetilde O(mK^{3/2}T^{1-1/(2m)})$ for $m$ players. They observed that without collision information it is not clear how to communicate implicitly, which made a $\sqrt T$-regret strategy seem impossible~\cite[Section~2.3]{bubeck2020}. They conjectured that with neither collision information nor shared randomness, regret $\Omega(T^{2/3})$ is necessary already for two players, and that the exponent might degrade further as $m$ grows~\cite[Section~2.5]{bubeck2020}. We show that decentralization costs no additional power of $T$ or~$K$.

\begin{informaltheorem}
For all $2\le m<K$ and $T\ge1$ there is a strategy for $m$ players, using no collision information, no shared randomness and no communication channel, whose expected regret against every oblivious adversary is
\[
 O\big(m^{3/2}\sqrt{KT}\,\log^2(KmT)\big).
\]
\end{informaltheorem}

This refutes the conjecture and improves the exponent of $T$ from $1-\frac1{2m}$ to $\frac12$. The bound is optimal in $K$ and $T$ up to logarithmic factors and exceeds the centralized rate by a factor $m$, up to logarithms. A strategy may ignore collision information, so the same bound holds when collisions are announced. In that model it improves the dependence on $K$ of the two-player bound $O(K^2\sqrt{T\log K\log T})$ of~\cite[Theorem~1]{bubeck2020}, and for $m>2$ it gives, to our knowledge, the first $\sqrt T$-type bound with polynomial dependence on~$m$; \cite[Section~2.5]{bubeck2020} left open whether that dependence is polynomial or exponential. Table~\ref{tab:landscape} summarizes the adversarial landscape.

\begin{table}[t]
\centering\small
\begin{tabularx}{\linewidth}{@{}>{\raggedright\arraybackslash}p{4.6cm}cc>{\raggedright\arraybackslash}X@{}}
\toprule
Setting (adversarial rewards) & Players & Shared rand. & Regret\\
\midrule
Centralized semi-bandit (one learner picks $m$ arms) & --- & --- & $\Theta(\min\{mT,\sqrt{mKT}\})$~\cite{audibert2014}\\
Collision info, \emph{adaptive} adversary & $2$ ($K=3$) & no & $\Omega(T)$ for every strategy~\cite[Prop.~3]{bubeck2020}\\
Collision info, oblivious & $m$ & no & $O_{m,K}(T^{3/4})$~\cite{bande2019adversarial}; $\widetilde O(m^{4/3}K^{2/3}T^{2/3})$~\cite{alatur2020}\\
Collision info, oblivious & $2$ & no & $O(K^2\sqrt{T\log K\log T})$~\cite[Thm.~1]{bubeck2020}\\
No collision info, oblivious & $2$ & no & $\widetilde O(KT^{3/4})$~\cite[Thm.~2]{bubeck2020}; conjectured $\Omega(T^{2/3})$~\cite[Sec.~2.5]{bubeck2020}\\
No collision info, oblivious & $m$ & no & $\widetilde O(mK^{3/2}T^{1-1/(2m)})$~\cite[Thm.~21]{bubeck2020}\\
No collision info, bounded runs or counts of zero rewards & $m$ & no & sublinear under these assumptions~\cite{shi2021}\\
\textbf{No collision info (hence also with)} & $\boldsymbol m$ & \textbf{no} & $\boldsymbol{O(m^{3/2}\sqrt{KT}\log^2(KmT))}$ \textbf{(Theorem~\ref{sq:thm:main})}\\
\bottomrule
\end{tabularx}
\caption{Regret for adversarial multiplayer bandits with $m$ players and $K>m$ arms.}
\label{tab:landscape}
\end{table}

\subsection{Model and main result}\label{main:model}

Let $2\le m<K$ and $T\ge1$ be integers, and write $[q]=\{1,\ldots,q\}$. The players have distinct labels $0,\ldots,m-1$; player $0$ is the \emph{coordinator} and the others are \emph{followers}. The labels, $K,m,T$, the synchronous round clock and a public table, described below, are common knowledge. An oblivious adversary fixes rewards $r_t(a)\in[0,1]$ for $t\in[T]$ and $a\in[K]$. In round $t$, player $i$ chooses $A_{i,t}\in[K]$ and observes only
\begin{equation}
 Y_{i,t}=r_t(A_{i,t})
 \ind\{A_{j,t}\ne A_{i,t}\text{ for every }j\ne i\}.
 \label{main:feedback}
\end{equation}
Its action may depend on the common knowledge, its own previous actions and observations, and its private randomness. There is no other communication channel and no common random seed. For $x\in\mathbb R^K$ and $S\subseteq[K]$, put $x(S)=\sum_{a\in S}x_a$, and let $\cS=\{S\subseteq[K]:|S|=m\}$. Given the table, the regret against a reward matrix $r=(r_t(a))_{t\in[T],a\in[K]}$ is
\begin{equation}
 R_T(r)=\max_{S\in\cS}\sum_{t=1}^T r_t(S)
       -\E\sum_{t=1}^T\sum_{i=0}^{m-1}Y_{i,t},
 \label{main:regret}
\end{equation}
where the expectation is over the players' private randomness. The substitution $\ell_t(a)=1-r_t(a)$ gives the loss model of~\cite{bubeck2020}. Throughout, $N=2Km(T+1)$ and $b_N=\lceil\log_2(N+1)\rceil=\Theta(\log(KmT))$.

\Needspace{10\baselineskip}
\begin{theorem}[Main result]
\label{sq:thm:main}
There is an absolute constant $C$ such that, for all integers $2\le m<K$ and $T\ge1$, there exist a table and a strategy profile in model~\eqref{main:feedback} such that, for every reward matrix $r\in[0,1]^{T\times K}$, including matrices chosen with knowledge of the table,
\begin{equation}
 R_T(r)\le C\,m^{3/2}\sqrt{KT}\,b_N^2 .
 \label{sq:eq:main}
\end{equation}
Moreover, a table sampled from an explicit distribution has this property with probability at least $1-O(1/N)$. Sampling the table and executing the strategy take polynomially many operations in $K,m,T$, in the exact-arithmetic model of Section~\ref{sq:sec:complexity}.
\end{theorem}

\paragraph{Remarks.}
\begin{enumerate}[label=(\arabic*),leftmargin=*,itemsep=1pt]
\item \emph{The table is not shared randomness.} It is a fixed part of the strategy, and the adversary may know it. After the rewards are fixed, the coordinator privately draws a row index~$J$, and the players communicate $J$ through their arm choices; the selected row then acts as a common random schedule. This is the principle behind Newman's theorem~\cite{newman1991} and Ahlswede's elimination of correlation~\cite{ahlswede1978}. The collision-aware algorithm in~\cite[Appendix~B.4]{bubeck2020} removes shared randomness in the same way, with the short index sent through announced collisions. 
Without collision information, delivering the index is itself a central difficulty (Section~\ref{overview:public}). The table has $\widetilde O(T^{9/2})$ cells, mostly forward bits, and play takes $\widetilde O(T^3)$ operations in total (Proposition~\ref{sq:prop:complexity}). 
\item \emph{Obliviousness is necessary.} Against adaptive adversaries every strategy without shared randomness suffers linear regret, already for two players and three arms and even with collision information~\cite[Proposition~3]{bubeck2020}.
\item \emph{Conventions.} As in~\cite{bubeck2020}, players are labeled, rounds are synchronous, and $K,m,T$ are known. For $m=K$, player $i$ always choosing arm $i+1$ has zero regret; $m=1$ is the ordinary adversarial bandit problem.
\end{enumerate}

\subsection{Why square-root regret is possible}\label{main:why}

An adversary that blocks communication by zeroing rewards also cancels the regret it can cause. We therefore measure communication delays in units of positive reward rather than in rounds. Let
\[
 g_t=\#\{a\in[K]:r_t(a)>0\},\qquad \bar g_t=\min\{m,g_t\}
\]
be the number of arms with positive reward in round $t$, and its cap at~$m$.
\begin{itemize}[itemsep=1pt]
\item In a round in which the players communicate instead of exploiting, the regret is at most $\bar g_t$, since the best $m$-set earns no more.
\item Our channel makes progress on any awaited message in round $t$ with probability at least $g_t/(8mK)$.
\end{itemize}
Hence the expected regret spent on a message of $b$ bits is $O(mK(b+\log N))$, however the adversary schedules its zeros. It remains to keep the total amount of communication small. The players agree once on an $O(\log N)$-bit index into the public table. Afterwards they exchange $O(\log N)$-bit assignments and score reports only $O(\sqrt{mT/K})$ times in expectation, because a single \emph{fast} player explores in every round while the arms of the $m-1$ \emph{slow} players change rarely. With learning rate $\eta\approx1/\sqrt{mKT}$, learning and communication both cost $\widetilde O(m^{3/2}\sqrt{KT})$.

\paragraph{Why the $T^{2/3}$ intuition fails.}
In bandits with switching costs the minimax regret is $\widetilde\Theta(T^{2/3})$~\cite{dekel2014switching}: learning from bandit feedback forces frequent switches. A similar tension seems unavoidable here, because every change of the coordinated assignment must be communicated. It disappears because exploration never requires the slow players to move. The coordinator's fast arm is redrawn every round at no communication cost and includes uniform exploration. The slow set follows a lazy Markov chain whose marginal is the exponential-weights law in every round, so moving rarely loses nothing in expectation. This is the shrinking-dartboard idea of~\cite{geulen2010}, which~\cite{bubeck2020} used for two players with collision information. The new difficulty is that here every resampling must be carried out with information that arrives late, over a channel the adversary can jam.

\subsection{Related work}\label{main:related-work}

\paragraph{Adversarial rewards.}
Table~\ref{tab:landscape} lists the known bounds. Closest to our setting, Shi and Shen~\cite{shi2021} communicate through deliberate collisions protected by error-correcting codes; their regret bounds depend on the longest run, or the total number, of zero-reward rounds on each arm, and are sublinear when these are $O(T^\alpha)$ for some $\alpha<1$. Charging communication to reward support removes every such assumption.

\paragraph{Stochastic rewards.}
The centralized multiple-play problem goes back to~\cite{anantharam1987}. With collision information, logarithmic decentralized regret was established in~\cite{liu2010}, and musical chairs assign distinct good arms~\cite{rosenski2016}; reward-only feedback was studied empirically~\cite{bonnefoi2018,besson2018}. Without collision information or shared randomness, logarithmic or gap-free square-root regret is known under assumptions on the means, such as a positive gap or a known lower bound on the smallest mean~\cite{lugosi2022,boursier2019,shi2020,huang2022,pacchiano2023cooperative}; see also~\cite{hu2025shareable,izumi2026multichannel}. With shared randomness, square-root regret with few collisions is possible~\cite{bubeckbudzinski2020,bubeckbudzinskisellke2021}, and for two players and three arms a deterministic strategy achieves $O(\sqrt{T\log T})$ without collision information~\cite{bubeckbudzinski2020arxiv}. Heterogeneous means, delays and asynchronous participation are treated in~\cite{bistritz2021,boursier2020heterogeneous,fan2025delayed,arrouespaykin2026asynchronous}. All these guarantees exploit stationary rewards. 

\section{Overview of the strategy and of the proof}\label{overview:section}

We first describe the strategy assuming that all players know a common random schedule that is unknown to the adversary (Sections~\ref{overview:implementation}--\ref{overview:cycle}). We then explain how the players acquire such a schedule from a public table (Section~\ref{overview:public}) and why one fixed table works for all reward matrices (Section~\ref{overview:assembly}). Lemmas~\ref{inf:channel}, \ref{inf:cycle}, \ref{inf:table} and Propositions~\ref{inf:learner}, \ref{inf:boot} below are informal versions of the main results; each cites its formal counterparts.

\paragraph{Conventions and vocabulary.}
Put $d=m-1$ and $D=2^{b_N}$. Every observed reward is rounded down to a multiple of $1/D$, and from now on $r_t(a)$ denotes the rounded reward; rounding costs at most $mT/D\le1$ in regret. All logarithms without a subscript are natural; $x\lesssim y$ means $x\le Cy$ for an absolute constant~$C$, and $x\asymp y$ means $x\lesssim y\lesssim x$. A \emph{row} is a complete random schedule for the $T$ rounds; the public table consists of an identification array and $M$ independent rows. Within the ordinary protocol, a follower is \emph{unoccupied} if it has no decoded live assignment, \emph{pending} if it has one but has not yet been released, and \emph{released} if it may keep its assigned arm in every slot. An \emph{episode} fixes the learner's anchor and recording rates until the next grid boundary and may contain several \emph{assignments}; an assignment fixes the followers' arms until its \emph{cutoff}, after which \emph{collection} returns their score reports. Small horizons, $T\le C_bmKb_N^2$ or $N<N_0$, use a fixed assignment to $m$ distinct arms; we describe the \emph{construction branch} $T>C_bmKb_N^2$, $N\ge N_0$ (Section~\ref{sq:sec:boot}).

\subsection{A channel paid for in units of reward support}\label{overview:implementation}

Each round of the row has a slot type, DATA, DOWN or UP, with probabilities $\frac12,\frac14,\frac14$. DATA rounds feed the learner; on DOWN rounds the coordinator sends assignments, and on UP rounds the followers send score reports. On a DOWN or UP round the row selects a uniform follower $k_t\in[d]$ and a uniform tuple of distinct arms $(R_{0,t},\ldots,R_{d,t})$. A message is a word $v\in\mathbb F_2^B$ with $B=O(b_N)$, and the row supplies a fresh uniform affine map $H_t(v)=c_t\cdot v\oplus b_t$. The receiver plays its listening arm; the sender joins it if $H_t(v)=0$ and plays its own tuple arm if $H_t(v)=1$ (Figure~\ref{fig:parity}). A zero observation carries no information, but
\[
 \text{positive receiver observation}\quad\Longrightarrow\quad c_t\cdot v=1\oplus b_t .
\]
The receiver collects such equations in batches of $Q_{\rm dec}=B+O(b_N)$ and solves them; it accepts a unique solution that passes a local test of the assignment identifier. All other players stay off the receiver's arm, so an awaited word makes progress in round $t$ with probability
\[
 \underbrace{\tfrac14}_{\text{direction}}\cdot
 \underbrace{\tfrac1d}_{\text{follower}}\cdot
 \underbrace{\tfrac12}_{\text{parity}}\cdot
 \underbrace{\tfrac{g_t}K}_{\text{positive arm}}
 =\frac{g_t}{8dK},
\]
whereas the regret of that round is at most $\bar g_t\le g_t$.

\begin{ilemma}[Parity channel; Lemmas~\ref{sq:lem:decoder},~\ref{sq:lem:waiting},~\ref{sq:lem:handover} and~\ref{sq:lem:collection}]\label{inf:channel}\leavevmode
\begin{enumerate}[label=(\alph*),itemsep=1pt]
\item Except with probability $O(N^{-2})$, no receiver accepts a word other than the one currently sent to it, and an awaited word is accepted within $2Q_{\rm dec}$ further positive receiver observations.
\item For every reward sequence, the expected sum of $g_t$ until $d$ concurrently awaited words are all accepted is $O(dK(Q_{\rm dec}+\log d))=O(mKb_N)$.
\end{enumerate}
\end{ilemma}

Part~(b) uses an exponential supermartingale: serving $d$ followers in parallel costs only an additive $O(dK\log d)$ rather than a factor~$d$. It bounds consumed \emph{support}, not elapsed time: a message may wait arbitrarily long, but in expectation the comparator earns only $O(mKb_N)$ while it waits.

\begin{figure}[t]
\centering
\begin{tikzpicture}[font=\small,>=Latex,
  arm/.style={draw,minimum width=14mm,minimum height=7mm,fill=black!5},
  pl/.style={circle,draw,inner sep=1pt,minimum size=7mm,font=\footnotesize},
  capt/.style={align=center,font=\footnotesize,text width=4.4cm},
  arr/.style={->,semithick}]
\begin{scope}
\node[arm] (r1) at (0,0) {$R_{\rm rec}$};
\node[arm] (s1) at (2.3,0) {$R_{\rm send}$};
\node[pl,fill=Cerulean!20] (rec1) at (0,1.45) {rec};
\node[pl,fill=Goldenrod!35] (snd1) at (2.3,1.45) {snd};
\draw[arr] (rec1) -- (r1); \draw[arr] (snd1) -- (r1);
\node[font=\bfseries] at (-1.0,1.45) {(a)};
\node[capt] at (1.15,-1.05) {$H_t(v)=0$: collision;\\ the receiver sees $0$ (no information)};
\end{scope}
\begin{scope}[xshift=5.6cm]
\node[arm] (r2) at (0,0) {$R_{\rm rec}$};
\node[arm] (s2) at (2.3,0) {$R_{\rm send}$};
\node[pl,fill=Cerulean!20] (rec2) at (0,1.45) {rec};
\node[pl,fill=Goldenrod!35] (snd2) at (2.3,1.45) {snd};
\draw[arr] (rec2) -- (r2); \draw[arr] (snd2) -- (s2);
\node[font=\bfseries] at (-1.0,1.45) {(b)};
\node[capt] at (1.15,-1.05) {$H_t(v)=1$: the receiver is alone;\\ if $r_t(R_{\rm rec})>0$ it learns $c_t\cdot v=1\oplus b_t$};
\end{scope}
\node[draw,rounded corners=2pt,fill=black!3,align=left,text width=4.2cm,font=\footnotesize,anchor=west] at (9.9,0.45)
 {$B+O(\log N)$ positive observations determine $v$ by Gaussian elimination.\\[2pt]
  Per round: progress w.p.\ $g_t/(8dK)$; regret $\le\min\{m,g_t\}$.\\[2pt]
  $\Rightarrow$ expected cost $O(mK\log N)$ per message, whatever the zeros.};
\end{tikzpicture}
\caption{The parity channel and the support-charging argument. The other players avoid $R_{\rm rec}$ by construction: tuple arms are distinct, and assigned players step aside.}
\label{fig:parity}
\end{figure}
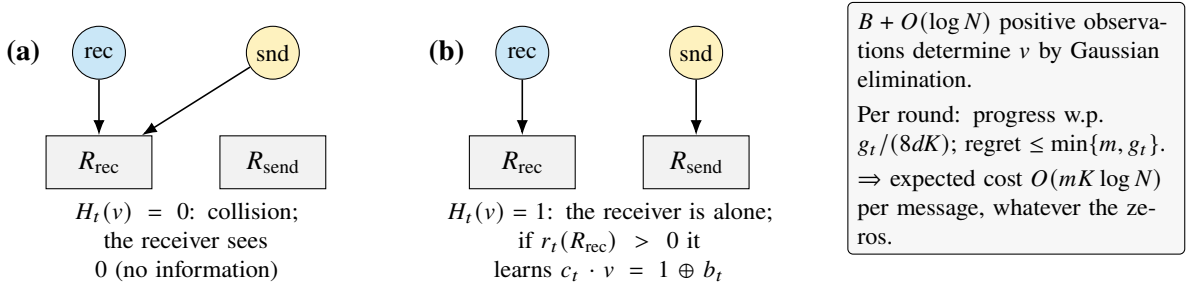

\subsection{A learner whose slow players rarely move}\label{overview:learner}

The \emph{logical learner} is an idealized process: it sees its estimates immediately and controls all players. The learner keeps integer scores $z_i$ and weights $w_i=(1+1/D)^{z_i}$ on arms, and the product law on $m$-sets
\[
 Q_w(S)=\frac{\prod_{i\in S}w_i}{\sum_{|S'|=m}\prod_{i\in S'}w_i},\qquad \theta_w(i)=\Pr_{S\sim Q_w}(i\in S).
\]
It splits $S\sim Q_w$ into a slow set $U$ of $d$ arms, played by the followers, and one fast arm, played by the coordinator: the fast arm is $j\in S$ with probability proportional to $\bar w_j^{-1}$, for an anchor $\bar w$ fixed at the episode start. Write $p_w$ for the law of $U$. The conditional law of the fast arm,
\[
 q_w(j\mid U)\propto\frac{w_j}{1+\bar w_j\sum_{i\in U}\bar w_i^{-1}}\qquad(j\notin U),
\]
does not involve the current weights of the slow arms, which the coordinator learns only later. A small uniform mixture $q_w^\gamma$ adds exploration.

\emph{Sparse recording.} In each round at most one observation is \emph{recorded}: arm $i$ with a prescribed probability $\xi_i$, conditional on everything before the round \emph{including the current slow set}, and no arm with probability $1-\sum_i\xi_i$. The row selects one recording source and an independent thinning test realizes these exact probabilities (Lemma~\ref{sq:lem:recorder}). The recorded reward is importance-weighted, so a rare recording can be large. Applying it at once would move the target distribution too much, so excess estimated reward waits in a queue that releases at most four units per arm and round. Then every weight changes by a factor at most $e^{4\eta}$ per round and stays within a factor~$2$ of the anchor during an episode of length $H\asymp1/\eta$.

\emph{Lazy slow set.} Write $p_t=p_{w_t}$ and $h=4m\eta$. Bounded updates give $p_{t+1}\ge(1-h)p_t$, so $\rho_t=(p_{t+1}-(1-h)p_t)/h$ is a probability law. After each update the learner keeps its slow set, except that with probability $h$ (a \emph{proposal}) it redraws it from~$\rho_t$. Since $(1-h)p_t+h\rho_t=p_{t+1}$, the slow set has law exactly $p_t$ in every round given the past estimates. Sparse recording is what makes this true: the recorded arm has the same law whatever the slow set is, so the estimates reveal nothing about it.

\begin{iprop}[Logical learner; Lemmas~\ref{sq:lem:queue-bound},~\ref{sq:lem:conditional-law} and Proposition~\ref{sq:prop:learning}]\label{inf:learner}
For every admissible learning rate $\eta\lesssim1/(mKb_N)$, the learner changes its slow set only at proposals, which occur with probability $4m\eta$ per round, and at episode boundaries, every $\asymp1/\eta$ rounds. On all enabled rounds, including those spent communicating, its expected regret is
\[
 O\biggl(\frac{\log\binom Km+m}\eta+(m^2K+Kb_N)\,\eta T+1\biggr).
\]
\end{iprop}

The term $m^2K$ comes from the variance of one recording per round; $Kb_N$ comes from the uniform mixture that keeps every arm's recording rate at least $\beta\asymp\eta b_N$.

\subsection{One reassignment cycle}\label{overview:cycle}

The physical protocol implements the learner through cycles (Figure~\ref{fig:timeline}). At an assignment start the coordinator draws the slow set $U=\{a_1<\cdots<a_d\}$ and sends $a_i$ to follower~$i$ on DOWN, together with the start time, the recording rate and the queue of~$a_i$. An unoccupied follower plays the row's uniform \emph{pilot} arms on DATA; after decoding it plays $a_i$. The coordinator plays pilots until a positive DOWN observation on each recipient's listening arm \emph{confirms} delivery; it then draws a fast arm every round. Pending followers step aside on message rounds until a positive test \emph{releases} them. At the \emph{cutoff}, at the next grid boundary or after a proposal, the assignment ends. During \emph{collection} every follower reports, on UP, the score increment and queue of its arm, and the coordinator reconstructs all scores. The proposal bits are row coordinates, so every player that needs to know when the cutoff occurs can read it.

\begin{figure}[t]
\centering
\begin{tikzpicture}[font=\small,>=Latex,node distance=4mm,
  stage/.style={draw,rounded corners=2pt,fill=#1,align=center,inner sep=3pt,minimum height=12mm,text width=3.25cm},
  stage/.default={black!4},
  lab/.style={font=\footnotesize,align=center,text=black!75},
  arr/.style={->,semithick}]
\node[stage] (tab) {Public table:\\ ID array and\\ $M$ random rows};
\node[stage,right=of tab] (adv) {Adversary fixes $r$\\ (may know\\ the table)};
\node[stage,right=of adv] (J) {Coordinator draws\\ $J\in[M]$\\ privately};
\node[stage=Goldenrod!18,right=of J] (boot) {Bootstrap: followers\\ learn $J$; coordinator\\ launches (no acks)};
\draw[arr] (tab) -- (adv); \draw[arr] (adv) -- (J); \draw[arr] (J) -- (boot);
\node[stage=Cerulean!12,below=16mm of tab] (asg) {Assign: DOWN\\ words to followers;\\ confirmations};
\node[stage=Cerulean!12,right=of asg] (rel) {Release: followers\\ learn that all\\ confirmations arrived};
\node[stage=Cerulean!12,right=of rel] (live) {Live play: followers\\ on slow set $U$;\\ coordinator: fast arm};
\node[stage=Cerulean!12,right=of live] (col) {Cutoff, then collection:\\ UP score reports;\\ next $U\sim\nu_t$ or $p_t$};
\draw[arr] (asg) -- (rel); \draw[arr] (rel) -- (live);
\draw[arr] (live) -- (col);
\draw[arr] (col.south) -- ++(0,-5mm) -| (asg.south);
\draw[arr] (boot.south) -- ++(0,-7mm) -| node[lab,pos=0.25,above]{launch} (asg.north);
\node[lab,above=1mm of boot] {cost $O(m^{3/2}\sqrt{KT}\log^2N)$};
\node[lab,below=7mm of rel.south east] {each cycle costs $O(mK\log N)$ in expectation; $O(\sqrt{mT/K})$ cycles in expectation};
\end{tikzpicture}
\caption{The life of the protocol. A cutoff occurs at the next grid boundary or after a proposal of the lazy chain. During collection the logical learner keeps updating until the grid boundary and pauses after it.}
\label{fig:timeline}
\end{figure}
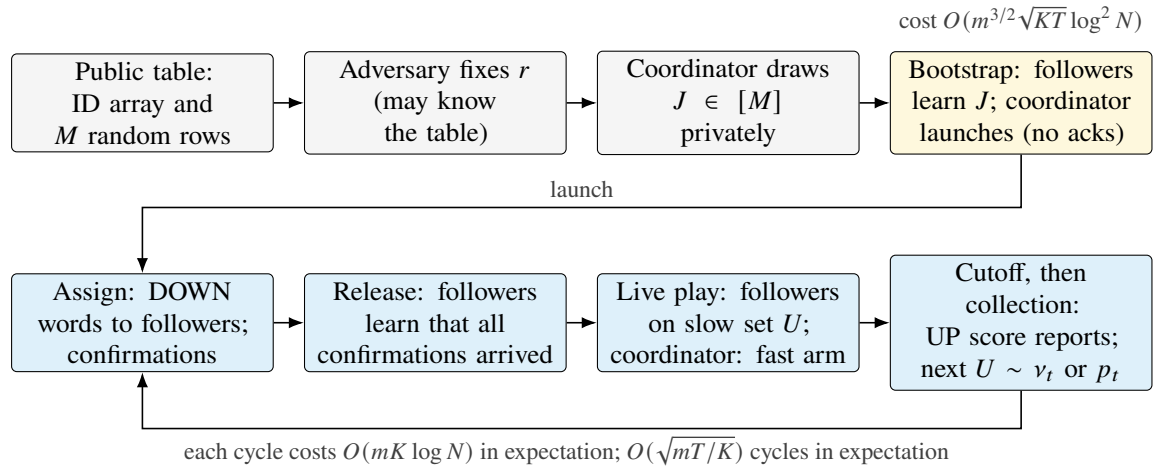

\paragraph{Exact reconstruction.}
While unoccupied, a follower logs the DATA observations that could later be recorded for an arm it may be assigned. When its DOWN word arrives, possibly after the cutoff, it replays the deterministic queue recursion from the transmitted initial queue in calendar order, inserting zero arrivals on rounds without records. Its UP reply is a fixed endpoint record: the score increment of its arm up to the cutoff, the increment of the last live round, and the queue at the cutoff. During collection the coordinator owns every arm and buffers new arrivals until each queue is known. Every update is therefore attributed to the calendar round in which it arose, and the physical scores coincide with the logical ones (Lemmas~\ref{sq:lem:reconstruct},~\ref{protocol:lem:sparse} and~\ref{protocol:lem:identities}).

\paragraph{The next assignment is not drawn from $p_t$.}
Suppose a proposal at round $u-1$ causes the cutoff $u$, and collection ends with a new assignment at $t$ before the grid boundary. Right after the proposal the logical slow set has law $\rho_{u-1}$, not $p_u$. During collection the physical actions ignore the hidden logical slow set, which keeps evolving by proposals of its own. Conditional on the whole physical collection path, its law $\mu_r$ therefore satisfies $\mu_u=\rho_{u-1}$ and $\mu_{r+1}=(1-h)\mu_r+h\rho_r$. Subtracting the identity $p_{r+1}=(1-h)p_r+h\rho_r$ gives
\[
 \mu_t=p_t+\frac{(1-h)^{t-u+1}}h\,(p_u-p_{u-1})=:\nu_t .
\]
The coordinator knows the three score vectors at $u-1,u,t$ and samples the next assignment from $\nu_t$. If collection reaches the grid boundary, a new episode starts and the coordinator samples from $p_t$ with a renewed anchor.

\begin{ilemma}[One reassignment; Lemmas~\ref{sq:lem:lazy},~\ref{sq:lem:collection} and Proposition~\ref{sq:prop:transport}]\label{inf:cycle}\leavevmode
\begin{enumerate}[label=(\alph*),itemsep=1pt]
\item The physical protocol and the logical learner can be coupled so that, up to the first queue overflow, their scores agree in every calendar round, their slow sets agree on every live round, and every physical fast-arm draw is the logical fast arm, while every coordinate of the row keeps its law.
\item Relative to the learner, one cycle costs $O(mKb_N)$ in expectation, and there are at most $(4m+128)\eta T+2$ cycles in expectation.
\end{enumerate}
\end{ilemma}

Part~(a) is the simulation lemma. Its proof generates the hidden logical variables of the collection phase from auxiliary randomness, never from the unread coordinates of the row, and then draws the coordinator's sampling coordinates uniformly from the fibre of its sampler over the logical slow set; the row coordinates thus remain uniform. A separate deterministic induction (Lemma~\ref{protocol:lem:total}) shows that every local operation is well defined; it uses only bounded service and exact reconstruction, so the argument is not circular. For part~(b), during delivery and collection the regret is at most $\bar g_t$ per round and Lemma~\ref{inf:channel}(b) applies; after delivery, the release tests cost $O(mK\log m)$.

\subsection{Agreeing on a row without a common seed}\label{overview:public}

The coordinator draws $J$ privately after the rewards are fixed; the followers know the table but not~$J$. Initialization teaches them $J$ in blocks of length $B_{\rm boot}=\lceil\sqrt{T/(mK)}\rceil$ (Figure~\ref{fig:boot}). In each block, every uninformed follower is \emph{active} with probability $1/m$; an active follower holds an arm $a$, drawn by multiplicative weights over its past gains mixed with the uniform law, and a private identification index~$\iota$. It stays on $a$ exactly when the public ID bit $I_\iota(t)$ is~$0$. If it is the only active follower and holds the coordinator's random target~$z$, the coordinator's positive observations reveal ID bits and prune the candidate indices to a list $\Lambda$ of size $L_0\asymp b_N$ (DISCOVER). The coordinator then leaves the target according to the forward bits $\varphi_J(u,t)$ of row~$J$ for all $u\in\Lambda$ (MULTICAST). Each positive stay observation of the follower certifies a forward bit of row~$J$ at its own index, so after $q_{\rm dec}\asymp b_N^2$ such observations every other row is eliminated except on the false-row failure event controlled by Lemma~\ref{comm:acceptance}.

\begin{figure}[t]
\centering
\begin{tikzpicture}[font=\footnotesize,>=Latex,
  panel/.style={draw,rounded corners=2pt,inner sep=5pt,fill=black!3,anchor=north west}]
\node[panel] (d) at (0,0) {\begin{minipage}[t]{6.9cm}\raggedright
 \textbf{DISCOVER} (the coordinator finds a short list)
 \begin{itemize}[leftmargin=1em,itemsep=1pt,topsep=2pt]
 \item Follower: private index $\iota$, held arm $a$; it plays $a$ iff $I_\iota(t)=0$.
 \item Coordinator: plays a random target $z$ on DATA rounds.
 \item If $z=a$, a positive coordinator observation implies $I_\iota(t)=1$; keep only $\{u\in\Lambda: I_u(t)=1\}$.
 \item After $6b_N$ positives, $|\Lambda|\le L_0$ and $\iota\in\Lambda$.
 \end{itemize}\end{minipage}};
\node[panel] (mc) at (7.85,0) {\begin{minipage}[t]{6.9cm}\raggedright
 \textbf{MULTICAST} (the follower learns $J$)
 \begin{itemize}[leftmargin=1em,itemsep=1pt,topsep=2pt]
 \item Coordinator leaves $z$ iff $\varphi_J(u,t)=1$ for all $u\in\Lambda$ (probability $\ge1/4$).
 \item A positive follower observation on $a$ implies $\varphi_J(\iota,t)=1$.
 \item A row $v\ne J$ survives each such test w.p.\ $1-1/L_0$.
 \item After $q_{\rm dec}=8L_0b_N$ tests only $J$ survives except on the false-row failure event.
 \end{itemize}\end{minipage}};
\draw[->,semithick] (d.east |- mc.west) -- (mc.west);
\end{tikzpicture}
\caption{Teaching the row index within one block of length $\approx\sqrt{T/(mK)}$.}
\label{fig:boot}
\end{figure}
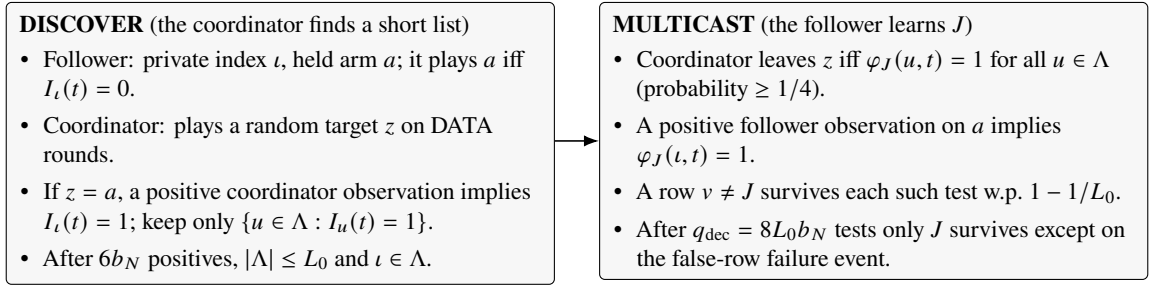

A failed attempt is still useful: its positive observations raise the held arm's weight. The multiplicative-weights potential shows that, while some follower is uninformed, every $m$-set has had cumulative positive-reward support at most a deterministic budget $Q=O(m^{3/2}\sqrt{KT}\,b_N^2)$, except on a small event. This bounds the comparator's reward, and hence the regret, while initialization is incomplete, even if a follower never learns~$J$.

\emph{Launch without acknowledgments.} The coordinator cannot see which followers are informed. It counts its own positive observations on its DOWN tuple arm, and lets $\widehat D_r$ be $K$ times the sum of its $m$ largest counts after block~$r$. These counts estimate the comparator's support $D_r$, and concentration gives
\[
 \text{some follower is uninformed}\ \Longrightarrow\ D_r\le Q\ \Longrightarrow\ \widehat D_r\le3Q .
\]
The coordinator launches learning once $\widehat D_r>4Q$: it cannot see the followers, but it can see that enough reward has passed for all of them to be informed. The converse estimate $D_r\le128(\widehat D_r+mKB_{\rm boot}\cdot4b_N)$ bounds the support spent before launch.

\begin{iprop}[Initialization; Proposition~\ref{sq:prop:startup}]\label{inf:boot}
For a good ID array, which a random array is with probability $1-3/N$, the comparator's expected reward before launch is $O(m^{3/2}\sqrt{KT}\,b_N^2)$, and except with probability $O(N^{-2})$ every follower knows $J$ at launch.
\end{iprop}

\subsection{One fixed table for all reward matrices}\label{overview:assembly}

The analysis above is carried out in the \emph{fresh-row model}, in which the rewards and the ID array are fixed first and the rows are then sampled independently, and for a \emph{reference execution}: an analysis device that tests only the true row during initialization, rejects wrongly decoded words, and corrects any follower still uninformed at launch. Table~\ref{tab:executions} lists the executions that appear in the proof. For a fixed rounded reward matrix $r$, let $F_{j,r}(R_j)$ be the expected regret of the reference execution given $J=j$ and row $R_j$, and $E_{j,r}(R_j)$ the probability that it needs a correction, makes a decoder error or overflows a queue. Both are functions of $R_j$ alone, so the variables $F_{j,r}(R_j)+2mTE_{j,r}(R_j)\in[-mT,3mT]$ are independent across rows.

\begin{ilemma}[One fixed table; Section~\ref{sq:sec:tuning} and Lemma~\ref{comm:acceptance}]\label{inf:table}
Let $M=\lceil8T^2/(mb_N)\rceil$. With probability $1-O(1/N)$ over the table, simultaneously for all $(D+1)^{KT}$ rounded reward matrices: the average over rows of $F_{j,r}+2mTE_{j,r}$ exceeds its expectation by at most $2m^{3/2}\sqrt{KT}\,b_N$, and a wrong row passes a follower's bootstrap test with private probability at most $1/B_{\rm boot}\le\sqrt{mK/T}$.
\end{ilemma}

The first part is Hoeffding's inequality over the $M$ independent rows and a union bound over the rounded matrices, as in~\cite[Appendix~B.4]{bubeck2020}. The adversary may know all rows; it does not know $J$, and for each fixed matrix the average over rows concentrates. False bootstrap passes need a separate argument, because testing all candidate rows makes the actual execution depend on rows other than~$J$. On the complement of these failure events the actual protocol and the reference execution coincide, and each failure costs at most $2mT$ times its probability.

\begin{table}[t]
\centering\small
\begin{tabularx}{\linewidth}{@{}>{\raggedright\arraybackslash}p{0.2\linewidth}XX@{}}
\toprule
Execution & Randomness & Role in the proof\\
\midrule
Logical learner & Records, live proposal bits and live fast arms shared with the physical protocol; auxiliary draws otherwise. & Regret bound (Proposition~\ref{sq:prop:learning}); coupled to the reference in Lemma~\ref{sq:lem:lazy}.\\
Fresh-row reference & Rewards and ID array fixed first; independent rows and private draws. & All conditional probability estimates: initialization, learning, communication.\\
Fresh-row actual & As above, but all rows are tested by local rules only. & Compared with the reference in Proposition~\ref{sq:prop:startup}(i).\\
Fixed-table reference & Published table and rewards fixed; private draws only. & Defines the row costs $F_{j,r},E_{j,r}$; no fresh-row identity is used after the table is fixed.\\
Fixed-table actual & Published table, possibly known to the adversary; private draws only. & The strategy of Theorem~\ref{sq:thm:main}; agrees with the reference outside the failure events.\\
\bottomrule
\end{tabularx}
\caption{Executions in the proof. The reference execution is an analysis device; the players run the actual protocol.}
\label{tab:executions}
\end{table}

\subsection{Putting it together}\label{overview:total}

\begin{center}
\fbox{\begin{minipage}{0.96\linewidth}\small
\textbf{The protocol at a glance.} \emph{Public:} $K,m,T$, an identification array and rows $R_1,\ldots,R_M$. Each row holds forward bits for the bootstrap and, for every round, a slot type, a selected follower, pilot arms, an injective tuple of arms, affine maps for both directions, and uniform coordinates for recording, thinning, proposals, the fast arm and assignment sampling.
\begin{enumerate}[label=\arabic*.,leftmargin=1.5em,itemsep=0pt,topsep=2pt]
\setcounter{enumi}{-1}
\item After the rewards are fixed, the coordinator draws $J\in[M]$ uniformly and privately.
\item \emph{Bootstrap}, in blocks of length $\lceil\sqrt{T/(mK)}\rceil$: DISCOVER and MULTICAST as in Figure~\ref{fig:boot}. Uninformed followers test candidate rows and adopt the unique passing row at the end of the block. The coordinator counts its positive DOWN observations and launches once $\widehat D_r>4Q$.
\item \emph{Learning cycles} until round $T$:
 (a)~the coordinator draws $U\sim p_t$ (new episode) or $U\sim\nu_t$ (same episode, after a collection) and sends each follower its arm, rate and queue on DOWN; positive coordinator observations confirm delivery.
 (b)~Followers play their assigned arms once they decode, stepping aside on message rounds until released; after all confirmations the coordinator draws a fast arm every round.
 (c)~On DATA at most one observation is recorded, arm $i$ with probability $\xi_i$; recorded rewards enter queues that release at most four units per arm and round.
 (d)~At the next grid boundary, or after a proposal, the assignment is cut off; followers report on UP, and the coordinator reconstructs all scores and queues.
\item If $T\le C_bmKb_N^2$ or $N<N_0$, the players use a fixed set of $m$ distinct arms instead.
\end{enumerate}
\end{minipage}}
\end{center}

Table~\ref{tab:budget} collects the costs. Choosing $\eta\asymp\sqrt{(\log\binom Km+m)/(m^2Kb_NT)}$ balances learning against communication and gives $O(m^{3/2}\sqrt{KT}\,b_N)$ for every term except initialization, which contributes $O(m^{3/2}\sqrt{KT}\,b_N^2)$. The proof combines the \emph{signed} logical regret, which may be negative on communication rounds, with the physical deficits charged to each cycle; it never discards negative logical regret (Section~\ref{sq:sec:tuning}). All algorithmic constants are explicit (Appendix~\ref{app:explicit-constants}).

\begin{table}[t]
\centering\small
\begin{tabularx}{\linewidth}{@{}l>{\raggedright\arraybackslash}X>{\raggedright\arraybackslash}p{4.4cm}>{\raggedright\arraybackslash}p{2.4cm}@{}}
\toprule
Source & Mechanism & Expected cost & Formal result\\
\midrule
Before launch & comparator support while a follower is uninformed, plus launch overshoot & $O(m^{3/2}\sqrt{KT}\,b_N^2)$ & Prop.~\ref{sq:prop:startup}\\
Learning & exponential weights on $m$-sets; queued sparse estimates; uniform mixing & $O((\log\binom Km+m)/\eta+(m^2K+Kb_N)\,\eta T)$ & Prop.~\ref{sq:prop:learning}\\
Reassignments & delivery, release and collection: $O(mKb_N)$ each, $O(m\eta T+1)$ of them & $O(m^2Kb_N\,\eta T+mKb_N)$ & Prop.~\ref{sq:prop:transport}\\
One fixed table & concentration over rows; false bootstrap passes & $O(m^{3/2}\sqrt{KT}\,b_N)$ & Section~\ref{sq:sec:tuning}\\
Rare failures, rounding & decoder errors, overflows, corrections; rounding to multiples of $1/D$ & $O(1)$ & Lemmas~\ref{sq:lem:queue-bound},~\ref{sq:lem:decoder}; Prop.~\ref{sq:prop:startup}(iii)\\
\midrule
\textbf{Total} & $\eta$ tuned as in Section~\ref{sq:sec:tuning} & $\boldsymbol{O(m^{3/2}\sqrt{KT}\,b_N^2)}$ & Theorem~\ref{sq:thm:main}\\
\bottomrule
\end{tabularx}
\caption{Where the regret goes.}
\label{tab:budget}
\end{table}

\paragraph{Open problems.}
Is the factor $m$ above the centralized rate necessary, and is the factor $\log^2(KmT)$? Can the table be made explicit, small, or efficiently verifiable? Can the horizon be unknown, or the players unlabeled? Can a single algorithm achieve logarithmic regret on stochastic instances with positive reward gaps while retaining the adversarial regret guarantee, without knowing which environment it faces?

\paragraph{Organization.}
The formal development follows the order of this overview. Section~\ref{sq:sec:learner} defines the logical learner and bounds its regret. Section~\ref{sq:sec:transport} gives the channel, the arm rules and the recording rule, and Section~\ref{sq:sec:clock} the cutoff, the exact reconstruction and the coupling; Lemma~\ref{sq:lem:lazy} is the simulation lemma and Lemma~\ref{sq:lem:collection} bounds the cost of a cycle. Section~\ref{sq:sec:boot} specifies and analyzes initialization, Section~\ref{sq:sec:tuning} proves Theorem~\ref{sq:thm:main}, and Section~\ref{sq:sec:complexity} bounds computation. Appendix~\ref{protocol:spec} proves local well-definedness (Lemma~\ref{protocol:lem:total}) and the decoder bound; Appendix~\ref{comm:appendix} contains the table construction, the bootstrap analysis and the explicit constants.

\section{The slow--fast learner}\label{sq:sec:learner}

We define the logical slow--fast learner and bound its regret from sparse recordings and score updates. The learner is an analytical process: it may use all current scores, which individual players do not yet know. Sections~\ref{sq:sec:transport}--\ref{sq:sec:clock} construct a local protocol that realizes it, and Lemma~\ref{sq:lem:lazy} proves the simulation. The results of this section use only the interface (I1)--(I4) stated after Algorithm~\ref{sq:alg:learner}.

\subsection{A set distribution with one fast arm}

In the logical learner, the coordinator uses the fast arm and the followers use the slow arms. We orient a product-weight $m$-set by designating one arm as fast. For a vector $x=(x_i)_{i\in I}$ indexed by a finite set $I$, let $e_k(x)=\sum_{S\subseteq I,\,|S|=k}\prod_{i\in S}x_i$ be its degree-$k$ elementary symmetric sum, with $e_0(x)=1$. A subscript $-J$ removes the coordinates indexed by $J$, and $-i$ abbreviates $-\{i\}$. For positive arm weights $w$, define
\[
 \begin{aligned}
 \mathcal{Z}(w)=e_m(w)=\sum_{|S|=m}\prod_{i\in S}w_i,\qquad
 Q_w(S)=\frac{\prod_{i\in S}w_i}{\mathcal{Z}(w)},\qquad
 \theta_w(i)=\sum_{S\ni i} Q_w(S).
 \end{aligned}
\]

Fix a positive anchor $\bar w$. After drawing $S\sim Q_w$, choose its fast arm $j\in S$ with probability
\begin{equation}\label{sq:eq:orientation}
 \alpha_j(S)=\frac{\bar w_j^{-1}}{\sum_{i\in S}\bar w_i^{-1}}.
\end{equation}
Assign $U=S\setminus\{j\}$ to the slow players in increasing arm-label order, and let
\[
 p_w(U)=\sum_{j\notin U}Q_w(U\cup\{j\})\alpha_j(U\cup\{j\}).
\]
Writing $s_U=\sum_{i\in U}\bar w_i^{-1}$, the conditional law of the fast arm is
\begin{equation}\label{sq:eq:q}
 q_w(j\mid U)=
 \frac{w_j/(1+s_U\bar w_j)}{\sum_{k\notin U}w_k/(1+s_U\bar w_k)}
 \qquad(j\notin U)\,.
\end{equation}
For $j\notin U$,~\eqref{sq:eq:orientation} gives $\alpha_j(U\cup\{j\})=1/(1+s_U\bar w_j)$, so
\begin{align*}
 Q_w(U\cup\{j\})\alpha_j(U\cup\{j\})
 &=\frac{\prod_{i\in U}w_i}{\mathcal{Z}(w)}
   \frac{w_j}{1+s_U\bar w_j}\\
 &=p_w(U)q_w(j\mid U)\,.
\end{align*}
The first factor is independent of $j$, so normalizing gives the last equality; $q_w(\cdot\mid U)$ requires no current weights on $U$. Since $\sum_{j\in S}\alpha_j(S)=1$, sampling $U\sim p_w$ and then $j\sim q_w(\cdot\mid U)$ recovers $Q_w$ for $U\cup\{j\}$.

\begin{lemma}[Polynomial marginal bounds]\label{sq:lem:marginal}
Suppose $\bar w_i\le w_i\le2\bar w_i$ for every $i$. Then, for every
$(m-1)$-set $U$ and $i\notin U$,
\begin{equation}\label{sq:eq:marginal}
 \begin{aligned}
 q_w(i\mid U)&\ge\frac{\theta_{\bar w}(i)}{2[(K-m+1)\theta_{\bar w}(i)+m(1-\theta_{\bar w}(i))]},\\
 \theta_w(i)&\le2\theta_{\bar w}(i).
 \end{aligned}
\end{equation}
\end{lemma}
\begin{proof}
Put $v=\bar w$, fix $i\notin U$, and write $s=\sum_{b\in U}v_b^{-1}$ and $g_U(x)=x/(1+sx)$. Let $B$ index the $m-1$ largest coordinates of $v_{-i}$ and put
\[
 L=\sum_{j\notin B,\ j\ne i}v_j,
 \qquad \ell=\frac{e_m(v_{-i})}{e_{m-1}(v_{-i})}.
\]
The identity
\[
 m e_m(v_{-i})
 =\sum_{\substack{V\subseteq[K]\setminus\{i\}\\|V|=m-1}}
   \left(\prod_{j\in V}v_j\right)
   \sum_{j\notin V,\ j\ne i}v_j
\]
gives $L\le m\ell$, because every inner sum is at least $L$.

Write $W=\sum_{j\notin U,\ j\ne i}g_U(v_j)$. Since $|B|=|U|=m-1$, pair the elements of $B\setminus U$ with those of $U\setminus B$. For a pair $j\in B\setminus U$ and $a\in U\setminus B$,
\[
 g_U(v_j)\le1/s\le v_a.
\]
Each remaining summand is bounded by its own weight, giving
\begin{equation*}
 W\le\sum_{a\in U\setminus B}v_a
       +\sum_{\substack{j\notin B\cup U\\j\ne i}}v_j
   =L\le m\ell\,.
\end{equation*}
Also $sW\le K-m$, since $sg_U(v_j)\le1$. Using $\ell/v_i=(1-\theta_v(i))/\theta_v(i)$ gives
\[
 \begin{aligned}
 \frac1{q_v(i\mid U)}
 &=1+\frac W{v_i}+sW\\
 &\le K-m+1+\frac{m(1-\theta_v(i))}{\theta_v(i)}.
 \end{aligned}
\]
Holding the anchor $\bar w=v$ fixed, changing the current weights from $v$ to $w$ multiplies each unnormalized mass in~\eqref{sq:eq:q} by $w_j/v_j\in[1,2]$. Thus $q_w(i\mid U)\ge q_v(i\mid U)/2$. These estimates imply the first inequality in~\eqref{sq:eq:marginal}.

The inclusion odds are $\theta_w(i)/(1-\theta_w(i))=w_i e_{m-1}(w_{-i})/e_m(w_{-i})$. For $j\ne i$, put $E_k=e_k(w_{-\{i,j\}})$ and write these odds as
\[
 w_i\frac{E_{m-1}+w_jE_{m-2}}{E_m+w_jE_{m-1}}.
\]
Their derivative with respect to $w_j$ has the sign of $E_{m-2}E_m-E_{m-1}^2\le0$ by Newton's inequalities (\cite[Section 2.22, Theorem 51]{hardy1952}).
So increasing the other weights cannot increase inclusion of~$i$. Multiplying its own weight by $c\ge1$ sends an inclusion probability $\theta$ to $c\theta/[1+(c-1)\theta]\le c\theta$. Taking $c=w_i/\bar w_i\le2$ gives $\theta_w(i)\le2\theta_{\bar w}(i)$.
\end{proof}

\subsection{A sparse estimate and infrequent refreshes}

We describe the recording rule for the logical learner in this section. At most one observation per round is \emph{recorded} for a raw score increment.

Recall $N=2Km(T+1)$ and $D=2^{b_N}$, and set $L_{\rm stab}=\lceil\log_2(2K(T+1)N^2)\rceil$. Choose a dyadic learning rate $\eta$ (the specific value will be specified in Section~\ref{sq:sec:tuning}) with $\eta D$ integral and
\[
 \begin{aligned}
 0<\eta\le(1024mKL_{\rm stab})^{-1},\qquad
 \beta=32\eta L_{\rm stab},\qquad
 H&=\lfloor(64\eta)^{-1}\rfloor.
 \end{aligned}
\]
An episode starting at $\tau$ consists of the rounds $\tau\le t<\bar\tau$, where $\bar\tau=\min\{T+1,H(\lfloor\tau/H\rfloor+1)\}$ is the next grid boundary. Call these rounds \emph{enabled}; between episodes the schedule pauses learning. At an episode start, set the anchor $\bar w=w_\tau$ and define
\begin{equation}\label{sq:eq:rates}
 \begin{aligned}
 f_i&=\frac{\theta_{\bar w}(i)}{32m[K\theta_{\bar w}(i)+m(1-\theta_{\bar w}(i))]},\qquad
 \xi_i=2^{\lceil\log_2\max\{\beta,f_i\}\rceil},\\
 \gamma&=32(K-m+1)\beta,\qquad
 q_w^\gamma(\cdot\mid U)=(1-\gamma)q_w(\cdot\mid U)
                       +\gamma\Unif([K]\setminus U)\,.
 \end{aligned}
\end{equation}
Here $\xi_i$ are the recording probabilities and $\gamma$ is the uniform-exploration probability. The anchor, recording probabilities, and $\gamma$ remain fixed until $\bar\tau$, whereas the fast-arm law $q_w^\gamma(\cdot\mid U)$ depends on the current weights $w$ and slow set $U$.

Whenever $\bar w\le w\le2\bar w$, the rates $\xi_i$ and uniform-exploration probability $\gamma$ satisfy
\begin{equation}\label{sq:eq:feasible}
 \begin{aligned}
 \beta\le\xi_i&\le\frac1{16mK},\\
 q_w^\gamma(i\mid U)&\ge16\xi_i\quad(i\notin U),\\
 \sum_i\frac{\theta_w(i)}{\xi_i}&\le R:=64m^2(2K-m).
 \end{aligned}
\end{equation}
The bounds $f_i,\beta\le1/(32mK)$ give the first line by dyadic rounding and imply $\sum_i\xi_i\le1/(16m)\le1$. Also $\gamma\le(K-m+1)/(mK)\le1/2$. Lemma~\ref{sq:lem:marginal} and $f_i\le1/(32mK)$ give, for $i\notin U$,
\[
 \begin{aligned}
 q_w(i\mid U)\ge16m f_i,\qquad
 \frac1{K-m+1}\ge16m f_i.
 \end{aligned}
\]
Both component laws assign mass at least $16mf_i\ge32f_i$ to $i$, and the uniform contribution is $\gamma/(K-m+1)=32\beta$. Thus $q_w^\gamma(i\mid U)\ge32\max\{f_i,\beta\}\ge16\xi_i$ by dyadic rounding.
Since $\theta_w(i)\le2\theta_{\bar w}(i)$ and $\xi_i\ge f_i$,
\[
 \frac{\theta_w(i)}{\xi_i}
 \le64m[K\theta_{\bar w}(i)+m(1-\theta_{\bar w}(i))]
 \le64mK\,.
\]
Summing the bound with anchor marginals and using $\sum_i\theta_{\bar w}(i)=m$ gives the last line of~\eqref{sq:eq:feasible}.

Use the rounded rewards $r_t(i)$ of Section~\ref{overview:section}. Store integer scores $z_i$, initially zero, and weights $w_i=(1+1/D)^{z_i}$. For round $t$, write $z_{t,i}=z_i(t)$ and $w_t=(w_{t,i})_{i\in[K]}$, where $w_{t,i}=(1+1/D)^{z_{t,i}}$.

Let $\mathcal F_t$ be the history after any start-of-round slow-set draw and before the current slot, fast-arm draw, recording coordinates and transport coordinates are exposed. It contains the current logical scores, the current slow set $U_t$, and all previously exposed variables; it is an analysis $\sigma$-field, not a player's information. The recording rates and whether learning is enabled are $\mathcal F_t$-measurable.

On enabled rounds, the logical learner requires a recorded arm $I_t$ (or $\bot$ if none is recorded) satisfying
\begin{equation}\label{sq:eq:recordlaw}
 \begin{aligned}
 \Pr(I_t=i\mid\mathcal F_t)&=\xi_i\quad(i\in[K]),\\
 \Pr(I_t=\bot\mid\mathcal F_t)&=1-\sum_i\xi_i.
 \end{aligned}
\end{equation}
The recorded arm and the fast arm may be dependent within a round. The conditioning in~\eqref{sq:eq:recordlaw} excludes the current fast-arm draw; recording a thinned fast-arm observation is therefore allowed.

For every arm $i$, define the quantized estimate by
\begin{equation}\label{sq:eq:marker}
 \widetilde r_t(i)
 =\frac{\ind\{I_t=i\}}{\eta D}
   \left\lfloor\frac{D\eta r_t(i)}{\xi_i}\right\rfloor\,.
\end{equation}
Set $\widetilde r_t=0$ on disabled rounds. With integer arrivals $\Delta_t(i)=\eta D\widetilde r_t(i)$ and initially zero integer backlogs $\mathsf B_t(i)$, serve every arm and update its score by
\begin{equation}\label{sq:eq:queue}
 \begin{aligned}
 \Gamma_t(i)&=\min\{4\eta D,\mathsf B_t(i)+\Delta_t(i)\},\\
 \mathsf B_{t+1}(i)&=\mathsf B_t(i)+\Delta_t(i)-\Gamma_t(i),\\
 z_{t+1,i}&=z_{t,i}+\Gamma_t(i).
 \end{aligned}
\end{equation}
On disabled rounds, set $\Delta_t=\Gamma_t=0$ and preserve scores and queues. Service also runs on enabled DOWN and UP rounds without arrivals. For analysis, write $b_t(i)=\mathsf B_t(i)/(\eta D)$ and $\widehat r_t(i)=\Gamma_t(i)/(\eta D)$. Enabled updates satisfy
\[
 0\le\widehat r_t(i)\le4,
 \qquad b_{t+1}(i)=(b_t(i)+\widetilde r_t(i)-4)_+.
\]
A rare recording can produce a large increment; the queue releases at most four normalized units per enabled round and keeps the rest, so raw increments equal served increments plus the change in backlog. A realized DOWN or UP round has no arrival but still serves the queues; this agrees with~\eqref{sq:eq:recordlaw}, which conditions \emph{before} the slot type is revealed.

Since $\eta/\xi_i\le1/(32L_{\rm stab})$ and each integer arrival $\Delta_t(i)$ is at most $D$, this gives
\begin{equation}\label{sq:eq:queue-conservation}
 z_{t,i}+\mathsf B_t(i)=\sum_{s<t}\Delta_s(i)\le TD.
\end{equation}
Every logical score and integer backlog therefore lies in $[0,TD]$.

Write $p_t=p_{w_t}$ and set the proposal probability $h=4m\eta\le1/4$. In each one-step comparison, both $p_t$ and $p_{t+1}$ use the current episode's fixed anchor. After a score update, define
\begin{equation}\label{sq:eq:residual}
 \rho_t(U)=\frac{p_{t+1}(U)-(1-h)p_t(U)}h.
\end{equation}
After each update, draw $B_t\sim\mathrm{Bernoulli}(h)$ independently of the history through that update; otherwise set $B_t=0$. A \emph{proposal} $B_t=1$ samples a fresh logical slow set from $\rho_t$; otherwise retain it. This makes the logical slow set change rarely while maintaining the prescribed marginal law.

For analysis, set $B_{\rm cap}=1/(32\eta)$ and define the first queue-overflow round by
\begin{equation*}
 \tau_* =\inf\{t\in[T]:\max_i b_{t+1}(i)>B_{\rm cap}\},
 \qquad\inf\varnothing=+\infty\,.
\end{equation*}
Retain the update at $\tau_*$, then keep the logical scores and queues fixed, set later raw estimates and service to zero, and stop all logical refreshes. For later auxiliary rewards use $U_t=\{1,\ldots,m-1\}$ and $j_t=m$, so the auxiliary total reward remains in $[0,m]$. The logical learner continues through decoder failures; only queue overflow invokes this continuation. The physical protocol need not detect this analytical stop.

\begin{algorithm}[H]
\caption{Logical slow--fast learner}\label{sq:alg:learner}
\begin{algorithmic}[1]
\Require The episode schedule and a recording operation satisfying~\eqref{sq:eq:recordlaw}.
\State $(z_i,\mathsf B_1(i),w_i)\gets(0,0,1)$ for $i\in[K]$.
\For{each scheduled episode $[\tau,\bar\tau)$}
 \State Set $\bar w\gets w$; compute $\xi_i,\gamma$ by~\eqref{sq:eq:rates}.
 \State Draw $U\sim p_{\bar w}$; use its sorted arms as the logical slow-player arms.
 \For{$t=\tau,\ldots,\bar\tau-1$}
  \State Draw $j_t\sim q_{w_t}^\gamma(\cdot\mid U)$.
  \State Obtain $I_t$ with conditional law~\eqref{sq:eq:recordlaw}.
  \State $(\widetilde r_t,\Delta_t)\gets(0,0)$.
  \If{$I_t=i\in[K]$}
   \State Form $\widetilde r_t(i)$ by~\eqref{sq:eq:marker}; $\Delta_t(i)\gets\eta D\widetilde r_t(i)$.
  \EndIf
  \State Apply~\eqref{sq:eq:queue} to every arm.
  \State $w_i\gets(1+1/D)^{z_i}$ for $i\in[K]$.
  \State Independently draw $B_t\sim\mathrm{Bernoulli}(h)$.
  \If{$B_t=1$}
   \State Draw the next slow set $U\sim\rho_t$ using fresh randomness.
  \EndIf
  \If{$\max_i b_{t+1}(i)>B_{\rm cap}$}
   \State Preserve scores and queues; set $\widetilde r_s=\Delta_s=\Gamma_s=0$, $B_s=0$ for $s>t$.
   \State For $s>t$, use $U_s=\{1,\ldots,m-1\}$ and $j_s=m$; \Return.
  \EndIf
 \EndFor
\EndFor
\end{algorithmic}
\end{algorithm}

\paragraph{Interface.}
The remaining results of this section hold on any probability space that carries Algorithm~\ref{sq:alg:learner} together with a filtration $(\mathcal F_t)$ of pre-round histories, as defined above, such that the current slow set, scores, queues, recording rates and enabled status are $\mathcal F_t$-measurable, and:
\begin{enumerate}[label=(I\arabic*),leftmargin=*,itemsep=1pt]
\item episode starts are $(\mathcal F_t)$-stopping times, each episode's endpoint is fixed at its start, and the initial slow set of an episode starting before overflow has conditional law $p_{\bar w}$ given the history before the draw;
\item on every enabled round $t\le\tau_*$, the recorded arm has the conditional law~\eqref{sq:eq:recordlaw} given $\mathcal F_t$;
\item on every enabled round $t\le\tau_*$, the fast arm has conditional law $q^\gamma_{w_t}(\cdot\mid U_t)$ given $\mathcal F_t$;
\item on every enabled round $t\le\tau_*$, the proposal bit is Bernoulli$(h)$ and independent of the history through that round's update, and a refresh draws from $\rho_t$ with fresh randomness.
\end{enumerate}
In the coupling of Section~\ref{sq:sec:clock}, $\mathcal F_t$ is the joint pre-round history of the physical and logical executions. An unconditional statement about the recording probabilities would not suffice.

\begin{lemma}[Stability and fixed refresh probability]\label{sq:lem:stability}
Every active episode satisfies $\bar w\le w_t<2\bar w$.
Each enabled update obeys
\begin{equation}\label{sq:eq:hazard}
 p_{t+1}(U)\ge(1-h)p_t(U).
\end{equation}
\end{lemma}
\begin{proof}
Each update has $\log(w_{t+1,i}/w_{t,i})=D\log(1+1/D)\eta\widehat r_t(i)\le4\eta$. Over at most $H$ episode rounds,
\[
 \log(w_{t,i}/\bar w_i)\le4\eta H\le1/16<\log2.
\]
For an enabled update, each product of $m$ coordinates increases by at most $e^{4m\eta}$, so $\mathcal{Z}(w_{t+1})/\mathcal{Z}(w_t)\le e^{4m\eta}$. Since weights only increase, for every $j\notin U$,
\[
 \begin{aligned}
 \frac{Q_{w_{t+1}}(U\cup\{j\})}{Q_{w_t}(U\cup\{j\})}
 &=\frac{\mathcal{Z}(w_t)}{\mathcal{Z}(w_{t+1})}
   \prod_{i\in U\cup\{j\}}\frac{w_{t+1,i}}{w_{t,i}}\\
 &\ge e^{-4m\eta}\ge1-h.
 \end{aligned}
\]
The coefficients $\alpha_j(U\cup\{j\})$ are nonnegative and fixed by $\bar w$. Weighting the ratio bound by $Q_{w_t}(U\cup\{j\})\alpha_j(U\cup\{j\})$ and summing over $j\notin U$ proves~\eqref{sq:eq:hazard}.
\end{proof}

By~\eqref{sq:eq:hazard}, the law~\eqref{sq:eq:residual} has nonnegative terms summing to one, validating every refresh step.

\begin{lemma}[Queue bound]\label{sq:lem:queue-bound}
We have $\Pr(\tau_*\le T)\le N^{-2}$. In the stopped
continuation every normalized backlog $b_t(i)$ is at most $2B_{\rm cap}$, and
\begin{equation}\label{sq:eq:weighted-backlog}
 \E\sum_{\substack{t\le\tau_*\\t\text{ enabled}}}
     \sum_i\theta_{w_t}(i)b_t(i)
 \le\frac2{11}\left(R+\frac{33}2m\right)T,
\end{equation}
where $R=64m^2(2K-m)$ is from~\eqref{sq:eq:feasible}.
\end{lemma}
\begin{proof}
Call an enabled round \emph{retained} if $t\le\tau_*$. This test is $\mathcal F_t$-measurable because it checks only earlier overflows. On such a round, the recording law and rounding give
\begin{equation}\label{sq:eq:record-moments}
 \begin{aligned}
 0\le\widetilde r_t(i)&\le1/\beta,\\
 r_t(i)-\frac{\xi_i}{\eta D}
 &\le\E[\widetilde r_t(i)\mid\mathcal F_t]\le r_t(i)\le1,\\
 \E[\widetilde r_t(i)^2\mid\mathcal F_t]
 &\le\frac{r_t(i)^2}{\xi_i}\le\frac1{\xi_i}\,.
 \end{aligned}
\end{equation}
Let $x_t(i)=\widetilde r_t(i)-4$ on retained enabled rounds, and zero otherwise. Since $e^x$ is convex on $[0,1]$ and $\beta\widetilde r_t(i)\in[0,1]$, on a retained enabled round
\[
 \begin{aligned}
 \E[e^{\beta x_t(i)}\mid\mathcal F_t]
 &\le e^{-4\beta}\bigl(1+(e-1)\beta
               \E[\widetilde r_t(i)\mid\mathcal F_t]\bigr)\\
 &\le e^{(e-1-4)\beta}\le1.
 \end{aligned}
\]
On other rounds $x_t(i)=0$, so the same conditional bound holds. The queue update and stop rules give $b_1(i)=0$ and $b_{t+1}(i)=(b_t(i)+x_t(i))_+$. Since $e^{y_+}\le1+e^y$, the process $X_t(i)=e^{\beta b_t(i)}$ satisfies
\begin{align*}
 \E[X_{t+1}(i)\mid\mathcal F_t]
 &\le1+X_t(i)\E[e^{\beta x_t(i)}\mid\mathcal F_t]
 \le1+X_t(i)\,.
\end{align*}
Iterating from $X_1(i)=1$ gives $\E X_{T+1}(i)\le T+1$. Since queues are fixed after overflow, $\tau_*\le T$ implies $X_{T+1}(i)>e^{\beta B_{\rm cap}}$ for some $i$. Markov's inequality and a union bound give
\begin{align*}
 \Pr(\tau_*\le T)
 &\le e^{-\beta B_{\rm cap}}\sum_i\E X_{T+1}(i)\\
 &\le K(T+1)e^{-\beta B_{\rm cap}}
 =K(T+1)e^{-L_{\rm stab}}\le N^{-2}\,,
\end{align*}
since $\beta B_{\rm cap}=L_{\rm stab}$ and $L_{\rm stab}\ge\log(2K(T+1)N^2)$.
Before a retained update, $b_t(i)\le B_{\rm cap}$. Its successor is at
most $B_{\rm cap}+1/\beta\le2B_{\rm cap}$, since $L_{\rm stab}\ge1$, and queues remain fixed thereafter.

For the weighted bound write $\theta_t(i)=\theta_{w_t}(i)$ and put $V_t=\sum_i\theta_t(i)b_t(i)^2$. On a retained enabled round, capped service gives $w_{t+1,i}/w_{t,i}\le e^{4\eta}$. The coordinate monotonicity proved in Lemma~\ref{sq:lem:marginal} therefore gives
\begin{equation*}
 \theta_{t+1}(i)\le e^{4\eta}\theta_t(i)
 \le(1+8\eta)\theta_t(i)\,.
\end{equation*}
The last inequality uses $4\eta\le1$ and $e^x\le1+2x$ for $0\le x\le1$.

The queue recursion, nonnegativity of $\widetilde r_t(i)$, and the bounds $b_{t+1}(i)\le2B_{\rm cap}$ and $b_{t+1}(i)\le b_t(i)+\widetilde r_t(i)$ give
\begin{align*}
 b_{t+1}(i)^2-b_t(i)^2
 &\le2b_t(i)(\widetilde r_t(i)-4)+\widetilde r_t(i)^2+16,\\
 b_{t+1}(i)^2
 &\le2B_{\rm cap}(b_t(i)+\widetilde r_t(i))\,.
\end{align*}
Apply the inclusion-probability bound before taking expectations. Since $\theta_t(i)$ is $\mathcal F_t$-measurable, $16\eta B_{\rm cap}=1/2$, and $\sum_i\theta_t(i)=m$, these inequalities yield
\begin{align*}
 \E[V_{t+1}-V_t\mid\mathcal F_t]
 &\le\sum_i\theta_t(i)\E\!\left[
 b_{t+1}(i)^2-b_t(i)^2+8\eta b_{t+1}(i)^2
 \,\middle|\,\mathcal F_t\right]\\
 &\le\sum_i\theta_t(i)
 \left[-6b_t(i)+\frac1{\xi_i}+16+\frac12(b_t(i)+1)\right]\\
 &\le-\frac{11}2\sum_i\theta_t(i)b_t(i)
       +R+\frac{33}2m\,.
\end{align*}
The second inequality uses~\eqref{sq:eq:record-moments}. The last inequality is~\eqref{sq:eq:feasible}.
Other rounds and anchor resets leave $V_t$ unchanged. Telescoping from $V_1=0$ and using $V_{T+1}\ge0$, we obtain~\eqref{sq:eq:weighted-backlog}.
\end{proof}

\subsection{Learning bound}

Let $\mathcal A$ be all rounds enabled by the episode schedule, including any such rounds after queue overflow. Use Algorithm~\ref{sq:alg:learner} through the overflow update and its bounded continuation afterwards. For a fixed comparator $m$-set $S$, define
\[
 R_{\rm aux}(S)=\E\sum_{t\in\mathcal A}
       \bigl(r_t(S)-r_t(U_t)-r_t(j_t)\bigr).
\]

Fix an episode starting at $\tau$ before any queue overflow. Its exposed history $\mathcal H_\tau$ before the initial slow-set draw fixes the endpoint~$\bar\tau$ and the recording rates. Define
\[
 \mathcal E_t
 =\sigma\bigl(\mathcal H_\tau,
          \widetilde r_\tau,\ldots,\widetilde r_{t-1}\bigr),
 \qquad \tau\le t<\bar\tau.
\]
The scores, queues, and $p_t$ are $\mathcal E_t$-measurable. The history $\mathcal E_t$ records past estimates but hides the within-episode proposals and slow-set draws, whereas $\mathcal F_t$ also contains the current slow set and is used for the recording moments. The law $p_t$ is claimed only given $\mathcal E_t$; Lemma~\ref{sq:lem:lazy} derives the different law that holds at a physical assignment start after a collection.

\begin{lemma}[Conditional slow-set law]\label{sq:lem:conditional-law}
For each round $t$ of the episode,
\[
 \Pr(U_t=U\mid\mathcal E_t)=p_t(U)
 \quad\text{on }\{t\le\tau_*\}.
\]
The retention event $\{t\le\tau_*\}$ is $\mathcal E_t$-measurable, so the identity includes the overflow round.
\end{lemma}
\begin{proof}
Induct from the initial draw at~$\tau$. Suppose the identity holds at~$t$ with no earlier overflow. By~\eqref{sq:eq:recordlaw} and~\eqref{sq:eq:marker}, the next estimate has a common conditional law $K_t$, determined by the $\mathcal E_t$-measurable frozen rates, fixed rewards, and enabled status. For every set~$C$ of estimates,
\[
 \Pr(\widetilde r_t\in C\mid\mathcal F_t)=K_t(C).
\]
For each slow set $U$, the tower property now gives
\begin{align*}
 \Pr(U_t=U,\widetilde r_t\in C\mid\mathcal E_t)
 &=\E[\ind\{U_t=U\}K_t(C)\mid\mathcal E_t]\\
 &=p_t(U)K_t(C).
\end{align*}
Thus revealing the estimate preserves the law~$p_t$, although the estimate determines the queue and score update and hence $\rho_t$. The proposal bit is independent of the history through that estimate, and a refresh draws from $\rho_t$ independently of the preceding set. Hence, for $t+1<\bar\tau$, on $\{t+1\le\tau_*\}$,
\begin{align*}
 \Pr(U_{t+1}=U\mid\mathcal E_{t+1})
 &=(1-h)p_t(U)+h\rho_t(U)\\
 &=p_{t+1}(U).\qedhere
\end{align*}
\end{proof}

\Needspace{6\baselineskip}
\begin{proposition}[Auxiliary learning bound]\label{sq:prop:learning}
For every fixed $m$-set $S$ chosen independently of the players' randomness,
\begin{equation}\label{sq:eq:learnbound}
 \begin{aligned}
 R_{\rm aux}(S)&\le\frac{2\log\binom Km+m/16}{\eta}\\
 &\quad+\left(16m^2+\frac{16}{11}R+32m
        +1024(K-m+1)L_{\rm stab}\right)\eta T\\
 &\quad+\frac{mT}{\eta D}+mTN^{-2}.
 \end{aligned}
\end{equation}
In particular, $S$ may maximize the original oblivious reward matrix.
\end{proposition}
\begin{proof}
Fix an episode $[\tau,\bar\tau)$ starting before overflow. On a retained round, Lemma~\ref{sq:lem:conditional-law} and the joint-law identity following~\eqref{sq:eq:q} give $U_t\cup\{j_t\}\sim Q_{w_t}$ conditional on $\mathcal E_t$ before uniform mixing. Mixing replaces only the fast arm and costs at most $\gamma$, since rewards lie in $[0,1]$. Thus, for $\tau\le t<\bar\tau$,
\begin{equation}\label{sq:eq:conditional-reward}
 \E[r_t(U_t)+r_t(j_t)\mid\mathcal E_t]
 \ge\sum_i\theta_{w_t}(i)r_t(i)-\gamma\qquad\text{on }\{t\le\tau_*\}\,.
\end{equation}

Put $\mathcal A_*:=\{t\in\mathcal A:t\le\tau_*\}$. Within the fixed episode, retention is $\mathcal E_t$-measurable. To sum over the random episode schedule, index its successive start times, condition on each start time and its exposed history, sum the retained-round conditional inequalities until that episode's already fixed endpoint, and then apply the tower property. Finally sum over the disjoint episodes. There are at most $T$ episodes and $|\mathcal A_*|\le T$, so this gives
\begin{equation}\label{sq:eq:stopped-reward-comparison}
 \E\sum_{t\in\mathcal A_*}
 \bigl(r_t(S)-r_t(U_t)-r_t(j_t)\bigr)
 \le\E\sum_{t\in\mathcal A_*}
 \left(r_t(S)-\sum_i\theta_{w_t}(i)r_t(i)\right)+\gamma T.
\end{equation}

For an $m$-set $S$, write $\widetilde r_t(S)=\sum_{i\in S}\widetilde r_t(i)$ and $\widehat r_t(S)=\sum_{i\in S}\widehat r_t(i)$ for the raw and served estimates.
All subsequent potential sums run over calendar rounds $t\in[T]$; estimates and service are zero outside $\mathcal A_*$. Put $\lambda=D\log(1+1/D)\eta\in[\eta/2,\eta]$.
For every $m$-set $S'$, its log-product increment is
$x(S')=\lambda\widehat r_t(S')\in[0,4m\lambda]\subseteq[0,1]$. Hence
\[
 \begin{aligned}
 \log\frac{\mathcal{Z}(w_{t+1})}{\mathcal{Z}(w_t)}
 &=\log\sum_{|S'|=m}Q_{w_t}(S')e^{x(S')}\\
 &\le\lambda\sum_i\theta_{w_t}(i)\widehat r_t(i)
      +\lambda^2\sum_{|S'|=m}Q_{w_t}(S')\widehat r_t(S')^2.
 \end{aligned}
\]
Here we used $e^x\le1+x+x^2$ for $0\le x\le1$ and
$\log(1+x)\le x$. Since $0\le\widehat r_t(S')\le4m$, summing the
one-step bound and telescoping the logarithms gives
\[
 \log\frac{\mathcal{Z}(w_{T+1})}{\mathcal{Z}(w_1)}
 \le\lambda\sum_{t,i}\theta_{w_t}(i)\widehat r_t(i)
      +16m^2\lambda^2T.
\]
On the other hand, the weight product for $S$ is one term of
$\mathcal{Z}(w_{T+1})$. Since $\mathcal{Z}(w_1)=\binom Km$ and
$\prod_{i\in S}w_{T+1,i}=\exp(\lambda\sum_t\widehat r_t(S))$,
\[
 \log\frac{\mathcal{Z}(w_{T+1})}{\mathcal{Z}(w_1)}
 \ge\lambda\sum_t\widehat r_t(S)-\log\binom Km.
\]
Comparing the two bounds and dividing by~$\lambda$ gives
\[
 \sum_t\left(\widehat r_t(S)
       -\sum_i\theta_{w_t}(i)\widehat r_t(i)\right)
 \le\frac{\log\binom Km}{\lambda}+16m^2\lambda T.
\]
The weights persist across episodes, so this pays $\log\binom Km$ only once.

Write $b_t(S)=\sum_{i\in S}b_t(i)$. Queue conservation gives $\widetilde r_t(i)-\widehat r_t(i)=b_{t+1}(i)-b_t(i)$. Summation by parts recovers the raw estimates through
\[
 \begin{aligned}
 \sum_t\left(\widetilde r_t(S)
       -\sum_i\theta_{w_t}(i)\widetilde r_t(i)\right)
 =&\sum_t\left(\widehat r_t(S)
       -\sum_i\theta_{w_t}(i)\widehat r_t(i)\right)
       +b_{T+1}(S)-\sum_i\theta_{w_T}(i)b_{T+1}(i)\\
 &+\sum_{t=2}^T\sum_i(\theta_{w_t}(i)-\theta_{w_{t-1}}(i))b_t(i).
 \end{aligned}
\]
By Lemma~\ref{sq:lem:queue-bound}, the terminal expression is at most $2mB_{\rm cap}=m/(16\eta)$. Only indices $t$ with $t-1\in\mathcal A_*$ can change the inclusion probabilities. For these indices,
\[
 \begin{aligned}
 (\theta_{w_t}(i)-\theta_{w_{t-1}}(i))_+
 &\le8\eta\theta_{w_{t-1}}(i),\\
 b_t(i)&\le b_{t-1}(i)+\widetilde r_{t-1}(i).
 \end{aligned}
\]
Using the conditional first-moment bound in~\eqref{sq:eq:record-moments}, $\sum_i\theta_{w_{t-1}}(i)=m$, and \eqref{sq:eq:weighted-backlog}, we obtain
\begin{align*}
 \E\sum_{t=2}^T\sum_i
       (\theta_{w_t}(i)-\theta_{w_{t-1}}(i))b_t(i)
 &\le8\eta\left[\frac2{11}\left(R+\frac{33}2m\right)T+mT\right]
 =\left(\frac{16}{11}R+32m\right)\eta T\,.
\end{align*}

For $t\in \mathcal A_*$, the moment bounds~\eqref{sq:eq:record-moments}, with $\sum_{i\in S}\xi_i\le m$, imply
\[
 \begin{aligned}
 r_t(S)-\sum_i\theta_{w_t}(i)r_t(i)
 \le\E\!\left[\widetilde r_t(S)
           -\sum_i\theta_{w_t}(i)\widetilde r_t(i)
                     \,\middle|\,\mathcal F_t\right]
           +\frac{m}{\eta D}.
 \end{aligned}
\]
Apply the conditional estimate comparison on each retained round, whose indicator is $\mathcal F_t$-measurable, then the potential and queue bounds. Using $\lambda\le\eta$ and $1/\lambda\le2/\eta$ gives
\begin{align*}
 \E\sum_{t\in\mathcal A_*}
       \left(r_t(S)-\sum_i\theta_{w_t}(i)r_t(i)\right)
 \le\frac{2\log\binom Km+m/16}{\eta}
       +\left(16m^2+\frac{16}{11}R+32m\right)\eta T
       +\frac{mT}{\eta D}.
\end{align*}
By~\eqref{sq:eq:stopped-reward-comparison}, add $\gamma T=1024(K-m+1)L_{\rm stab}\eta T$ to bound retained auxiliary regret.

Rounds in $\mathcal A\setminus\mathcal A_*$ occur only after overflow, and each regret summand is at most $m$, giving
\[
 R_{\rm aux}(S)
 \le\E\sum_{t\in\mathcal A_*}
       \bigl(r_t(S)-r_t(U_t)-r_t(j_t)\bigr)
       +mT\Pr(\tau_*\le T).
\]
Lemma~\ref{sq:lem:queue-bound} bounds the last probability by~$N^{-2}$.
\end{proof}

\section{Assignment and score communication}\label{sq:sec:transport}
This section and the next explain how players use communication to implement the learning protocol.
Communication uses arm choices: a sender collides with a receiver or leaves its arm free, and a positive receiver observation gives a message equation.

The analysis in Sections~\ref{sq:sec:transport}--\ref{sq:sec:boot} uses the \emph{fresh-row model} of Section~\ref{overview:assembly}: the rounded rewards and the identification array are fixed first, the rows are then sampled independently, and $J$ is uniform. Its \emph{reference execution} uses the reference initialization of Section~\ref{sq:sec:boot} (Algorithm~\ref{sq:alg:reference}) and rejects every decoded ordinary word other than the word currently encoded on its channel, before applying assignment-dependent updates; ambiguous batches are cleared as prescribed, and the reference continues after either event. Section~\ref{sq:sec:tuning} transfers the resulting bounds to a fixed published table.

\subsection{Message decoding}\label{sq:sec:message-decoding}

Let $Y_{i,t}$ denote player $i$'s observed reward rounded down to a multiple of $1/D$, as in Section~\ref{overview:section}. The selected row specifies a common round type, or slot, $\tau_t\in\{\mathrm{DATA},\mathrm{DOWN},\mathrm{UP}\}$. In the fresh-row model, these types have probabilities $1/2$, $1/4$, and $1/4$, respectively.

Recall that $d=m-1$ is the number of followers. On each DOWN or UP round, the row selects $k_t\in[d]$ uniformly and supplies an independent uniform injective arm tuple $(R_{0,t},\ldots,R_{d,t})\in[K]^{d+1}$. Conditional on $k_t=i$, DOWN goes from the coordinator to player $i$ and UP goes in reverse. We may suppress the time index on tuple arms if this causes no confusion.

Set $B=5\lceil\log_2(TD+1)\rceil$. Maps in Section~\ref{sq:sec:encoding} encode each message into the $B$ bits.
For each round and each direction and selected player, independently draw
$c_t\sim\Unif(\mathbb F_2^B)$ and $b_t\sim\Unif(\mathbb F_2)$, and define
\[
 H_t(u)=c_t\cdot u\oplus b_t.
\]
Here $\oplus$ is addition modulo two, and $c\cdot u$ is the binary dot product.
On DOWN, put $(R_{\rm send},R_{\rm receive})=(R_0,R_i)$; on UP, put $(R_{\rm send},R_{\rm receive})=(R_i,R_0)$. The receiver plays $R_{\rm receive}$. To send $u$, the sender plays $R_{\rm receive}$ if $H_t(u)=0$, and $R_{\rm send}$ if $H_t(u)=1$. Thus $H_t(u)=0$ forces a collision and a zero receiver observation. A positive receiver observation therefore implies
\begin{equation*}
 c_t\cdot u=1\oplus b_t\,.
\end{equation*}

The decoder uses batches of
\begin{equation}\label{sq:eq:decode-quota}
 Q_{\rm dec}=B+\left\lceil\log_2(4dTN^2)\right\rceil
\end{equation}
positive observations. Each directed channel stores one initially empty batch. 

\begin{algorithm}[H]
\caption{One affine message channel}
\label{sq:alg:sender}\label{sq:alg:decoder}
\begin{algorithmic}[1]
\Procedure{SendArm}{$u,R_{\rm send},R_{\rm receive}$}
 \State \Return $R_{\rm receive}$ if $H_t(u)=0$, and $R_{\rm send}$ otherwise.
\EndProcedure
\Procedure{DecodeStep}{$W,Y$}
 \Require Batch $W$; observation $Y$.
 \State If $Y>0$, append $c_t\cdot u=1\oplus b_t$ to $W$.
 \State If $|W|<Q_{\rm dec}$, \Return $\bot$.
 \State $M\gets$ the unique solution of $W$ over $\mathbb F_2^B$, or $\bot$ otherwise.
 \State $W\gets\varnothing$.
 \State \Return $M$.
\EndProcedure
\end{algorithmic}
\end{algorithm}

A message is \emph{awaited} while unaccepted if, on every active round selecting its direction and follower, the sender encodes it, the receiver runs $\operatorname{DecodeStep}$, and it satisfies the acceptance test~\eqref{protocol:eq:acceptance}. The initial UP placeholder $(0,0,0,0)$ fails the UP acceptance test and is not awaited.

\begin{lemma}[Message decoding]\label{sq:lem:decoder}
Consider the fresh-row reference execution with corrected initialization and the construction-branch rate choice. Work through the first queue-overflow update or the horizon. With probability at least $1-N^{-2}$,
no completed batch has a unique solution different from the word currently encoded on its channel that passes~\eqref{protocol:eq:acceptance}, and no completed batch formed entirely from one fixed awaited word has multiple solutions. On this event, an awaited word is accepted within $2Q_{\rm dec}$ positive receiver observations after its sender fixes that word, within this prefix and provided service continues.
\end{lemma}

Appendix~\ref{protocol:decoder-proof} proves this lemma. Either excluded batch event is called a \emph{decoder failure}.

\subsection{One assignment: delivery and play}
\label{sq:sec:release}

A \emph{cutoff} during play ends the current assignment before arm choices and starts UP collection. The assignment is \emph{live} from its start until this cutoff. The coordinator continues sending the old DOWN records during collection.

The first assignment starts with zero scores and queues. Before each later assignment, the coordinator accepts all preceding UP replies and reconstructs the current scores and queues from them and its collection observations (Section~\ref{sq:sec:clock}). At an episode start, it sets $\bar w$ to the current weights and computes the recording rates and $\gamma$ by~\eqref{sq:eq:rates}; assignment starts within the same episode retain these quantities.

At assignment time $\sigma$, the coordinator chooses $U=\{a_1<\cdots<a_d\}$ and assigns $a_i$ to follower $i$. Put $F=[K]\setminus U$, and choose the auxiliary arm $e\in F$ minimizing the frozen starting weight $\bar w_e$, breaking ties by arm label. The coordinator sends
\begin{equation}\label{sq:eq:handover-record}
 M_i^{\rm DOWN}=(\sigma,a_i,e,-\log_2\xi_{a_i},\mathsf B_\sigma(a_i))\,.
\end{equation}
Here $\xi_{a_i}$ is the episode's fixed recording probability, and $\mathsf B_\sigma(a_i)$ is the integer queue before round $\sigma$. These fields specify the recording and queue updates over the assignment's live interval, including updates reconstructed after delayed decoding.

For its UP replies, follower $i$ stores $(\sigma_i,S_i,\delta_i,\mathsf b_i)$, initially $(0,0,0,0)$. The fields are its last decoded start, cumulative served increments for that assignment, service in the last completed calendar round of its live interval, and queue after that round. The service $\delta_i$ may be zero. Section~\ref{sq:sec:clock} specifies the updates and recovery after delayed decoding.

A follower with no decoded live assignment is \emph{unoccupied}. Decoding a live assignment puts it in the \emph{pending} state. If a DOWN record is accepted at or after cutoff, the follower reconstructs that assignment's terminal reply and remains unoccupied.

During a live assignment, the coordinator maintains the set $\mathcal P\subseteq[d]$ of followers whose confirmations are missing, initially $\mathcal P=[d]$. For $i\in\mathcal P$, a \emph{confirmation for $i$} is the event
\begin{equation*}
 \{\tau_t=\mathrm{DOWN},\ k_t=i,\ H_t(M_i^{\rm DOWN})=0,\ Y_{0,t}>0\}\,.
\end{equation*}
The coordinator then removes $i$ from $\mathcal P$.

On DATA, the row supplies independent uniform pilot arms $P_0,\ldots,P_d\in[K]$. An unoccupied follower plays $P_i$, producing observations that can be used after decoding; a follower with a decoded live assignment plays $a_i$. The coordinator plays $P_0$ during collection or while confirmations remain missing. During a live assignment with $\mathcal P=\varnothing$, it instead samples its fast arm from $q^\gamma_{w_t}(\cdot\mid U)$ on every round. Equations~\eqref{sq:eq:q} and~\eqref{sq:eq:rates} use the already fixed anchor and current outside weights, maintained through the coordinator's own observations and queue service; no current follower scores are needed.

During a live round, a follower $i$ that was pending at arm choice becomes \emph{released} on the \emph{release observation}
\begin{equation*}
 \{Y_{i,t}>0\}\cap
 \bigl(\{\tau_t=\mathrm{UP},\ k_t=i,\ R_0=e\}
       \cup\{\tau_t=\mathrm{DATA},\ P_0=a_i\}\bigr)\,.
\end{equation*}
For the DATA test, $P_0=a_i$ forces a collision while confirmations are missing. A positive observation in this test therefore certifies that all confirmations have arrived. Pending and released followers are \emph{occupied}. At cutoff they become unoccupied and retain their reply fields until the next DOWN acceptance.

\begin{algorithm}[H]
\caption{Choose the arms on one active round}
\label{sq:alg:handover}
\begin{algorithmic}[1]
\State On DOWN or UP, $i\gets k_t$.
\If{the coordinator's assignment is live and $\mathcal P=\varnothing$}
 \State Draw $X_t\sim q^\gamma_{w_t}(\cdot\mid U)$ using the coordinator's independent sampler coordinate.
 \State $A_{0,t}\gets X_t$.
\ElsIf{the slot is DATA}
 \State $A_{0,t}\gets P_0$.
\ElsIf{DOWN selects $i\in\mathcal P$}
 \State $A_{0,t}\gets$ \Call{SendArm}{$M_i^{\rm DOWN},R_0,R_i$}.
\Else
 \State $A_{0,t}\gets R_0$.
\EndIf
\For{$j=1,\ldots,d$}
 \If{$j$ is unoccupied}
  \If{the slot is DATA}
   \State $A_{j,t}\gets P_j$.
  \ElsIf{UP selects $j$}
   \State $A_{j,t}\gets$ \Call{SendArm}{$(\sigma_j,S_j,\delta_j,\mathsf b_j),R_j,R_0$}.
  \Else
   \State $A_{j,t}\gets R_j$.
  \EndIf
 \ElsIf{$j$ is released or the slot is DATA}
  \State $A_{j,t}\gets a_j$.
 \ElsIf{UP selects $j$ and $R_0=e$}
  \State $A_{j,t}\gets e$.
 \ElsIf{$a_j=R_i$ on DOWN or $a_j=R_0$ on UP}
  \State $A_{j,t}\gets e$.
 \Else
  \State $A_{j,t}\gets a_j$.
 \EndIf
\EndFor
\State Each player $j=0,\ldots,d$ plays $A_{j,t}$.
\end{algorithmic}
\end{algorithm}

\begin{lemma}[The auxiliary arm has bounded fast probability]
\label{sq:lem:auxiliary-arm}
Suppose $\bar w_j\le w_j\le2\bar w_j$ for every arm $j$, and let
$e$ minimize $\bar w_j$ over $j\in F=[K]\setminus U$. Then
\[
 q^\gamma_w(e\mid U)\le\frac23.
\]
\end{lemma}
\begin{proof}
Recall $s_U=\sum_{i\in U}\bar w_i^{-1}$ from~\eqref{sq:eq:q}. Since $x/(1+s_Ux)$ increases on $(0,\infty)$, minimality of $\bar w_e$ and the weight bounds give, for $j\in F\setminus\{e\}$,
\[
 \frac{w_j}{1+s_U\bar w_j}
 \ge\frac{\bar w_e}{1+s_U\bar w_e}
 \ge\frac12\frac{w_e}{1+s_U\bar w_e}.
\]
Normalizing in~\eqref{sq:eq:q} gives
\begin{equation*}
 q_w(e\mid U)\le\left(1+\frac{|F|-1}{2}\right)^{-1}=\frac2{|F|+1}\,.
\end{equation*}
Since $|F|=K-m+1\ge2$ and $0\le\gamma\le1$,
\[
 q^\gamma_w(e\mid U)
 \le(1-\gamma)\frac2{|F|+1}+\frac{\gamma}{|F|}
 \le\frac23.
\]
\end{proof}

\begin{lemma}[Deterministic avoidance and confirmations]\label{sq:lem:avoidance}
Consider a reference execution during a live assignment.
While $\mathcal P\ne\varnothing$, no follower is released.
On DOWN selecting $i\in\mathcal P$, every pending follower $j$
satisfies $A_{j,t}\ne R_i$. On such a round,
\begin{align*}
 Y_{i,t}&=H_t(M_i^{\rm DOWN})r_t(R_i)
       &&\text{if }i\text{ is unoccupied},\\
 Y_{0,t}&=r_t(R_i)
       &&\text{if }i\text{ is pending and }H_t(M_i^{\rm DOWN})=0.
\end{align*}
In particular, a confirmation for $i$ implies that $i$ was occupied at the start of the round.
\end{lemma}
\begin{proof}
No follower is released at assignment start, since none has decoded. If $\mathcal P\ne\varnothing$ at round $t$'s start, either release test forces a collision with the coordinator: on $e$ for the UP test, or on $a_i$ for the DATA test. Thus $Y_{i,t}=0$, proving the first assertion by induction.

On DOWN selecting $i\in\mathcal P$, a pending follower $j$ uses $a_j$ unless $a_j=R_i$, when it uses $e\notin U$. In either case it avoids $R_i$.

An unoccupied $i$ plays $R_i$. If $H_t(M_i^{\rm DOWN})=0$, the coordinator also plays $R_i$, so $Y_{i,t}=Y_{0,t}=0$ and confirmation cannot occur. If $H_t(M_i^{\rm DOWN})=1$, the coordinator plays $R_0$. Tuple injectivity gives $R_0\ne R_i$ and $R_j\ne R_i$ for every other unoccupied follower $j$, while pending followers avoid $R_i$. Hence follower $i$ is alone on $R_i$ and observes $Y_{i,t}=r_t(R_i)$.

With $i$ pending and $H_t(M_i^{\rm DOWN})=0$, the coordinator is alone on $R_i$: pending players avoid it, and unoccupied players use distinct tuple arms with labels different from $i$.
\end{proof}

\subsection{Communication costs}

We first prove a general lemma to bound the cost of a parallel service process, then apply this lemma to bound the communication costs before the cutoff. Fix a \emph{service phase} with first round $\sigma$ and first excluded round $\vartheta$, where $1\le\sigma\le\vartheta\le T+1$ are stopping times for the pre-round histories $(\mathcal F_t)$. A phase terminated after round $t$ has $\vartheta=t+1$.

For $\sigma\le t<\vartheta$, let $P_t\subseteq[d]$ be the $\mathcal F_t$-measurable set of unresolved tasks at round start, with $P_{t+1}\subseteq P_t$ while the phase continues; set $P_t=\varnothing$ outside the phase.

Define the predictable indicator \(A_t=\ind\{\sigma\le t<\vartheta, \ P_t\ne\varnothing\}\).
Let $X_{i,t}\in\{0,1\}$ indicate a \emph{service observation} for task $i$, observable by round end and zero if the slot or selected player cannot serve $i$. In our applications it is a positive reward observation, contributing a decoder equation or completing a confirmation or release test. Write
\[
 J_{i,t}=\ind\{i\in P_t\},\qquad
 N_i(v)=\sum_{t=1}^vJ_{i,t}X_{i,t}.
\]

A positive integer \emph{quota} $q$ requires $J_{i,t}=1\Longrightarrow N_i(t-1)<q$.

\begin{lemma}[Cost of parallel service]\label{sq:lem:waiting}
For a phase as above, fix an integer $q\ge1$, $a>0$, and
$\mathcal F_t$-measurable $\omega_t\in[0,K]$. Suppose the quota $q$
condition holds and, on $\{J_{i,t}=1\}$,
\[
 \E[X_{i,t}\mid\mathcal F_t]\ge a\omega_t
\]
for every $i\in[d]$ and $t\le T$. Then each task satisfies
\[
 \E\sum_{t=1}^T J_{i,t}\omega_t\le\frac qa.
\]
The total phase charge satisfies
\begin{equation}\label{sq:eq:parallel}
 \E\sum_{t=1}^T A_t\omega_t
 \le\frac{q+\log d}{(1-e^{-1})a}.
\end{equation}
Both bounds also hold conditional on the phase-start history
$\mathcal F_\sigma$.
\end{lemma}
\begin{proof}
Set $G_i(v)=\sum_{t=1}^vJ_{i,t}\omega_t$. The quota prevents counting after the $q$-th observation, so $N_i(T)\le q$; predictability gives
\[
 a\E G_i(T)\le\E N_i(T)\le q.
\]

For the total charge, put $\kappa=1-e^{-1}$ and define
\[
 Z_i(v)=\exp\!\left(\kappa aG_i(v)-N_i(v)\right).
\]
On $\{J_{i,t}=1\}$, since $X_{i,t}\in \{0,1\}$,
\[
 \E[e^{-X_{i,t}}\mid\mathcal F_t]
 =1-\kappa\E[X_{i,t}\mid\mathcal F_t]
 \le e^{-\kappa a\omega_t}.
\]
On $\{J_{i,t}=0\}$, $Z_i(t)=Z_i(t-1)$. So $(Z_i(v))_{v=0}^T$ is a nonnegative supermartingale with $Z_i(0)=1$.

Write $G=\sum_{t=1}^T A_t\omega_t$. Pathwise, $G=\max_{i\in[d]}G_i(T)$: a task in $P_t$ at the last positively charged round belonged to every earlier charged set because $P_t$ shrinks, while $J_{i,t}\le A_t$ bounds each task's charge by $G$. Both sides vanish if no round is charged.

Since $N_i(T)\le q$ and $\E Z_i(T)\le1$,
\[
 \E e^{\kappa aG}
 \le\sum_{i=1}^d\E e^{\kappa aG_i(T)}
 \le d e^q.
\]
Jensen's inequality yields $\kappa a\E G\le q+\log d$.
Since $J_{i,t}=0$ before $\sigma$, we have $G_i(\sigma-1)=N_i(\sigma-1)=0$. Starting the tower, supermartingale, and conditional Jensen arguments at $\mathcal F_\sigma$ proves the conditional bounds.
\end{proof}

Recall that $g_t=\sum_{a=1}^K\ind\{r_t(a)>0\}$ counts arms with positive rounded reward.

\begin{lemma}[Cost of assignment delivery]\label{sq:lem:handover}
Let an assignment start at round $\sigma$ with no earlier decoder failure or queue overflow. Let $G_{\rm del}$ be the sum of $g_t$ over this assignment's live rounds that have at least one missing confirmation. Truncate this sum after the first decoder failure or queue-overflow update, whichever occurs first. Then
\begin{equation*}
 \E[G_{\rm del}\mid\mathcal F_\sigma]
 \le \frac{8dK}{1-e^{-1}}(2Q_{\rm dec}+\log d+1)\,.
\end{equation*}
\end{lemma}

\begin{proof}
On the stopped interval defining $G_{\rm del}$, let $P_t$ be the followers whose confirmations are missing at round start; set $P_t=\varnothing$ outside this interval. For $i\in P_t$, define
\begin{equation*}
 X_{i,t}
 =\ind\{\tau_t=\mathrm{DOWN},\,k_t=i\}
 \begin{cases}
  \ind\{Y_{i,t}>0\},
    &\text{before acceptance of its DOWN word},\\
  \ind\{H_t(M_i^{\rm DOWN})=0,\,Y_{0,t}>0\},
    &\text{after acceptance}.
 \end{cases}
\end{equation*}
Set $X_{i,t}=0$ when $i\notin P_t$. 

On a DOWN round selecting $i$, Lemma~\ref{sq:lem:avoidance} shows that service occurs exactly when $r_t(R_i)>0$ and
$H_t(M_i^{\rm DOWN})=1$ before acceptance, or
$H_t(M_i^{\rm DOWN})=0$ afterward. The selected follower, receiver arm, and affine map have their independent prescribed laws. Hence, on $\{i\in P_t\}$,
\begin{equation}\label{sq:eq:service}
 \Pr(X_{i,t}=1\mid\mathcal F_t)
 =\frac{1}{4d}\cdot\frac{g_t}{K}\cdot\frac12
 =\frac{g_t}{8dK}\,.
\end{equation}

Each DOWN word remains fixed. Starting from any partial batch, Lemma~\ref{sq:lem:decoder} ensures acceptance within $2Q_{\rm dec}$ further positive DOWN observations, unless a decoder failure occurs first. After acceptance, one confirmation observation removes the follower from $P_t$. The failure round is included, and cutoff, overflow, or horizon truncation can only shorten the task. Thus each stopped task satisfies the quota $q=2Q_{\rm dec}+1$.

Apply the conditional conclusion of Lemma~\ref{sq:lem:waiting} with
\begin{equation*}
 \omega_t=g_t\in[0,K],
 \qquad a=\frac{1}{8dK},
 \qquad q=2Q_{\rm dec}+1\,.
\end{equation*}
Together with~\eqref{sq:eq:service}, it gives the claimed estimate
\begin{equation*}
 \E[G_{\rm del}\mid\mathcal F_\sigma]
 \le\frac{8dK}{1-e^{-1}}
       (2Q_{\rm dec}+1+\log d)\,.
\end{equation*}
\end{proof}

Write $\chi_t(j)=\ind\{r_t(j)>0\}$.

\begin{lemma}[Local release and its cost]\label{sq:lem:release}
After the last assignment confirmation, while player $i$ remains pending, its conditional UP release probability given $\mathcal F_t$ is at least $\chi_t(e)/(12dK)$.
Its conditional DATA release probability is $\chi_t(a_i)/(2K)$.

Use the coordinator's draw $X_t$ as the logical fast arm. Let $\mathcal T$ consist of the rounds after this assignment's last confirmation and before its cutoff, truncated at the horizon or immediately after the first decoder failure or queue-overflow update, including that round. If no post-confirmation live round occurs, put $\mathcal T=\varnothing$. Set $\sigma_{\rm R}=\min(\mathcal T\cup\{T+1\})$. Then
\begin{equation}\label{sq:eq:release-cost}
\begin{aligned}
 \E\!\left[\sum_{t\in\mathcal T}
 \left(r_t(U)+r_t(X_t)-\sum_{i=0}^dY_{i,t}\right)
 \,\middle|\,\mathcal F_{\sigma_{\rm R}}\right]
 \le2dK+\frac{12dK}{1-e^{-1}}(\log d+1)\,.
\end{aligned}
\end{equation}
\end{lemma}
\begin{proof}
If $\mathcal T=\varnothing$, the claim is immediate. Otherwise $\sigma_{\rm R}$ is a stopping time: the last confirmation occurred earlier, and whether the cutoff has occurred and whether any earlier failure has occurred are known before arm choice. Throughout this proof, every calculation is conditional on $\mathcal F_{\sigma_{\rm R}}$. At the last confirmation, every follower is occupied and none is released (Lemma~\ref{sq:lem:avoidance}), thereafter the coordinator plays $X_t\in F$.

On UP selecting pending $i$ with $R_0=e$, player $i$ plays $e$, other pending players play $a_j\ne e$, and released players stay on distinct assigned arms. Conditional on this slot and tuple event, $i$ receives a positive reward exactly when $r_t(e)>0$ and $X_t\ne e$. The independent fast sampler and Lemma~\ref{sq:lem:auxiliary-arm} give
\begin{align*}
\Pr(\tau_t=\mathrm{UP},k_t=i,R_0=e,Y_{i,t}>0\mid\mathcal F_t)
 =\frac1{4dK}(1-q^\gamma_{w_t}(e\mid U))\chi_t(e)
 \ge\frac{\chi_t(e)}{12dK}.
\end{align*}
On DATA, slow players use distinct assigned arms and the coordinator chooses from $F$, so $Y_{i,t}=r_t(a_i)$. Independently, $P_0=a_i$ has probability $1/K$, giving DATA release probability $\chi_t(a_i)/(2K)$.

Released followers receive $r_t(a_i)$, since assigned arms are distinct and $e\in F$. A pending follower uses $a_i$ or $e$ and receives a nonnegative reward, so its loss relative to $r_t(a_i)$ is at most $\chi_t(a_i)$. The coordinator can collide with a follower only when $X_t=e$ and some follower is pending. Hence
\begin{equation}\label{sq:eq:sentinelcost}
\begin{aligned}
 r_t(U)+r_t(X_t)-\sum_{i=0}^dY_{i,t}
 \le&\sum_{i\text{ pending}}\chi_t(a_i)
 +r_t(e)\ind\{X_t=e,\ \text{some follower pending}\}.
\end{aligned}
\end{equation}

For auxiliary-arm collisions, let
\[
 \omega_t=q^\gamma_{w_t}(e\mid U)\chi_t(e)\in[0,1].
\]
On $\mathcal T$, take $P_t=\{i:i\text{ is pending before round }t\}$, $A_t=\ind\{t\in\mathcal T,\ P_t\ne\varnothing\}$, and
\[
 X_{i,t}=\ind\{\tau_t=\mathrm{UP},\ k_t=i,\ R_0=e,\
                        Y_{i,t}>0\}.
\]
An UP release removes $i$ by the next round, with no reentry before cutoff, so the quota is one. Since $\omega_t\le\chi_t(e)$, the UP bound satisfies Lemma~\ref{sq:lem:waiting} with $a=1/(12dK)$. Its conditional conclusion gives
\begin{equation}\label{sq:eq:release-weighted}
 \E\!\left[\sum_{t\in\mathcal T}
       \ind\{P_t\ne\varnothing\}\omega_t
       \,\middle|\,\mathcal F_{\sigma_{\rm R}}\right]
 \le \frac{12dK}{1-e^{-1}}(\log d+1)\,.
\end{equation}
While follower $i$ is pending, its conditional DATA release probability is $\chi_t(a_i)/(2K)$. Since it can be released at most once on $\mathcal T$, the first bound of Lemma~\ref{sq:lem:waiting} gives
\begin{equation}\label{sq:eq:release-follower}
 \E\!\left[\sum_{t\in\mathcal T}\ind\{i\in P_t\}\chi_t(a_i)\,\middle|\,\mathcal F_{\sigma_{\rm R}}\right]\le2K\,.
\end{equation}

For the second term in~\eqref{sq:eq:sentinelcost}, the fast-arm draw has conditional law $q^\gamma_{w_t}(\cdot\mid U)$ and $r_t(e)\le\chi_t(e)$, so its conditional expected charge is at most $A_t\omega_t$.
Taking conditional expectations in~\eqref{sq:eq:sentinelcost}, summing over $t\in\mathcal T$, and using~\eqref{sq:eq:release-weighted},~\eqref{sq:eq:release-follower}, gives~\eqref{sq:eq:release-cost}.
\end{proof}

\subsection{Recording rewards}\label{sq:sec:recorder}

During a live assignment, follower $i$ owns $a_i$, and the coordinator owns $F$. After the cutoff, the coordinator owns all arms. An unoccupied follower logs candidates and determines its assigned arm's increments from them after decoding.

On DATA the row independently selects the recorder source: coordinator with probability $s_0=1/2$, or follower $i$ with probability $s_i=1/(2d)$. The fresh-row execution thins observations using an independent $V_t\sim\Unif([0,1))$. For owner $i$ of arm $j$, let $\pi_i(j)$ be its probability of playing $j$ conditional on $\mathcal F_t$ and DATA.

Set $\zeta=(1-1/K)^d$ and define the recording mask by
\begin{equation}\label{sq:eq:recording-mask}
 \mathsf{Mask}_{i,t}(j)
 :=\ind\{P_{\ell,t}\ne j\text{ for every }\ell\ne i\},
\end{equation}
where $P_{\ell,t}$ is player $\ell$'s pilot arm in round $t$ given by the public row.
The coordinator stores the latest assignment start $\sigma$, initially $0$, and next grid boundary $\bar\tau$, initially $T+1$. Its score updates are frozen exactly when $\sigma>0$ and $t\ge\bar\tau$.

For each follower $i\in[d]$, let $\mathcal B_i$ be its initially empty, time-ordered list of candidate observation records.

\begin{algorithm}[H]
\caption{Record a DATA increment or one candidate observation}
\label{sq:alg:recorder}
\begin{algorithmic}[1]
\Require Activated player $i$, player $i$'s arm $A_{i,t}$ and quantized observation $Y_{i,t}$.
\State Initialize the output increment to zero.
\State If the slot is not DATA, stop.
\State If the public recorder source is not $i$, stop.
\State $j\gets A_{i,t}$; read $V_t$.
\If{$i$ is an unoccupied follower}
 \If{$\mathsf{Mask}_{i,t}(j)=1$}
  \State Append $(t,j,Y_{i,t},V_t)$ to $\mathcal B_i$.
 \EndIf
 \State Stop.
\EndIf
\If{$i=0$, $\sigma>0$, and $t\ge\bar\tau$}
 \State Stop.
\EndIf
\If{$i$ does not own $j$ or $\mathsf{Mask}_{i,t}(j)=0$}
 \State Stop.
\EndIf
\If{$V_t\le \xi_j/((1/2)s_i\pi_i(j)\zeta)$}
 \State $\Delta_t(j)\gets\lfloor D\eta Y_{i,t}/\xi_j\rfloor$.
 \State \Return $\Delta_t(j)$ as a raw arrival to Algorithm~\ref{protocol:alg:order}.
\EndIf
\end{algorithmic}
\end{algorithm}

For analysis, on enabled rounds set $I_t=j$ when the round is DATA, the selected source owns its played arm~$j$, and the mask and thinning tests of Algorithm~\ref{sq:alg:recorder} pass; set $I_t=\bot$ otherwise. For an unoccupied follower, evaluate these tests on its saved record using the true assignment, even if it has not decoded by the horizon. On enabled rounds $t\le\tau_*$, use $I_t$ in~\eqref{sq:eq:marker}, with logical raw arrivals $\Delta_t=\eta D\widetilde r_t$ as in Section~\ref{sq:sec:learner}.

\begin{lemma}[Physical recording law]\label{sq:lem:recorder}
In the fresh-row execution, on every enabled round, the arm $I_t$ defined above
satisfies~\eqref{sq:eq:recordlaw}. On $I_t=j$, the owner's observation
is $r_t(j)$, so its increment is the one specified by~\eqref{sq:eq:marker}.
\end{lemma}
\begin{proof}
Since $d=m-1\le K-2$ and $\log(1-x)\ge-x/(1-x)$ for $0\le x<1$,
\[
 \zeta=(1-1/K)^d\ge\exp(-d/(K-1))\ge e^{-1}>\tfrac14.
\]
Conditional on $\mathcal F_t$, fix arm $j$ and owner $i$. Its DATA action is independent of other pilots, so the raw probability of DATA, source $i$, action $j$, and the mask is $\tfrac12s_i\pi_i(j)\zeta$.
The owner is alone on $j$: every nonowner using a pilot avoids $j$ by the mask, an assigned nonowner uses a different slow arm, and a fast coordinator chooses from $F$, outside all follower-owned arms.

For its assigned arm, an unoccupied follower has $\pi_i(j)=1/K$ and an occupied follower has $\pi_i(j)=1$. The coordinator has $\pi_0(j)=1/K$ during collection or while confirmations are missing, and $\pi_0(j)=q^\gamma_{w_t}(j\mid\nobreak U)$ for $j\in F$ after all confirmations. Therefore
\[
 \tfrac12s_i\pi_i(j)\zeta\ge
 \begin{cases}
  1/(16mK),&\text{pilot coordinator or slow owner},\\
  q^\gamma_{w_t}(j\mid U)/16,&\text{fast owner}.
 \end{cases}
\]
Both lower bounds are at least $\xi_j$: the first uses $\xi_j\le1/(16mK)$, and the second uses $q^\gamma_{w_t}(j\mid U)\ge16\xi_j$ from~\eqref{sq:eq:feasible}. Thus
\begin{equation}\label{sq:eq:thinning}
 0\le\frac{\xi_j}{(1/2)s_i\pi_i(j)\zeta}\le1\,.
\end{equation}
Multiplying by the raw probability gives $\Pr(I_t=j\mid\mathcal F_t)=\xi_j$, leaving probability $1-\sum_j\xi_j$ for $\bot$.
\end{proof}

For an unoccupied follower assigned $a_i$, the DATA, source, arm, and mask tests have probability $\zeta/(4dK)$. With DOWN rate exponent $h_i=-\log_2\xi_{a_i}$, the thinning threshold is therefore $4dK\,2^{-h_i}/\zeta$.
Define one saved observation's contribution, for $y,v\in[0,1]$ and every integer $h'\in\{0,\ldots,b_N\}$, by
\begin{equation}\label{protocol:eq:saved-contribution}
 \psi_{h'}(y,v)=\ind\!\left\{v\le\frac{4dK\,2^{-h'}}{\zeta}\right\}\left\lfloor D\eta\,2^{h'}y\right\rfloor\,.
\end{equation}
This quantity will be used later in the next section.

\section{Scheduling and local implementation}\label{sq:sec:clock}

At cutoff, followers may still be decoding the expired assignment. The coordinator continues DOWN delivery while collecting score and queue replies.

We work with the reference execution of Section~\ref{sq:sec:transport}. The structural identities below are proved on prefixes that pass the local domain checks of Appendix~\ref{protocol:fallback}, before any local operation enters the absorbing error state~$\dagger$. We call such a prefix \emph{locally valid}.

\subsection{The cutoff and the score recovery}

Recall the grid spacing $H=\lfloor(64\eta)^{-1}\rfloor$ and the proposal bits $B_t$. For an assignment started at round $\sigma$, put
\begin{align*}
 \bar\tau(\sigma)&=\min\{T+1,H\bigl(\lfloor\sigma/H\rfloor+1\bigr)\}\,,\\
 u&=\min\bigl(\{\bar\tau(\sigma)\}\cup\{s+1:\sigma\le s<T,\ B_s=1\}\bigr)\,.
\end{align*}
So $\bar\tau=\bar\tau(\sigma)$ is the next clipped grid boundary. For $u\le T$, the cutoff is the first \emph{public interruption}, applied before arm choice in round $u$. During collection, $\mathcal P$ lists followers whose current UP reply has not been accepted.

In this section, $\Delta_r,\Gamma_r,\mathsf B_r$ denote physical calendar-round arrivals, service, and queues from~\eqref{sq:eq:queue}, attributed to their original rounds even when computed later. For a collection interval $[u,t)$, put
\begin{equation*}
 c(t)=\sum_{r=u}^{t-1}\Gamma_r\,.
\end{equation*}
The stored accumulator $c$ may be partial until the missing queues are recovered at the end of the interval.
For each arm $j$, let $\Gamma_t^{\rm loc}(j)$ equal $\Gamma_t(j)$ if the coordinator computes its service before decoding in round $t$, and zero otherwise. During collection, these services are added to $c$ at the end of the round; services reconstructed after UP acceptance are included in $c$ separately.

At cutoff $u$, the coordinator knows $z_j(u)$ for each arm $j$ outside the old slow set. For each follower arm $a_i$, it retains the assignment's base score $z_{a_i}(\sigma)$.
Upon accepting follower $i$'s UP reply, it reconstructs
\[
 z_{a_i}(u)=z_{a_i}(\sigma)+S_i,
\]
where $S_i$ is the cumulative score increment that follower $i$ sends on the UP channel.

Recall the slow-set marginal $p_w$ with its frozen anchor, and $h=4m\eta$. For an assignment start $t<\bar\tau$ following cutoff $u$, define, for each $(m-1)$-set $U$,
\begin{align*}
 \kappa_t&=\frac{(1-h)^{t-u+1}}{h}\,,\\
 \nu_t(U)&=p_{w_t}(U)+\kappa_t\bigl(p_{w_u}(U)-p_{w_{u-1}}(U)\bigr)\,.
\end{align*}
This accounts for subsequent refreshes during collection.

\begin{algorithm}[H]
\caption{Coordinator updates before and after a round}\label{sq:alg:clock}
\begin{algorithmic}[1]
\Require The coordinator is active; initially $\sigma=0$ and $\bar\tau=T+1$.
\Procedure{StartRound}{$t$}
 \If{$\sigma>0$, $u=T+1$, and either $t=\bar\tau$ or $B_{t-1}=1$}
  \State $(u,\mathcal P,c)\gets(t,[d],0)$; initialize collection raw buffers.
  \State Preserve known entries of $z(u)$, $\mathsf B_u$, and $\Gamma_{u-1}$.
 \EndIf
 \If{$\sigma=0$ or $(u\le t\text{ and }\mathcal P=\varnothing)$}
  \If{$\sigma=0$}
   \State Set all scores and queues to $0$, and all weights to $1$.
  \Else
   \State $z(t)\gets z(u)+c$; use $\mathsf B_t$ as the new assignment's initial queues.
  \EndIf
  \If{$\sigma=0$ or $t\ge\bar\tau$}
   \State Renew anchor and rates by~\eqref{sq:eq:rates}; draw $U\sim p_{w_t}$.
  \Else
   \State Recover $w_{u-1},w_u$ by Lemma~\ref{sq:lem:reconstruct}; draw $U\sim\nu_t$.
  \EndIf
  \State Write $U=\{a_1<\cdots<a_d\}$.
  \State Choose $e\in\arg\min_{j\notin U}\bar w_j$, breaking ties by label.
  \State Retain base scores $z_{a_i}(t)$ for all $i\in[d]$.
  \State Construct $(t,a_i,e,-\log_2\xi_{a_i},\mathsf B_t(a_i))$ for each $i\in[d]$.
  \State $(\sigma,u,\bar\tau,\mathcal P)\gets(t,T+1,\bar\tau(t),[d])$.
 \EndIf
\EndProcedure
\Procedure{FinishRound}{$t,\Gamma_t^{\rm loc}$}
 \If{$u\le t$}
  \State $c\gets c+\Gamma_t^{\rm loc}$.
 \Else
  \State Read the live assignment's proposal bit $B_t$.
 \EndIf
\EndProcedure
\end{algorithmic}
\end{algorithm}

For the reference execution, define the calendar-round score by
\begin{equation*}
 Z_j(r):=\sum_{s<r}\Gamma_s(j),\qquad j\in[K]\,.
\end{equation*}
An increment computed after decoding is attributed to its original service round.

\begin{lemma}[Recovery from accumulated increments]
\label{sq:lem:reconstruct}
Suppose an assignment started at $\sigma$ has cutoff $u$, and its
collection ends with an assignment start at $t$. Assume the coordinator
retained $z(\sigma)=Z(\sigma)$ and the exact queues at $\sigma$ and,
for every $i\in[d]$, obtains
$S_i(u)=\sum_{s=\sigma}^{u-1}\Gamma_s(a_i)$,
$\Gamma_{u-1}(a_i)$, and $\mathsf B_u(a_i)$.
These values and the coordinator's buffered raw arrivals determine
\begin{equation}\label{sq:eq:reconstruct-score}
 Z(t)=Z(u)+c(t),\qquad Z(u-1)=Z(u)-\Gamma_{u-1}.
\end{equation}
At this assignment start $t$, the coordinator's reconstructed register vector
satisfies $z(t)=Z(t)$, and every current queue is recovered exactly.
Consequently the coordinator can compute $w_{u-1},w_u,w_t$ without
retaining the intermediate collection score vectors.
\end{lemma}
\begin{proof}
The returned $S_i$ and the retained base scores recover $Z_{a_i}(u)$. The coordinator's own service gives the other cutoff scores, so $z(u)=Z(u)$. Its DATA observations and old recording rates determine collection arrivals without these scores or the missing queues. Starting from the returned cutoff queues on follower arms and the known cutoff queues on the other arms, apply \eqref{sq:eq:queue} successively for $r=u,\cdots,t-1$. For $r<\bar\tau$, use the buffered arrivals, taking $\Delta_r(j)=0$ whenever no arrival was recorded for arm $j$. Queue service still runs on these zero-arrival rounds. For $r\ge\bar\tau$, set $\Delta_r=\Gamma_r=0$ and preserve the queues. This recovers all services and terminal queues, with $c(t)=Z(t)-Z(u)$. Each returned queue is reconstructed through its UP-acceptance round and updated online thereafter. Thus, at the next assignment start $t$, the coordinator knows all queues $\mathsf B_t$ and the full service sum $c(t)$.

The returned $\Gamma_{u-1}(a_i)$ and retained coordinator services recover $\Gamma_{u-1}$. Subtracting $\Gamma_{u-1}$ from $Z(u)$ gives $Z(u-1)$; Algorithm~\ref{sq:alg:clock} then sets $z(t)=Z(t)$. Now $w_{r,j}=(1+1/D)^{Z_j(r)}$ gives all three weight vectors.
\end{proof}

\paragraph{Who knows the queue?}
A queue may be well-defined before its owner can compute it. The protocol
uses the following information pattern; ``known'' means available from
that player's local registers, observations, and selected public row.
\begin{center}
\small
\begin{tabularx}{\linewidth}{@{}p{0.19\linewidth}XX@{}}
\toprule
Stage & Follower $i$ & Coordinator\\
\midrule
Assignment start $\sigma$ & Unoccupied; may not know its new arm. & Knows all scores and queues, selects $U$, and sends each initial queue.\\
Live, before DOWN acceptance & Logs candidate observations without knowing ownership or rate. & Knows $U$ and all outside scores/queues; retains the base score of each slow arm.\\
Live, after DOWN acceptance & Replays from $\sigma$, then serves its own queue online. & Uses outside scores for fast play; does not need current slow-arm scores.\\
After cutoff $u$ & Fixes the endpoint reply, or decodes late and reconstructs it. & Owns every arm; buffers arrivals for queues whose replies are missing.\\
After current UP acceptance & Remains unoccupied until the next DOWN assignment. & Replays that arm from $u$ through the acceptance round; thereafter serves it online.\\
Next assignment start $v$ & All followers are unoccupied. & Knows $Z(u-1),Z(u),Z(v)$ and every queue $\mathsf B_v$.\\
\bottomrule
\end{tabularx}
\end{center}

\paragraph{A timestamp example.}
Suppose $\sigma=10$, $u=13$, $\bar\tau>18$, and follower $i$ accepts its DOWN word after
round $15$. It recovers only rounds $10,11,12$, using the initial queue
from the word; candidate records from rounds $13,14,15$ do not enter its
reply. Its reply contains the service sum on $[10,13)$, the service in
round $12$, and the queue before round $13$. If the coordinator accepts
this reply after round $18$, it uses its own buffered arrivals to replay
rounds $13,\ldots,18$. The service in round $18$ is included in that
replay, so the end-of-round online addition for this arm is zero.

\subsection{Local records and acceptance}
\label{sq:sec:local}\label{sq:sec:encoding}

A score reply must identify its assignment so that the coordinator can match it to the collection in progress.

Follower $i$ retains the last accepted DOWN record $(\sigma_i,a_i,e_i,h_i,q_i)$, the served sum $S_i$, last calendar-round service $\delta_i$, current queue $\mathsf b_i$, and state $\mathrm{state}_i\in\{0,1,2\}$, meaning unoccupied, pending, and released, respectively. Here $\xi_{a_i}=2^{-h_i}$ and $q_i$ is the accepted initial queue. Initially $(\sigma_i,S_i,\delta_i,\mathsf b_i,\mathrm{state}_i)=(0,0,0,0,0)$; other DOWN fields are unused until acceptance. An interruption changes states $1,2$ to $0$ and preserves $(\sigma_i,S_i,\delta_i,\mathsf b_i)$.

At arm choice on an active round with $\tau_t=\mathrm{UP}$, $k_t=i$, and $\mathrm{state}_i=0$, follower $i$ offers
\begin{equation}\label{protocol:eq:up-offer}
 (\sigma_i,S_i,\delta_i,\mathsf b_i)\,.
\end{equation}
Otherwise its offer is $\varnothing$. While unoccupied, these four fields stay fixed until a newer DOWN acceptance.

For an assignment $(\sigma,a_i)$ with cutoff $u$, define follower $i$'s \emph{endpoint UP record} by
\begin{equation}\label{sq:eq:endpoint-record}
 M_i^{\rm UP}
 =\left(\sigma,\sum_{s=\sigma}^{u-1}\Gamma_s(a_i),
                  \Gamma_{u-1}(a_i),\mathsf B_u(a_i)\right)\,.
\end{equation}

Let $L=\lceil\log_2(TD+1)\rceil$ be the field width; recall $B=5L$. Let $\operatorname{bin}_L:\{0,\ldots,2^L-1\}\to\mathbb F_2^L$ be the binary decomposition. Writing $\Vert$ for concatenation and $0^L$ for the all-zero $L$-bit block, encode the record types by
\begin{equation}\label{protocol:eq:word-encoding}
\begin{aligned}
 \operatorname{Enc}_{\rm DOWN}(\sigma,a,e,h,q)
 &=\operatorname{bin}_L(\sigma)\Vert\operatorname{bin}_L(a)
   \Vert\operatorname{bin}_L(e)\Vert\operatorname{bin}_L(h)\Vert\operatorname{bin}_L(q),\\
 \operatorname{Enc}_{\rm UP}(\sigma,S,\delta,q)
 &=\operatorname{bin}_L(\sigma)\Vert\operatorname{bin}_L(S)
   \Vert\operatorname{bin}_L(\delta)\Vert\operatorname{bin}_L(q)\Vert 0^L.
\end{aligned}
\end{equation}

Define the admissible record sets by
\[
\begin{aligned}
 \mathcal M_{\rm DOWN}
 =\bigl\{(\sigma,a,e,h,q):{}&
   \sigma\in[T],\ a,e\in[K],\ a\ne e,\\
 &h\in\{0,\ldots,b_N\},\
   \beta\le2^{-h}\le\tfrac{1}{16mK},\\
 &q\in\{0,\ldots,TD\}\bigr\},
\end{aligned}
\]
and
\[
 \mathcal M_{\rm UP}
 =\{(0,0,0,0)\}\cup
   \bigl([T]\times\{0,\ldots,TD\}\times\{0,\ldots,D\}
                      \times\{0,\ldots,TD\}\bigr).
\]
The slot selects the direction $\mathrm{dir}\in\{\mathrm{DOWN},\mathrm{UP}\}$. For $w\in\mathbb F_2^B\cup\{\bot\}$, define
\begin{equation*}
 \operatorname{Parse}_{\mathrm{dir}}(w)=
 \begin{cases}
  \operatorname{Enc}_{\mathrm{dir}}^{-1}(w),
    &w\in\operatorname{Enc}_{\mathrm{dir}}(\mathcal M_{\mathrm{dir}}),\\
  \bot,&\text{otherwise}.
 \end{cases}
\end{equation*}
For a parsed record $M$, write $M.\sigma$ for its first field and $M.S,M.\delta$ for the served-sum and last calendar-round service fields of an UP record. The queue accessor $M.q$ gives the initial queue on DOWN and the terminal queue on UP.

At a full-batch decoder decision after round $t$, let $M$ be the
parsed DOWN or UP record, or $\bot$. The local acceptance tests are
\begin{equation}\label{protocol:eq:acceptance}
\begin{aligned}
 \operatorname{Acc}_{i,t}^{\rm DOWN}(M)=1
 &\ \Longleftrightarrow\
   M\in\mathcal M_{\rm DOWN},\quad \sigma_i<M.\sigma\le t,\\
 \operatorname{Acc}_{i,t}^{\rm UP}(M)=1
 &\ \Longleftrightarrow\
   M\in\mathcal M_{\rm UP},\quad u\le t,\quad
   i\in\mathcal P,\quad M.\sigma=\sigma.
\end{aligned}
\end{equation}
Parsing checks field domains and acceptance checks receiver state; passing these checks does not guarantee that the word was sent.

\begin{lemma}[Message field bounds]\label{protocol:lem:fields}
On every locally valid prefix of Algorithms~\ref{sq:alg:clock}
and~\ref{protocol:alg:order}, the outgoing DOWN records belong to
$\mathcal M_{\rm DOWN}$, and the nonempty offered UP records belong to
$\mathcal M_{\rm UP}$. Every field lies in $[0,TD]$, so
\eqref{protocol:eq:word-encoding} is valid.
\end{lemma}
\begin{proof}
An assignment start lies in $[T]$, and its distinct arms $a_i,e$ lie in $[K]\subseteq[0,D]$. By~\eqref{sq:eq:rates} and~\eqref{sq:eq:feasible}, $\xi_{a_i}=2^{-h_i}$ is dyadic with $\beta\le\xi_{a_i}\le1/(16mK)$. Since $\eta D$ is a positive integer, $\xi_{a_i}\ge\beta\ge\eta\ge D^{-1}$, so the integer exponent satisfies $0\le h_i\le\log_2D=b_N\le D$. The checked local domains of Appendix~\ref{protocol:fallback} give $0\le\mathsf B_\sigma(a_i)\le TD$, so $M_i^{\rm DOWN}\in\mathcal M_{\rm DOWN}$.

The initial registers give the UP placeholder $(0,0,0,0)$; after a DOWN acceptance $\sigma_i\in[T]$. For admissible $h$ and $y,v\in[0,1]$,~\eqref{protocol:eq:saved-contribution} and $\beta\le2^{-h}$ give $0\le\psi_h(y,v)\le D\eta/\beta\le D$. Algorithm~\ref{sq:alg:recorder} gives the same $[0,D]$ bound for online arrivals, and each integer service lies in $[0,4\eta D]\subseteq[0,D]$. The served sum covers at most $T$ rounds, so $0\le S_i\le TD$ and $0\le\delta_i\le D$. The same local checks give $0\le\mathsf b_i\le TD$, so the UP record lies in $\mathcal M_{\rm UP}$. Since $T\ge1$, every field lies in $[0,TD]$.
\end{proof}

\subsection{Delayed increments from a sparse observation list}
\label{protocol:sparse-records}

On an active DATA round $s$ while unoccupied, follower $i$ appends $(s,P_{i,s},Y_{i,s},V_s)$ to $\mathcal B_i$ exactly when Algorithm~\ref{sq:alg:recorder}'s source and mask tests pass in its declared public row; $Y_{i,s}$ is its quantized observation. Logging neither tests $Y_{i,s}>0$ nor applies thinning, and continues before learning launch and during frozen collection.

After accepting $M=(\sigma,a,e,h,q)$ on round $t$, the follower computes
the exclusive endpoint of the observations it must process:
\begin{equation}\label{protocol:eq:known-cutoff}
 \ell=\min\bigl(\{t+1,\bar\tau(\sigma)\}
          \cup\{s+1:\sigma\le s<t,\ B_s=1\}\bigr).
\end{equation}
Using the clipped boundary $\bar\tau(\sigma)$, the follower scans the public row's past bits $B_s$, $\sigma\le s<t$, stopping at the first $B_s=1$ or at $\bar\tau(\sigma)$. The follower had not read these bits before acceptance, and if $\ell\le t$, proposal coordinates in $[\ell,t)$ remain unread. Since $B_t$ acts only at the start of round $t+1$, $\ell=t+1$ means the assignment is live through round $t$, otherwise $\ell\le t$ is its first cutoff.
For each calendar round $r\in[\sigma,\ell)$, recover its raw arrival using $\psi_h$ from~\eqref{protocol:eq:saved-contribution}:
\begin{equation}\label{protocol:eq:sparse-sum}
 \Delta_r^{\rm rec}
 =\sum_{\substack{(s,j,y,v)\in\mathcal B_i\\s=r,\ j=a}}\psi_h(y,v).
\end{equation}
Starting from $q$ at $\sigma$, apply~\eqref{sq:eq:queue} with these raw arrivals on every round of $[\sigma,\ell)$, including rounds with zero arrival. Set $S_i$ to be the total service, $\delta_i$ to the service on round $\ell-1$, and $\mathsf b_i$ to the terminal queue. For $\ell=t+1$, subsequent rounds use online service, otherwise these three returned quantities already form the fixed cutoff state.

\begin{lemma}[Exact recovery from sparse records]
\label{protocol:lem:sparse}
On a reference execution,~\eqref{protocol:eq:sparse-sum} recovers every $\Delta_r(a)$ on $[\sigma,\ell)$, and
\[
 S_i=\sum_{s=\sigma}^{\ell-1}\Gamma_s(a),\qquad
 \delta_i=\Gamma_{\ell-1}(a),\qquad \mathsf b_i=\mathsf B_\ell(a).
\]
\end{lemma}
\begin{proof}
Fix a DOWN acceptance after round $t$ of $(\sigma,a,e,h,q)$, with first cutoff $u$. Equation~\eqref{protocol:eq:known-cutoff} gives $\ell=\min\{u,t+1\}$. For $s\in[\sigma,\ell)$, player $i$ is active, unoccupied, and owns $a$. Define
\[
 E_s=\{s\text{ is DATA},\ \text{source}=i,\
       P_{i,s}=a,\ P_{j,s}\ne a\ (j\ne i)\}.
\]
Algorithm~\ref{sq:alg:recorder}'s mask and source tests give zero arrival on $E_s^c$. On $E_s$, the recorder’s tests and the unoccupied-owner rule~\eqref{protocol:eq:saved-contribution} give output $\psi_h(Y_{i,s},V_s)$. Hence pathwise,
\begin{equation}\label{protocol:eq:sparse-pathwise}
 \Delta_s(a)=\ind\{E_s\}\psi_h(Y_{i,s},V_s)
 \qquad(\sigma\le s<\ell).
\end{equation}
The logging test for arm $a$ is exactly $E_s$, thus~\eqref{protocol:eq:sparse-pathwise} gives~\eqref{protocol:eq:sparse-sum}. 

Starting from the DOWN field $q=\mathsf B_\sigma(a)$, induction over~\eqref{sq:eq:queue} on every round recovers all services and queues through $\ell$.
\end{proof}

\subsection{One local event rule}\label{protocol:events}

In the absorbing local error state $\dagger$, a player performs no learning or transport updates and plays arm $i+1$ on active rounds. It enters $\dagger$ if a local operation has an input or computed output outside its prescribed domain; Appendix~\ref{protocol:spec} defines the checked operations.

Algorithm~\ref{protocol:alg:order} specifies each player's operations within one round, combining arm selection, recording, queue updates, and message decoding. The coordinator also applies Algorithm~\ref{sq:alg:clock} at the beginning and end of the round.

\begin{algorithm}[p]
\caption{One round of the local protocol at player $i$}
\label{protocol:alg:order}
\normalsize
\begin{algorithmic}[1]
\If{$i=0$, active, and not in $\dagger$}
 \State Apply $\operatorname{StartRound}(t)$ of Algorithm~\ref{sq:alg:clock}.
\ElsIf{$i>0$, not in $\dagger$, $\mathrm{state}_i\in\{1,2\}$, and interrupted at $t$}
 \State $\mathrm{state}_i\gets0$; retain $S_i,\delta_i,\mathsf b_i$.
\EndIf
\If{the player is inactive}
 \State Play Algorithm~\ref{sq:alg:activation}'s bootstrap arm; observe the reward.
 \State Apply any bootstrap update and activation; \Return.
\EndIf
\State If in $\dagger$, play arm $i+1$, apply any bootstrap update, and \Return.
\State Play by Algorithm~\ref{sq:alg:handover}; observe and quantize the reward.
\State If $i=0$, $\Gamma_t^{\rm loc}\gets0$.
\State Apply Algorithm~\ref{sq:alg:recorder} once.
\If{$i>0$ and $\mathrm{state}_i\in\{1,2\}$ at arm choice}
 \State Serve its queue by~\eqref{sq:eq:queue}, including zero arrivals.
 \State $S_i\gets S_i+\Gamma_t(a_i)$; $\delta_i\gets\Gamma_t(a_i)$.
\ElsIf{$i=0$ and learning is enabled}
 \For{each owned arm $j$ with known queue}
  \State Serve $j$ by~\eqref{sq:eq:queue}; $\Gamma_t^{\rm loc}(j)\gets\Gamma_t(j)$.
  \State If outside collection, $z_j\gets z_j+\Gamma_t^{\rm loc}(j)$.
 \EndFor
 \State Save $(t,j,\Delta_t(j))$ for each $\Delta_t(j)>0$ with unknown queue.
\EndIf
\State If $i=0$, set known queues' working last services from $\Gamma_t^{\rm loc}$.
\If{DOWN selects $i>0$, unoccupied at arm choice}
 \State Feed its observation to $\operatorname{DecodeStep}$ (Algorithm~\ref{sq:alg:decoder}); parse the returned word as DOWN into $M$.
 \If{$\operatorname{Acc}_{i,t}^{\rm DOWN}(M)=1$}
  \State $(\sigma_i,a_i,e_i,h_i,q_i)\gets M$; compute $\ell$ by~\eqref{protocol:eq:known-cutoff}.
  \State From $q_i$, apply~\eqref{sq:eq:queue} on $[\sigma_i,\ell)$ using~\eqref{protocol:eq:sparse-sum}; recover $(S_i,\delta_i,\mathsf b_i)$.
  \State $\mathrm{state}_i\gets\ind\{\ell=t+1\}$; empty $\mathcal B_i$.
 \EndIf
\ElsIf{$i=0$, $u\le t$, and UP selects $j\in\mathcal P$}
 \State Feed its observation to $j$'s $\operatorname{DecodeStep}$; parse the returned word as UP into $M$.
 \If{$\operatorname{Acc}_{j,t}^{\rm UP}(M)=1$}
  \State $z_{a_j}(u)\gets z_{a_j}(\sigma)+M.S$; $\Gamma_{u-1}(a_j)\gets M.\delta$.
  \State $\mathsf B_u(a_j)\gets M.q$; apply~\eqref{sq:eq:queue} on every round $u,\ldots,t$, including zero arrivals and freezing at $\bar\tau$.
  \State $c_{a_j}\gets\sum_{r=u}^t\Gamma_r(a_j)$; retain queue $\mathsf B_{t+1}(a_j)$ and last service $\Gamma_t(a_j)$.
  \State $\mathcal P\gets\mathcal P\setminus\{j\}$.
 \EndIf
\EndIf
\If{$i=0$, $t<u$, and $\mathcal P$ was nonempty at arm choice}
 \If{DOWN selected $j$, $H_t(M_j^{\rm DOWN})=0$, and $Y_{0,t}>0$}
  \State $\mathcal P\gets\mathcal P\setminus\{j\}$.
 \EndIf
\ElsIf{$i>0$ was pending at arm choice}
 \State If UP selected $i$, $R_0=e_i$, and $Y_{i,t}>0$, set $\mathrm{state}_i\gets2$.
 \State If DATA, $P_0=a_i$, and $Y_{i,t}>0$, set $\mathrm{state}_i\gets2$.
\EndIf
\State If $i=0$, apply $\operatorname{FinishRound}(t,\Gamma_t^{\rm loc})$ of Algorithm~\ref{sq:alg:clock}.
\State If $i>0$ and $\mathrm{state}_i\in\{1,2\}$, read $B_t$ after the update.
\State Apply any bootstrap update and activation.
\end{algorithmic}
\end{algorithm}
\FloatBarrier

\Needspace{14\baselineskip}
\begin{lemma}[Local state and sum invariants]
\label{protocol:lem:local-state}
Consider the reference execution immediately after round $t$.
If follower $i$ has not applied a DOWN word, then
$\mathrm{state}_i=\sigma_i=S_i=\delta_i=\mathsf b_i=0$.
Otherwise let $\sigma=\sigma_i$ be its last applied assignment and
let $u$ be that assignment's first cutoff, with $u=T+1$ if no cutoff
occurs in the horizon. Then
\begin{align*}
 S_i=&\sum_{s=\sigma}^{\min\{u,t+1\}-1}\Gamma_s(a_i),\qquad
 \delta_i=\Gamma_{\min\{u,t+1\}-1}(a_i),\qquad
 \mathsf b_i=\mathsf B_{\min\{u,t+1\}}(a_i).
\end{align*}
The state is $0$ if and only if $u\le t$.
When $t<u$, the state is $2$ if a release observation has occurred
since this DOWN acceptance, and is $1$ otherwise.
\end{lemma}
\begin{proof}
Put $L_t=\min\{u,t+1\}$ and induct over rounds, restarting at each DOWN acceptance. Before the first acceptance, Algorithm~\ref{protocol:alg:order} keeps all five registers zero.

At acceptance after round $t$, the cutoff scan gives $\ell=L_t$, and Lemma~\ref{protocol:lem:sparse} gives the three register identities. The follower was unoccupied at arm choice, so no release test occurs; setting $\mathrm{state}_i=\ind\{L_t=t+1\}$ gives the state assertion.

With no new acceptance at $t<u$, the queue rule adds $\Gamma_t(a_i)$ to $S_i$, sets $\delta_i=\Gamma_t(a_i)$ and $\mathsf b_i=\mathsf B_{t+1}(a_i)$, and changes state $1$ to $2$ precisely upon a release observation. At $t=u$, the interruption sets the state to $0$ and preserves $S_i,\delta_i,\mathsf b_i$; later rounds preserve them. Since $L_t=u$ for $t\ge u$, the identities persist.
\end{proof}

\begin{lemma}[Assignment identities and fixed replies]
\label{protocol:lem:identities}
In the reference execution, every follower applies each assignment
at most once. Each assignment has one cutoff and one UP record per
follower. After its cutoff is known, the UP record stays fixed
until the next DOWN word is applied. During a known live assignment
no UP message is encoded. No later assignment is started before
all current UP replies have been accepted.
When a known live follower becomes unoccupied, its DOWN batch is
empty. At each assignment's first cutoff, every coordinator UP
batch is empty. An interruption that leaves a follower unoccupied preserves its existing partial DOWN batch.
\end{lemma}
\begin{proof}
By~\eqref{protocol:eq:acceptance}, successive DOWN acceptances require strictly increasing starts, so each assignment is applied at most once.

Fix $(\sigma,a_i)$ with first cutoff $u$. For acceptance before $u$, Lemma~\ref{protocol:lem:local-state} at the end of round $u-1$ gives $(\sigma_i,S_i,\delta_i,\mathsf b_i)=M_i^{\rm UP}$. For acceptance after round $t\ge u$,~\eqref{protocol:eq:known-cutoff} gives $\ell=u$, and Lemma~\ref{protocol:lem:sparse} gives the same fields. The state becomes $0$ at cutoff or acceptance; Algorithm~\ref{protocol:alg:order} preserves the fields until a newer DOWN acceptance. Thus the UP record is unique and fixed. In states $1$ and $2$, Algorithm~\ref{sq:alg:handover} encodes no UP word.

At cutoff, $\mathcal P=[d]$. An accepted UP word removes $i$ only if $i\in\mathcal P$ and its start equals $\sigma$; in the reference execution it is $M_i^{\rm UP}$. Algorithm~\ref{sq:alg:clock} waits for $\mathcal P=\varnothing$, so every current UP record is accepted before the next assignment.

Let $W_i^{\rm D},W_i^{\rm U}$ be the initially empty equation batches. DOWN acceptance empties $W_i^{\rm D}$. If acceptance precedes cutoff, the decoder stays disabled until that interruption, leaving its batch empty then. If cutoff precedes acceptance, the interruption performs no decoder step and preserves the partial batch.

UP acceptance empties $W_i^{\rm U}$ and removes $i$ from $\mathcal P$. Its UP decoder stays disabled until the next assignment's first cutoff, so its batch is empty at that interruption. Since every $i$ is removed before the next assignment, this proves the UP assertion.
\end{proof}

\subsection{Coupling to the logical learner}

The next lemma is an exact simulation statement, not merely a comparison of expected rewards. The physical execution may compute a queue update late, and its assigned arms stay fixed while communication takes place; the logical learner has an updated score vector and slow set in every enabled round. We show that the two score processes agree when every update is attributed to the round that produced it, and that the slow sets agree on live rounds. During collection they need not agree. A \emph{coupling} is a joint construction of two random executions that leaves the law of each unchanged; it gives no player extra information.

Recall the calendar-round queues $\mathsf B_r$ and scores $Z_j(r)=\sum_{s<r}\Gamma_s(j)$, including updates computed later from stored observations. Define the physical overflow time
\begin{equation*}
 \tau_*^{\rm phys}
 :=\inf\left\{r\in[T]:\max_j\frac{\mathsf B_{r+1}(j)}{\eta D}>B_{\rm cap}\right\},
 \qquad B_{\rm cap}=\frac1{32\eta},\quad\inf\varnothing=+\infty\,.
\end{equation*}
It is used only in the analysis; the physical protocol continues after overflow. The symbol $\tau_*$ keeps its meaning as the logical overflow time until the coupling identifies the two.

At an assignment start $t$, let $V^{\rm asg}_t\in[0,1)^{K+1}$ be the coordinates $\mathrm{assign}_1,\ldots,\mathrm{assign}_{K+1}$ of the selected row at time $t$ (Appendix~\ref{comm:control-rows}), and let $\Phi_t$ be the coordinator's sampler, which maps $V^{\rm asg}_t$ to the slow set. Let $\mathcal G_t^-$ be generated by the physical local registers, actions, observations, private draws and row coordinates inspected before the sampler, together with all raw arrivals $(\Delta_s)_{s<t}$. Let $\mathcal G_t^+\supseteq\mathcal G_t^-$ also contain every physical row coordinate with calendar index less than $t$, whether read or not. Given $\mathcal G_t^-$, all parameters of $\Phi_t$ are fixed. With $\sigma$ and $u$ the start and cutoff of the preceding assignment, define the \emph{deferred law}
\begin{equation}\label{sq:eq:deferred-law}
 \bar\nu_t:=
 \begin{cases}
  \nu_t=p_t+\kappa_t\bigl(p_u-p_{u-1}\bigr),&t<\bar\tau(\sigma),\\
  p_t,&\text{$t$ is the first assignment start, or }t\ge\bar\tau(\sigma),
 \end{cases}
\end{equation}
where $\kappa_t=(1-h)^{t-u+1}/h$ and $p_r=p_{w_r}$. The first case uses the old anchor; in the second, $p_t$ uses the renewed anchor.

\begin{lemma}[Exact simulation and deferred assignment law]\label{sq:lem:lazy}
Run the fresh-row reference physical protocol and Algorithm~\ref{sq:alg:learner} with the same launch and episode schedule, starting with zero scores and queues. Write $\widehat U_r$, $\widehat z(r)$ and $\widehat{\mathsf B}_r$ for the logical slow set, scores and queues. There are a probability space carrying both executions and a filtration $(\mathcal F_r)$ of joint pre-round histories such that the following hold.
\begin{enumerate}[label=(\roman*),leftmargin=*,itemsep=1pt]
\item The physical execution has its prescribed law; in particular, all coordinates of the selected row are independent with their prescribed laws.
\item The logical execution is a realization of Algorithm~\ref{sq:alg:learner} that satisfies (I1)--(I4) of Section~\ref{sq:sec:learner} with respect to $(\mathcal F_r)$.
\item Almost surely, $\tau_*^{\rm phys}=\tau_*$, and $(Z(r),\mathsf B_r)=(\widehat z(r),\widehat{\mathsf B}_r)$ from launch for $r\le\min\{\tau_*,T\}+1$. On live rounds $r\le\tau_*$ the assigned physical slow set is $\widehat U_r$, and every physical fast-arm draw is the logical fast arm.
\item At every assignment start $t\le\tau_*^{\rm phys}$ and for every $(m-1)$-set $U$,
\[
 \Pr(\widehat U_t=U\mid\mathcal G_t^-)=\Pr(\widehat U_t=U\mid\mathcal G_t^+)=\bar\nu_t(U).
\]
\end{enumerate}
\end{lemma}
\begin{proof}
\emph{Step 1: deterministic prerequisites.} Every service increment lies in $[0,4\eta D]$, so each weight grows by a factor at most $e^{4\eta}$ per enabled round and by less than $2$ within an episode. The feasibility and stability inequalities of Section~\ref{sq:sec:learner} therefore hold for the physical calendar-round scores without any probabilistic input, and Lemma~\ref{protocol:lem:total} shows that every local operation is defined. Its treatment of the sampler uses only the deterministic inequality $Q_{w_{r+1}}\ge(1-h)Q_{w_r}$, not the conclusion of the present lemma. An update computed late is attributed to its original round, so $(Z(r),\mathsf B_r)$ is a well-defined process even when some entries are not yet known to any player; Lemmas~\ref{sq:lem:reconstruct} and~\ref{protocol:lem:sparse} show that the players' reconstructions recover it exactly.

\emph{Step 2: the construction.} Generate all variables in calendar order. The event $\{r\le\tau_*\}$ depends only on the logical queues before round~$r$. In a round $r\le\tau_*$:
\begin{enumerate}[label=(\alph*),leftmargin=*,itemsep=1pt]
\item If $r$ is an assignment start that also starts an episode, draw $\widehat U_r\sim p_{w_r}$ from fresh auxiliary randomness; at any other assignment start $\widehat U_r$ is already defined. Then, with fresh randomness, draw $V^{\rm asg}_r$ from the kernel
\[
 \Pr\bigl(V^{\rm asg}_r\in C\mid\text{everything generated so far}\bigr)
 =\frac{\lambda(C\cap\Phi_r^{-1}(\widehat U_r))}{\bar\nu_r(\widehat U_r)},
\]
where $\lambda$ is Lebesgue measure on $[0,1)^{K+1}$; on the event $\bar\nu_r(\widehat U_r)=0$, which has probability zero by~(iv), use $\lambda$. The kernel depends only on $\widehat U_r$ and on $\mathcal G_r^-$, which fixes $\Phi_r$ and $\bar\nu_r$. For a one-dimensional sampler returning $A$ on $[0,\frac14)$ and $B$ on $[\frac14,1)$, this draws a label with probabilities $\frac14,\frac34$ and then a uniform point of its interval; the point is uniform and its image is the label.
\item Draw all other round-$r$ coordinates of the selected row and the round-$r$ private draws from their prescribed laws, independently of everything generated so far.
\item Play the round physically. The logical record and arrivals are the physical ones, defined through the true assignment as in Section~\ref{sq:sec:recorder}. The logical fast arm is the coordinator's fast-arm draw $X_r$ when it makes one, and otherwise a fresh draw from $q^\gamma_{w_r}(\cdot\mid\widehat U_r)$. Both executions update queues and scores by~\eqref{sq:eq:queue}.
\item On an enabled round, the logical proposal bit is the physical proposal coordinate $B_r$ if the round is live, and a fresh Bernoulli$(h)$ bit otherwise; on a disabled round it is zero. On a proposal, draw $\widehat U_{r+1}\sim\rho_r$ from fresh auxiliary randomness.
\end{enumerate}
In a round $r>\tau_*$, the logical learner follows the continuation of Algorithm~\ref{sq:alg:learner}, which uses no randomness, and all physical variables are drawn from their prescribed laws; at an assignment start, $V^{\rm asg}_r\sim\lambda$ independently of everything generated so far. Let $\mathcal F_r$ be generated by everything generated before step~(b) of round $r$; in particular $\widehat U_r$, and $V^{\rm asg}_r$ at an assignment start, are $\mathcal F_r$-measurable.

\emph{Step 3: the logical execution satisfies (I1)--(I4).} Steps~3--7 form one induction over rounds $r\le\tau_*$: the agreement of Step~7 through the beginning of round~$r$ is used at round~$r$ in Steps~3--6, which then give the agreement through the beginning of round $r+1$. The slow set, scores, queues, rates and enabled status of round $r$ are $\mathcal F_r$-measurable by construction. (I1): episode starts are physical stopping times, each endpoint $\bar\tau$ is fixed at the start, and each initial draw is fresh. (I2): on a live round the assigned physical slow set is $\widehat U_r$, and during collection the physical actions do not depend on $\widehat U_r$, because the coordinator owns every arm and plays its pilot while all followers are unoccupied. The computation in Lemma~\ref{sq:lem:recorder} uses only the fresh slot, recorder, pilot, fast-arm and thinning coordinates of step~(b), so it holds conditionally on $\mathcal F_r$: $\Pr(I_r=j\mid\mathcal F_r)=\xi_j$ on enabled rounds. (I3): the coordinator draws $X_r$ from $q^\gamma_{w_r}(\cdot\mid U)$ with its fresh fast-arm coordinate and its exact outside scores (Lemma~\ref{protocol:lem:total}, invariant~(ii)), which are the logical ones by the induction hypothesis; otherwise the draw is auxiliary. (I4): the proposal coordinate is separate from all coordinates used in the round, and the other proposal bits and all refreshes are auxiliary.

Algorithm~\ref{sq:alg:learner} prescribes the conditional laws of the record and of the fast arm separately and allows them to be dependent. On a live round with all confirmations, for instance, writing $q_j=q^\gamma_{w_r}(j\mid U)$ for a fast arm $j\notin U$,
\[
 \Pr(j_r=j,I_r=i\mid\mathcal F_r)=
 \begin{cases}
 q_j\xi_i,&i\in U,\\
 \xi_i\ind\{j=i\},&i\notin U,
 \end{cases}
 \qquad
 \Pr(j_r=j,I_r=\bot\mid\mathcal F_r)
 =q_j\Bigl(1-\sum_{i\in U}\xi_i\Bigr)-\xi_j\ge0,
\]
where nonnegativity follows from $q_j\ge16\xi_j$ and $\sum_i\xi_i\le1/(16m)$. Hence Lemmas~\ref{sq:lem:queue-bound} and~\ref{sq:lem:conditional-law} and Proposition~\ref{sq:prop:learning} apply to the logical execution.

\emph{Step 4: the deferred law.} Fix an assignment start $t\le\tau_*^{\rm phys}$. At the first start, and when $t\ge\bar\tau(\sigma)$, the logical slow set is a fresh draw from $p_t$, which proves~(iv). Suppose $t<\bar\tau(\sigma)$. Then the cutoff $u$ was caused by $B_{u-1}=1$, since the grid boundary has not arrived. Work on the event that fixes the calendar values of $u$ and~$t$; it is determined by the physical execution, as is $\{t\le\tau_*^{\rm phys}\}$.

The collection path on $[u,t)$ consists of the arms, the recordings, the DOWN and UP decoding, and the stopping time~$t$. It is a function of the old assignment, physical coordinates and private draws, and no assignment start occurs in $(u,t)$. The old assignment is the logical slow set \emph{before} the refresh at $u-1$, and that refresh is a fresh draw from $\rho_{u-1}$, independent of it. The proposals and refreshes during collection are auxiliary. Hence the auxiliary randomness used from the refresh at $u-1$ on is independent of $\mathcal G_t^+$, while $\rho_r$ is $\mathcal G_t^+$-measurable for $u-1\le r<t$. Writing $\mu_r=\Pr(\widehat U_r\in\cdot\mid\mathcal G_t^+)$ for $u\le r\le t$, we get
\[
 \mu_u=\rho_{u-1},\qquad \mu_{r+1}=(1-h)\mu_r+h\rho_r\quad(u\le r<t).
\]
Thus $\mu_t$ is a convex combination of $\rho_{u-1},\ldots,\rho_{t-1}$, hence a probability law, even though its closed form below is a signed combination. Using $h\rho_r=p_{r+1}-(1-h)p_r$,
\[
 \mu_u-p_u=\frac{1-h}h(p_u-p_{u-1}),\qquad
 \mu_{r+1}-p_{r+1}=(1-h)(\mu_r-p_r),
\]
so $\mu_t=p_t+\kappa_t(p_u-p_{u-1})=\nu_t$; all these laws use the same anchor. Conditioning on the triggering proposal is what changes the initial collection law from $p_u$ to $\rho_{u-1}$. Finally, $\bar\nu_t$ is $\mathcal G_t^-$-measurable, because the raw arrivals in $\mathcal G_t^-$ determine the score vectors at $u-1,u,t$; Lemma~\ref{sq:lem:reconstruct} shows that the coordinator computes them. Since $\mathcal G_t^-\subseteq\mathcal G_t^+$, the tower property gives~(iv).

\emph{Step 5: the sampler realizes $\bar\nu_t$.} The sequential sampler (Section~\ref{sq:sec:complexity} and Appendix~\ref{protocol:fallback}, paragraph~5) considers the arms in order and includes the next arm exactly when $W_{\rm in}>VW$. Here $V$ is the next assignment coordinate, $W$ is the mass of the $m$-sets consistent with the choices made so far, and $W_{\rm in}$ is the mass of those that also contain the next arm. Each inclusion therefore has conditional probability $W_{\rm in}/W$, and the product of these ratios telescopes: the sampled $m$-set has law $f_t=Q_{w_t}+\kappa_t(Q_{w_u}-Q_{w_{u-1}})$ when $t<\bar\tau(\sigma)$, and $Q_{w_t}$ otherwise. Lemma~\ref{protocol:lem:total} shows that $f_t\ge0$, that $\sum_Sf_t(S)=1$, and that every normalizer is positive. The last coordinate then designates the fast arm $j\in S$ with probability $\alpha_j(S)$ from~\eqref{sq:eq:orientation}, which depends only on the anchor. Hence the law of the slow set is linear in the law of the $m$-set, and
\[
 \lambda\bigl(\Phi_t^{-1}(U)\bigr)
 =\sum_{j\notin U}f_t(U\cup\{j\})\,\alpha_j(U\cup\{j\})
 =p_t(U)+\kappa_t\bigl(p_u(U)-p_{u-1}(U)\bigr)=\bar\nu_t(U),
\]
and likewise $\lambda(\Phi_t^{-1}(U))=p_t(U)$ in the second case of~\eqref{sq:eq:deferred-law}.

\emph{Step 6: the physical law.} At an assignment start $t\le\tau_*$, (iv), Step~5 and the kernel of step~(a) give
\[
 \Pr\bigl(V^{\rm asg}_t\in C\mid\mathcal G_t^+\bigr)
 =\sum_{U:\bar\nu_t(U)>0}\bar\nu_t(U)\,
       \frac{\lambda(C\cap\Phi_t^{-1}(U))}{\bar\nu_t(U)}
 =\lambda(C).
\]
At an assignment start $t>\tau_*$ the same identity holds by construction, and by the induction $\{t\le\tau_*\}=\{t\le\tau_*^{\rm phys}\}$, an event in $\mathcal G_t^-$. Every other physical variable is generated freshly in step~(b) or after round $\tau_*$, and $\mathcal G_t^+$ contains every physical variable generated before $V^{\rm asg}_t$. By the chain rule, the physical execution, including the unread row coordinates, has its prescribed law. The vector $V^{\rm asg}_t$ is not independent of the logical set; that dependence is what realizes the coupling.

\emph{Step 7: agreement.} The construction gives $\Phi_t(V^{\rm asg}_t)=\widehat U_t$ almost surely at every assignment start $t\le\tau_*$, common records, arrivals and queue recurrences in every round $r\le\tau_*$, and common proposal bits on live rounds. Hence, if the scores and queues agree at the beginning of a round $r\le\tau_*$, then on a live round the slow sets and the fast arms agree, the cutoffs coincide, and the scores and queues agree at the beginning of round $r+1$. Common pre-round queues and arrivals give common post-round queues, so the first overflow rounds coincide; this proves $\tau_*^{\rm phys}=\tau_*$ and agreement through the beginning of round $\min\{\tau_*,T\}+1$. Afterwards the two executions are not compared, and each has its prescribed law by Step~2.

This construction is an analysis device. The coordinator simply applies $\Phi_t$ to its own row coordinates and never observes a logical set.
\end{proof}

\subsection{Score collection and total cost of interruptions}\label{sq:sec:scores}

In the fresh-row reference execution, fix the rounded reward matrix and a transfer starting at a stopping time $s$: a cutoff starts score collection, while the first assignment starts the initial transfer. Let the \emph{collection phase} $\mathcal C$ contain rounds from the cutoff through the round before the next assignment start. Let the \emph{delivery phase} $\mathcal D$ contain rounds from that assignment start through the last DOWN confirmation, or through the round before the next cutoff if a confirmation remains missing. Let the \emph{release phase} $\mathcal R$ contain the subsequent rounds before that cutoff. For the initial transfer $\mathcal C=\varnothing$ and $\mathcal D$ begins at $s$; any later phase is empty if its start does not occur. For analysis, truncate these phases immediately after the first decoder failure of Lemma~\ref{sq:lem:decoder} or queue-overflow update, including that round, or at the horizon.

Recall that $g_t$ counts arms with positive rounded reward and $\bar g_t=\min\{m,g_t\}$. A comparator's reward and any logical $m$-set's reward are at most $\bar g_t$.

With $d=m-1$ followers and the positive-observation batch size $Q_{\rm dec}$ from~\eqref{sq:eq:decode-quota}, define the phase bounds
\[
\begin{aligned}
 B_{\rm col}&:=\frac{8dK}{1-e^{-1}}(4Q_{\rm dec}+\log d),\\
 B_{\rm del}&:=\frac{8dK}{1-e^{-1}}(2Q_{\rm dec}+\log d+1),\\
 B_{\rm rel}&:=2dK+\frac{12dK}{1-e^{-1}}(\log d+1),
\end{aligned}
\]
and their sum, recalling $b_N=\lceil\log_2(N+1)\rceil$,
\begin{equation}\label{sq:eq:session}
\begin{aligned}
 C_{\rm tr}
 &:=B_{\rm col}+B_{\rm del}+B_{\rm rel}\\
 &=O\!\left(dK[Q_{\rm dec}+\log d+1]\right)
   =O(mKb_N).
\end{aligned}
\end{equation}

On release rounds, $X_t$ denotes the coordinator's fast-arm draw from Lemma~\ref{sq:lem:release}.

\paragraph{What happens during collection.}
The cutoff $u$ transfers ownership of \emph{every} arm to the coordinator, and all followers are unoccupied throughout $[u,t)$, where $t$ is the next assignment start: an occupied follower becomes unoccupied at the public interruption, and a follower that decodes the expired DOWN word after round $s\ge u$ finds $\ell=u$ in~\eqref{protocol:eq:known-cutoff} and stays unoccupied. Two kinds of information are recovered. The follower's fixed UP reply~\eqref{sq:eq:endpoint-record} describes its arm at the cutoff. The coordinator's own DATA observations give the raw arrivals after the cutoff; their recording probabilities use its pilot and the old rates, not the missing queue, so it buffers them and replays the queue once $\mathsf B_u(a_i)$ arrives. If the reply is accepted after round $s\ge u$, the replay covers rounds $u,\ldots,s$, including zero-arrival rounds and freezing at $\bar\tau$, and sets $c_{a_i}=\sum_{r=u}^{s}\Gamma_r(a_i)$. Since the queue was unknown before this decision, $\Gamma_s^{\rm loc}(a_i)=0$, so round $s$ is not counted twice. When all replies have arrived, $c_j=\sum_{r=u}^{t-1}\Gamma_r(j)$ and $z_j(t)=z_j(u)+c_j=Z_j(t)$ for every arm~$j$ (Lemma~\ref{protocol:lem:total}).

\begin{lemma}[Conditional cost of one transfer]
\label{sq:lem:collection}
Define the stopped transfer charge by
\begin{equation}\label{sq:eq:transfer-charge}
\begin{aligned}
 J_s
 :=\sum_{t\in\mathcal C\cup\mathcal D} \bar g_t
 +\sum_{t\in\mathcal R}
       \left(r_t(U_t)+r_t(X_t)-\sum_{i=0}^dY_{i,t}\right).
\end{aligned}
\end{equation}
Then conditional on the pre-round history
$\mathcal F_s$,
\[
 \E[J_s\mid\mathcal F_s]\le C_{\rm tr}.
\]
\end{lemma}
\begin{proof}
We stop after the first decoder failure or queue-overflow update, including that round. All probabilities below are conditional on pre-round histories, which contain all earlier raw arrivals; we never condition on the absence of future failures.

During collection, let $P_t$ list the followers whose current UP reply is unaccepted before round $t$. For $i\in P_t$, count a positive DOWN observation at $i$ while its old DOWN word is unaccepted, and a positive UP observation at the coordinator after that acceptance; if $i$ decoded before the cutoff, its task starts in the second stage. The stage changes only after a round's observations, so it is $\mathcal F_t$-measurable, and by the fixed-reply invariant of Lemma~\ref{protocol:lem:identities} it never reverts. A current UP reply can be accepted only after its follower accepts the old DOWN word, so the task set shrinks.

The old DOWN word and, once formed, the UP reply are fixed. All followers are unoccupied during collection. On DOWN selecting $i$, every other follower uses a distinct tuple arm, and map value one puts the coordinator at $R_0\ne R_i$, while map value zero forces a collision. On UP selecting $i$, the coordinator listens on $R_0$, every other follower uses a distinct $R_j$, and map value one puts $i$ at $R_i\ne R_0$, while map value zero blocks $R_0$. Thus the fresh slot, follower, fair map bit and uniform receiver arm give conditional service probability $g_t/(8dK)$ in either stage; the encoded word is fixed before the current map is read.

An arbitrary partial DOWN batch needs at most $2Q_{\rm dec}$ further positive tests before acceptance or a decoder failure (Lemma~\ref{sq:lem:decoder}). UP positives before the current reply is formed are not counted in this task. They may leave a partial batch containing a placeholder or an older reply, so retain the same $2Q_{\rm dec}$ allowance for positives after formation. Each allowance includes the detecting round of a failed fixed-word batch. Consequently the stopped composite task has quota $q=4Q_{\rm dec}$, also under truncation by overflow or the horizon. Lemma~\ref{sq:lem:waiting}, with $\omega_t=g_t$ and $a=1/(8dK)$, gives
\begin{equation*}
 \E\!\left[\sum_{t\in\mathcal C}g_t
        \,\middle|\,\mathcal F_s\right]
 \le\frac{8dK}{1-e^{-1}}(4Q_{\rm dec}+\log d)
 = B_{\rm col}\,.
\end{equation*}
This bounds the support consumed during collection, not its number of rounds.

At a delivery start $\sigma_{\rm D}$, Lemma~\ref{sq:lem:handover} bounds the expected support of the delivery phase, conditional on $\mathcal F_{\sigma_{\rm D}}$, by $B_{\rm del}$. At a release start $\sigma_{\rm R}$, Lemma~\ref{sq:lem:release} bounds the expected charge of the release phase, conditional on $\mathcal F_{\sigma_{\rm R}}$, by $B_{\rm rel}$. An absent phase contributes zero. Since $\bar g_t\le g_t$ on $\mathcal C\cup\mathcal D$ and the phase starts are stopping times, the tower property bounds $\E[J_s\mid\mathcal F_s]$ by $B_{\rm col}+B_{\rm del}+B_{\rm rel}=C_{\rm tr}$.
\end{proof}

\begin{proposition}[Total stopped transport envelope]\label{sq:prop:transport}
Consider the fresh-row reference execution launched with zero scores. Let $\tau_{\rm bad}$ be the first decoder
failure or first queue-overflow update, with value $T+1$ if neither
occurs. Include its round in the phase charges and stop afterwards.
Let $\mathcal C_*,\mathcal D_*,\mathcal R_*$ be the disjoint unions of
these stopped collection, delivery, and release phases. Define
\[
 J_{\rm stop}:=\sum_{t\in\mathcal C_*\cup\mathcal D_*}\bar g_t
 +\sum_{t\in\mathcal R_*}
   \left(r_t(U_t)+r_t(X_t)-\sum_{i=0}^dY_{i,t}\right).
\]
With $C_{\rm tr}$ from~\eqref{sq:eq:session},
\begin{equation}\label{sq:eq:transporttotal}
 \E J_{\rm stop}
 \le C_{\rm tr}\bigl[(4m+128)\eta T+2\bigr].
\end{equation}
Moreover, $\Pr(\tau_{\rm bad}\le T)\le2N^{-2}$.
\end{proposition}
\begin{proof}
\emph{Number of transfers.}
Let $L_t$ indicate that the live assignment's proposal coordinate is read before the stopping event, and set $n_{\rm prop}=\sum_{t=1}^T L_tB_t$. In the history $\mathcal F_t^{\rm prop}$ immediately before a possible read, $L_t$ is determined and $B_t$ is Bernoulli with parameter $h=4m\eta$. Hence
\[
 \E[L_tB_t\mid\mathcal F_t^{\rm prop}]=hL_t,
 \qquad \E n_{\rm prop}=h\sum_{t=1}^T\E L_t\le hT.
\]
Let $n_{\rm tr}$ count transfers started through $\tau_{\rm bad}$. Apart from the initial transfer, starts require a read proposal or grid boundary; boundaries inside collection add no transfer. Thus $n_{\rm tr}\le n_{\rm prop}+\lfloor T/H\rfloor+2$. Since $H=\lfloor(64\eta)^{-1}\rfloor\ge(128\eta)^{-1}$,
\begin{equation}\label{sq:eq:transfer-count}
 \E n_{\rm tr}\le(4m+128)\eta T+2.
\end{equation}

\emph{Conditional envelope charges.}
For $k\in[T+1]$, let $\sigma_k$ be the $k$-th transfer start, or $T+1$ if absent, and set $J_k=J_{\sigma_k}$ for an existing transfer, zero otherwise. These are stopping times: whether a transfer starts and whether any earlier failure has occurred are known before arm choice. Lemma~\ref{sq:lem:collection} gives
\[
 \E[J_k\mid\mathcal F_{\sigma_k}]
 \le C_{\rm tr}\ind\{\sigma_k\le T\}.
\]
Since $n_{\rm tr}=\sum_k\ind\{\sigma_k\le T\}$, the tower property yields
\begin{equation}\label{sq:eq:all-transfer-upper}
 \E\sum_kJ_k\le C_{\rm tr}\E n_{\rm tr}.
\end{equation}
By the phase definitions, distinct transfers' phases are disjoint and cover every round from launch through $\min\{\tau_{\rm bad},T\}$. Thus $\sum_kJ_k=J_{\rm stop}$, proving~\eqref{sq:eq:transporttotal}.

\emph{Failure probability.}
Lemma~\ref{sq:lem:decoder} bounds the probability of a decoder failure through the first queue overflow by $N^{-2}$. By Lemma~\ref{sq:lem:lazy}(iii) the physical overflow time equals the logical one, so Lemma~\ref{sq:lem:queue-bound} bounds the probability of an overflow by $N^{-2}$. Hence $\Pr(\tau_{\rm bad}\le T)\le2N^{-2}$.
\end{proof}

\section{Initialization}\label{sq:sec:boot}

This section specifies how the followers learn the coordinator's row index and when learning starts; Section~\ref{overview:public} gives the intuition. The coordinator shortens its list of possible identification indices, then transmits its private row index through the surviving list. The launch threshold controls prelaunch regret even when initialization never finishes.

Recall $N=2Km(T+1)$, $b_N=\lceil\log_2(N+1)\rceil$, and $D=2^{b_N}$. Fix $C_b=16$ and $N_0=2^{28}$; Appendix~\ref{app:explicit-constants} verifies these conservative choices. Call
\begin{equation}\label{sq:eq:bootbranch}
 T>C_bmKb_N^2,\qquad N\ge N_0
\end{equation}
the \emph{construction branch}; on its complement the players use a fixed assignment to $m$ distinct arms.

On this branch, preprocessing independently samples a binary identification array $(I_u(t))$ and $M=\lceil8T^2/(mb_N)\rceil$ complete rows $R_j$, each containing learning/control coordinates and forward bits $\varphi_j(u,t)$. After the adversary fixes the rewards, the coordinator privately draws $J\sim\Unif([M])$, independently of preprocessing, selecting $R_J$.

\subsection{Activation and launch}\label{sq:sec:ready}

Partition calendar time into consecutive blocks $\mathcal B_s$ of length
\begin{equation}\label{sq:eq:boot-block}
 B_{\rm boot}:=\lceil\sqrt{T/(mK)}\rceil,\qquad n_{\rm boot}:=\lceil T/B_{\rm boot}\rceil\,,
\end{equation}
with a possibly shorter final block. The number of identification indices and the list-size cutoff for multicast are
\begin{equation}\label{sq:eq:key-alphabet}
 A=A_{\rm id}:=\left\lceil\frac{T(B_{\rm boot}+b_N)}{b_N}\right\rceil,
 \qquad L_0:=2^{\lceil\log_2(2b_N)\rceil}\,.
\end{equation}

The ID entries are independent fair bits. All forward bits are mutually independent, independent of the ID array and learning/control coordinates, and equal one with probability $1-1/L_0$. The bootstrap decoding quota is $q_{\rm dec}:=8L_0b_N$.

Initially only the coordinator is informed, and no player has activated. A follower becomes \emph{informed} and activates in its declared row when its declaration takes effect; it is \emph{uninformed} beforehand. The coordinator activates at launch. At the start of each block, every uninformed follower independently
draws a fresh private Bernoulli coin of mean $1/m$ and is called \emph{active} in that block if the outcome is one. Each uninformed follower $i\in[m-1]$ has positive gain weights $w_i^{\rm boot}(a)$, $a\in[K]$. Its held arm in an active block has distribution
\begin{equation}\label{comm:boot-mixture}
 p_i(a)=\frac12\left(\frac{w_i^{\rm boot}(a)}{\sum_z w_i^{\rm boot}(z)}+\frac1K\right)\,.
\end{equation}
Write $a$ and $\iota\in[A]$ for an active follower's held arm and private identification index, both fixed throughout the block. These draws use only the preceding history and fresh private randomness, without using current-block complete-row symbols. Let $G_{i,s}$ count all positive stay observations ($I_\iota(t)=0$, $Y_{i,t}>0$) over the entire block $s$, including those after a declaration is saved, and set it to zero when inactive. At the block boundary, an active follower updates its held-arm weight by
\begin{equation}\label{sq:eq:boot-weight}
 w_i^{\rm boot}(a)\gets
 2^{-4b_N}\left\lfloor2^{4b_N} w_i^{\rm boot}(a)
 \left(1+\frac{G_{i,s}}{4KB_{\rm boot}p_i(a)}\right)\right\rfloor\,.
\end{equation}
Other arm weights are unchanged; the follower discards these weights on activation.

In \textsc{DISCOVER}, the coordinator reduces its ID list $\Lambda\subseteq[A]$; in \textsc{MULTICAST}, it transmits through that list. Its block target is $z\in[K]$; empty products equal one. Its UP decoder is disabled before launch.

Let $C_r(a)$ count the coordinator's positive DOWN observations on arm $a$ through block $r$. To test readiness, it monitors
\begin{equation}\label{sq:eq:root-monitor}
 \widehat D_r:=K\max_{S\in\cS}\sum_{a\in S}C_r(a)\,.
\end{equation}
The maximum sums the $m$ largest counters. Set $C_h=4096$ and $h_{\rm boot}:=\lceil C_hL_0b_N\rceil$. Set $C_*=2^{20}$ and define the integer threshold
\begin{equation}\label{sq:eq:launch-budget}
 Q:=\left\lceil C_*\bigl[m^2KB_{\rm boot}b_N+mh_{\rm boot}n_{\rm boot}
       +m^2h_{\rm boot}b_N+mB_{\rm boot}\bigr]\right\rceil\,.
\end{equation}

On activation, a follower enters the unoccupied state in its declared row, with accepted start, score, last increment and queue initially zero. It uses Algorithm~\ref{sq:alg:handover}'s unoccupied arm rules and initial UP record $(0,0,0,0)$.

In Algorithm~\ref{sq:alg:activation}, $r$ is the current block index, $\tau_t$ is the slot in the coordinator's row $R_J$, and $R_{0,t}$ is its tuple arm on DOWN. All players choose their arms before processing that round's observations and declarations.

\begin{algorithm}[p]
\caption{Activation and launch at player $i$}\label{sq:alg:activation}
\normalsize
\begin{algorithmic}[1]
\Require Only the coordinator knows $J$; explicit uniform arm draws are fresh and private.
\State Set follower gain weights to one; coordinator sets $C_0(a)\gets0$ ($a\in[K]$).
\For{$t=1,\ldots,T$}
 \If{the player has activated}
  \State Run one round of Algorithm~\ref{protocol:alg:order} in the selected or declared row.
 \Else
  \If{$t$ starts a block}
   \If{$i>0$}
    \State Draw a fresh private activity coin of mean $1/m$.
    \State If active, draw $a\sim p_i$, then $\iota\sim\Unif([A])$.
    \State $G_{i,r}\gets0$; reset the pending declaration, candidate counts and flags.
   \Else
    \State Draw $z\sim\Unif([K])$; $(\Lambda,C_r)\gets([A],C_{r-1})$; enter DISCOVER.
   \EndIf
  \EndIf
  \If{$i=0$}
   \If{$\tau_t=\mathrm{DATA}$}
    \State $A_{0,t}\gets z$.
    \If{in MULTICAST and $\prod_{u\in\Lambda}\varphi_J(u,t)=1$}
     \State Draw $A_{0,t}\sim\Unif([K]\setminus\{z\})$.
    \EndIf
   \ElsIf{$\tau_t=\mathrm{DOWN}$}
    \State $A_{0,t}\gets R_{0,t}$.
   \Else
    \State Draw $A_{0,t}\sim\Unif([K])$.
   \EndIf
  \ElsIf{inactive this block}
   \State Draw $A_{i,t}\sim\Unif([K])$.
  \Else
   \State $A_{i,t}\gets a$; if $I_\iota(t)=1$, draw $A_{i,t}\sim\Unif([K]\setminus\{a\})$.
  \EndIf
  \State Play $A_{i,t}$; observe $Y_{i,t}$ rounded down to a multiple of $1/D$.
  \If{$i>0$ is active this block}
   \State $G_{i,r}\gets G_{i,r}+\ind\{I_\iota(t)=0,\ Y_{i,t}>0\}$.
   \If{no declaration is pending}
    \State Apply the current-observation candidate update of Algorithm~\ref{comm:alg:raw-decoder}.
   \EndIf
  \ElsIf{$i=0$}
   \If{$\tau_t=\mathrm{DATA}$, in DISCOVER, and $Y_{0,t}>0$}
    \State $\Lambda\gets\{u\in\Lambda:I_u(t)=1\}$.
    \State If $|\Lambda|\le L_0$, enter MULTICAST from round $t+1$.
   \ElsIf{$\tau_t=\mathrm{DOWN}$}
    \State $C_r(R_{0,t})\gets C_r(R_{0,t})+\ind\{Y_{0,t}>0\}$.
   \EndIf
  \EndIf
  \If{$t$ ends a block}
   \If{$i>0$}
    \State If active, update the held-arm weight by~\eqref{sq:eq:boot-weight}.
    \State If a declaration is pending, activate unoccupied in its declared row at $t+1$.
   \Else
    \State $\widehat D_r\gets K\max_{S\in\cS}\sum_{a\in S}C_r(a)$.
    \State If $\widehat D_r>4Q$, launch the zero-score, zero-queue learner at $t+1$.
   \EndIf
  \EndIf
 \EndIf
\EndFor
\end{algorithmic}
\end{algorithm}

\FloatBarrier
\subsection{Reference initialization and guarantees}

For the fresh-row analysis, fix the identification array and rounded rewards before sampling independent complete rows. Algorithm~\ref{sq:alg:reference} tests only the true bootstrap row. For fixed array, rewards and $J=j$, its expected cost over private draws depends only on $R_j$, allowing concentration across complete rows.

\begin{algorithm}[H]
\caption{Reference initialization (analysis only)}\label{sq:alg:reference}
\normalsize
\begin{algorithmic}[1]
\State Run Algorithm~\ref{sq:alg:activation} with the corrections below.
\State Use only the true-$J$ bootstrap test; save $J$ as pending as soon as it passes.
\State Apply declarations at block boundaries, entering the true-$J$ unoccupied state next round.
\If{a DOWN word would be accepted before launch}
 \State Record an initialization failure; reject it before assignment-dependent effects.
\EndIf
\State Continue DATA logging and DOWN batch updates and resets.
\If{$\widehat D_r>4Q$ at a block end, after pending declarations take effect}
 \For{each follower outside the true-$J$ unoccupied state}
  \State Record an initialization failure; put it in that state before launch.
 \EndFor
\EndIf
\end{algorithmic}
\end{algorithm}

The postlaunch reference is specified in Sections~\ref{sq:sec:learner} and~\ref{sq:sec:clock}. Let $\tau_{\rm L}\in[T+1]$ be the reference coordinator's launch round, equal to $T+1$ if no launch occurs. Set
\begin{equation}\label{sq:eq:bootQ}
 Q_{\rm boot}:=m^{3/2}\sqrt{KT}\,b_N^2\,.
\end{equation}
On the construction branch, Lemma~\ref{init:cor:cost} gives $Q=O(Q_{\rm boot})$ with a universal implicit constant.

\begin{proposition}[Constructive initialization]\label{sq:prop:startup}
Assuming~\eqref{sq:eq:bootbranch}, with probability at least $1-3/N$ over the identification array, the following holds for every fixed rounded reward matrix $r$ independent of the complete rows and every $j\in[M]$. Probabilities and expectations average over independent complete rows and private draws, conditional on that array and the uniform draw $J=j$. Set
\begin{equation}\label{sq:eq:bootfailure}
 \Delta_{\rm boot}:=5N^{-2}\,.
\end{equation}
\begin{enumerate}[label=(\roman*),leftmargin=*]
\item Couple Algorithms~\ref{sq:alg:activation} and~\ref{sq:alg:reference} using the same complete rows and private randomness. With probability at least $1-\Delta_{\rm boot}$, their actions, observations, gain weights, effective declarations, activation times, and ordinary learning and communication registers agree through initialization at the start of the launch round, or through round $T$ if no launch occurs.
\item At the beginning of every reference launch round, each follower is in the true-$j$ unoccupied state with zero accepted start, score, last increment, and queue; the coordinator's scores and queues are also zero. For every $S\in\cS$,
\begin{equation*}
 \E\sum_{t<\tau_{\rm L}}r_t(S)
 \le C Q_{\rm boot}+CmT\Delta_{\rm boot}\,,
\end{equation*}
for some universal constant $C$.
\item Among complete rows, the reference initialization and its correction decisions read only $R_j$. Correcting a follower at launch or suppressing a prelaunch DOWN acceptance has probability at most $\Delta_{\rm boot}$. The launch round is a stopping time for the selected-row and private history exposed before that round; conditional on this history, unexposed selected-row coordinates retain their product law.
\end{enumerate}
\end{proposition}
The proof is in Appendix~\ref{comm:startup-proof}.

\section{Proof of the main theorem}\label{sq:sec:tuning}

\begin{proof}[Proof of Theorem~\ref{sq:thm:main}]
Recall $N=2Km(T+1)$, $b_N=\lceil\log_2(N+1)\rceil$.
Set $C_1=2^{20}$, $C_b=16$, and $N_0=2^{28}$. Appendix~\ref{app:explicit-constants} checks the numerical domination and feasibility inequalities used below. Work first on the construction branch~\eqref{sq:eq:bootbranch}.

\emph{Rational tuning.} Set $\mathcal J=\lceil\log_2(\binom Km+1)\rceil+m$ and $W=C_1m^2Kb_N$, and choose the largest power of two $\eta$ satisfying
\[
 \eta\le(1024mKL_{\rm stab})^{-1},\qquad \eta^2WT\le\mathcal J.
\]
Since $\mathcal J\le\frac65mb_N$ and $L_{\rm stab}=\lceil\log_2(2K(T+1)N^2)\rceil\le3b_N$ (Appendix~\ref{app:explicit-constants}),
\[
 \frac{(1024mKL_{\rm stab})^2\mathcal J}{W}
 =\frac{KL_{\rm stab}^2\mathcal J}{b_N}
 \le\frac{54}5mKb_N^2<16mKb_N^2<T
\]
on the construction branch. Thus $\sqrt{\mathcal J/(WT)}\le(1024mKL_{\rm stab})^{-1}$, and the largest-dyadic rule gives
\begin{equation}\label{sq:eq:dyadic-rate}
 \frac12\sqrt{\frac{\mathcal J}{WT}}<\eta
 \le\sqrt{\frac{\mathcal J}{WT}}.
\end{equation}
Also $\mathcal J\asymp\log\binom Km+m$. Since $\mathcal J\ge m$ and $2mKT\le N$,~\eqref{sq:eq:dyadic-rate} gives
\begin{equation*}
 \eta>\frac1{2\sqrt{C_1mKb_NT}}\ge\frac1{\sqrt{2C_1b_NN}}\ge\frac1N,
\end{equation*}
where the last step uses $N\ge2^{21}b_N=2C_1b_N$, valid for $N\ge N_0=2^{28}$ (Appendix~\ref{app:explicit-constants}). Since $\eta$ and $D=2^{b_N}\ge N$ are powers of two, $\eta D$ is integral.

\emph{Whole-row concentration.}
A published row need not retain fresh-row conditional independence, so we concentrate complete-execution quantities across rows. Let $\mathcal G_0$ be the event of Lemma~\ref{comm:arrays}, then $\Pr(\mathcal G_0^c)\le3/N$. Fix an identification array in $\mathcal G_0$. The complete rows $\mathsf R_1,\ldots,\mathsf R_M$ are independent, with $M=\lceil8T^2/(mb_N)\rceil$. Set $V=m^{3/2}\sqrt{KT}b_N$.

Use the true-only corrected reference of Algorithm~\ref{sq:alg:reference}; it also rejects ordinary decoded words other than the word currently encoded on their channel. For each fixed denominator-$D$ reward matrix $r$ and index $j\in[M]$, let $F_{j,r}(R_j)\in[-mT,mT]$ be the complete reference execution's expected regret over private randomness conditional on $J=j,R_j$.

Let $E_{\rm init}$ be the event that a missing follower is corrected at launch or an ordinary DOWN acceptance is suppressed before launch. Incorrect-row forward passes are excluded.

Let $E_{\rm false}$ be the event that a completed ordinary batch has a unique solution that passes the local acceptance test and differs from the word currently encoded on its channel. Let $E_{\rm amb}$ be the event that a completed batch has multiple solutions and is formed wholly from one fixed awaited word. Consider both events through and including the first queue-overflow update, or through $T$ if none occurs, as in Appendix~\ref{protocol:decoder-proof}.

Define $E_{\rm queue}=\{\tau_*^{\rm phys}\le T\}$ using the physical
calendar-round queues of Section~\ref{sq:sec:clock}. Even an arrival not
yet reconstructed by its owner is a deterministic function of the true
physical assignment, selected row, private draws, and fixed rewards;
Lemma~\ref{protocol:lem:total} establishes this local reconstruction
invariant. Thus this event uses no auxiliary logical randomness.
Lemma~\ref{sq:lem:lazy} will identify its probability with the logical
overflow probability in the fresh-row analysis. Put
\begin{equation*}
 \mathcal E:=E_{\rm init}\cup E_{\rm false}\cup E_{\rm amb}\cup E_{\rm queue},\qquad
 E_{j,r}(R_j):=\Pr(\mathcal E\mid J=j,R_j)\in[0,1]\,.
\end{equation*}
Couple the actual protocol and the reference using the same private randomness, with $J$ drawn uniformly from $[M]$.

All complete-row reads of this reference, including initialization and the tests defining $\mathcal E$, use $R_j$. Thus $F_{j,r}$ and $E_{j,r}$ depend on $R_j$ alone. The virtual incorrect-row forward tests read other rows and are bounded separately below.

Define the combined row cost and its averages by
\begin{align*}
 G_{j,r}(R_j)&:=F_{j,r}(R_j)+2mT E_{j,r}(R_j)\in[-mT,3mT],\\
 \mu_G(r)&:=\frac1M\sum_{j=1}^M\E G_{j,r}(\mathsf R_j),\qquad
 \bar G_r:=\frac1M\sum_{j=1}^M G_{j,r}(R_j).
\end{align*}
Here $\mu_G(r)$ averages over newly sampled rows with the identification array fixed, whereas $\bar G_r$ is the average for the published rows. Row independence and Hoeffding's inequality~\cite{hoeffding1963} give, for fixed $r$,
\begin{equation*}
 \Pr\{\bar G_r>\mu_G(r)+2V\}
 \le e^{-MV^2/(2m^2T^2)}.
\end{equation*}
There are $(D+1)^{KT}$ denominator-$D$ matrices. The definition of $M$, $\log(D+1)\le2b_N$, $\log(2N)\le2b_N$, and $KT\ge1$ give
\begin{equation*}
 \frac{MV^2}{2m^2T^2}\ge4KTb_N
 \ge KT\log(D+1)+\log(2N)\,.
\end{equation*}
Let $\mathcal G_{\rm row}$ be the event that, simultaneously for every denominator-$D$ reward matrix,
\begin{equation}\label{sq:eq:table-cost}
 \bar G_r\le\mu_G(r)+2V.
\end{equation}
A union bound gives $\Pr(\mathcal G_{\rm row}^c\mid\mathcal G_0)\le N^{-1}$.
Let $\mathcal G_{\rm test}$ be the public event of Lemma~\ref{comm:acceptance}, with failure probability at most $N^{-1}$. On this event, conditional on the complete public output and any selected index, the private probability of a forward false pass is at most $\varepsilon_{\rm false}=1/B_{\rm boot}$ for every reward matrix. Then $\mathcal G_{\rm pub}:=\mathcal G_0\cap\mathcal G_{\rm row}\cap\mathcal G_{\rm test}$ has failure probability at most $5/N$. The independent-row calculation precedes this intersection.

\emph{Regret decomposition.}
Fix a public output in $\mathcal G_{\rm pub}$ and a rounded matrix $r$. Let $E_{\rm fwd}$ be the event that some $v\ne J$ would pass the bootstrap prefix test~\eqref{comm:decoder} of Algorithm~\ref{comm:alg:raw-decoder} on the true-only reference observations; these virtual tests do not affect the reference execution. On $\mathcal E^c\cap E_{\rm fwd}^c$, Proposition~\ref{sq:prop:startup} couples initialization, the decoder rule couples subsequent messages, and Lemma~\ref{protocol:lem:total} excludes a local domain error. The actual and reference cumulative rewards therefore agree outside $\mathcal E\cup E_{\rm fwd}$, and each lies in $[0,mT]$. Their absolute difference is at most $2mT(\ind\{\mathcal E\}+\ind\{E_{\rm fwd}\})$. On this output, the reference expected regret is $M^{-1}\sum_jF_{j,r}(R_j)$ and $\Pr(\mathcal E)=M^{-1}\sum_jE_{j,r}(R_j)$, since $J$ is uniform. Subtracting expected rewards from the same comparator reward and using $\Pr(E_{\rm fwd})\le\varepsilon_{\rm false}$ gives
\begin{equation}\label{sq:eq:actual-reference}
 \begin{aligned}
 R_T(r)&\le\frac1M\sum_{j=1}^M
       \bigl[F_{j,r}(R_j)+2mT E_{j,r}(R_j)\bigr]
       +2mT\varepsilon_{\rm false}\\
       &=\bar G_r+2mT\varepsilon_{\rm false}\,.
 \end{aligned}
\end{equation}

To bound $\mu_G(r)$, fix the identification array and $r$ and average over independent rows and private draws in the fresh-row model. Recall that $\tau_{\rm L}\in[T+1]$ is the reference launch round, equal to $T+1$ if no launch occurs. Proposition~\ref{sq:prop:startup}(ii) gives, for every $S\in\cS$,
\[
 \E\sum_{t<\tau_{\rm L}}r_t(S)
 \le CQ_{\rm boot}+CmT\Delta_{\rm boot}.
\]
The learner and transport bounds apply at this stopping time: launch initializes zero scores and queues, and future control coordinates are fresh conditional on the launch history.

Retain signed auxiliary regret on communication rounds. Fix $S\in\cS$. Put $c_t=r_t(S)$ and $y_t=\sum_iY_{i,t}$ on every round, and $\ell_t=r_t(U_t)+r_t(j_t)$ on enabled rounds. Let $\tau_{\rm bad}$ be the first decoder failure or queue-overflow update, or $T+1$ if neither occurs, and $J_{\rm stop}$ the transport charge through that round from Proposition~\ref{sq:prop:transport}. On an enabled collection or delivery round through $\tau_{\rm bad}$,
\[
 c_t-y_t=(c_t-\ell_t)+(\ell_t-y_t)
 \le(c_t-\ell_t)+\bar g_t,
\]
because $\ell_t\le\bar g_t=\min\{m,g_t\}$ and $y_t\ge0$. On a frozen collection round there is no auxiliary summand and $c_t-y_t\le\bar g_t$. On release rounds through $\tau_{\rm bad}$, Lemma~\ref{sq:lem:lazy} identifies $j_t$ with the coordinator's physical fast-arm draw $X_t$; this applies because $\tau_{\rm bad}\le\tau_*^{\rm phys}=\tau_*$ by the definition of $\tau_{\rm bad}$ and Lemma~\ref{sq:lem:lazy}. Hence $\ell_t-y_t$ is included in $J_{\rm stop}$.

Keep the full auxiliary sum over all rounds $\mathcal A$ enabled by the reference episode schedule (Section~\ref{sq:sec:learner}), including the bounded post-overflow continuation. A decoder failure does not stop the logical estimates. Its tail after $\tau_{\rm bad}$ is at least $-mT\ind\{\tau_{\rm bad}\le T\}$, whereas the physical reference-regret tail is at most $mT\ind\{\tau_{\rm bad}\le T\}$. Thus
\begin{equation}\label{sq:eq:reference-accounting}
\begin{aligned}
 \sum_{t=1}^T(c_t-y_t)\le\sum_{t<\tau_{\rm L}}r_t(S)
       +\sum_{t\in\mathcal A}(c_t-\ell_t)+J_{\rm stop}
       +2mT\ind\{\tau_{\rm bad}\le T\}.
\end{aligned}
\end{equation}
Taking expectations using Propositions~\ref{sq:prop:learning} and~\ref{sq:prop:transport}, with $\Pr(\tau_{\rm bad}\le T)\le2N^{-2}$, gives the learning and transport contributions
\begin{align*}
 \mathcal L_\eta
 &:=C\left[\frac{\log\binom Km+m}{\eta}
       +(m^2K+KL_{\rm stab})\eta T\right]
       +\frac{mT}{\eta D}+CmTN^{-2},\\
 \mathcal T_\eta
 &:=C_{\rm tr}\bigl[(4m+128)\eta T+2\bigr]+4mTN^{-2}.
\end{align*}
The last term of $\mathcal T_\eta$ bounds the expectation of the last term of~\eqref{sq:eq:reference-accounting}.

Add $2mT\Pr(\mathcal E)$ to this reference regret bound. In the fresh-row model, with the identification array and $r$ fixed, the component bounds are
\begin{align*}
 \Pr(E_{\rm init})&\le\Delta_{\rm boot}
   &&\text{(Proposition~\ref{sq:prop:startup}(iii))},\\
 \Pr(E_{\rm false}\cup E_{\rm amb})&\le N^{-2}
   &&\text{(Lemma~\ref{sq:lem:decoder})},\\
 \Pr(E_{\rm queue})&=\Pr(\tau_*\le T)\le N^{-2}
   &&\text{(Lemmas~\ref{sq:lem:lazy} and~\ref{sq:lem:queue-bound})}.
\end{align*}
The union bound therefore gives
\begin{align*}
 \Pr(\mathcal E)
 &\le\Pr(E_{\rm init})+\Pr(E_{\rm false}\cup E_{\rm amb})
       +\Pr(E_{\rm queue})\\
 &\le\Delta_{\rm boot}+2N^{-2}\,.
\end{align*}
These expectations use the same complete-row law that defines $\mu_G(r)$. Consequently,
\begin{equation}\label{sq:eq:product-cost}
 \mu_G(r)\le CQ_{\rm boot}+\mathcal L_\eta+\mathcal T_\eta
       +CmT(\Delta_{\rm boot}+2N^{-2}).
\end{equation}
Combine~\eqref{sq:eq:table-cost},~\eqref{sq:eq:actual-reference}, and~\eqref{sq:eq:product-cost}. The choice of $C_1$, $C_{\rm tr}=O(mKb_N)$, and $L_{\rm stab}=O(b_N)$ make $W=C_1m^2Kb_N$ dominate the $\eta T$ coefficients. Since $\log\binom Km\le(\log2)\mathcal J$, we obtain
\begin{equation}\label{sq:eq:assembled}
 \begin{aligned}
 R_T(r)\le C\Big[&Q_{\rm boot}+mKb_N+V
                  +\frac{\mathcal J}{\eta}+W\eta T
                  +\frac{mT}{\eta D}\\
                &+mT(\Delta_{\rm boot}+2N^{-2}+\varepsilon_{\rm false})\Big].
 \end{aligned}
\end{equation}

\emph{Bounds for the terms.} Recall $Q_{\rm boot}=m^{3/2}\sqrt{KT}b_N^2$, $\Delta_{\rm boot}=5N^{-2}$, while $\varepsilon_{\rm false}=1/B_{\rm boot}$, so
\begin{equation*}
 mT\Delta_{\rm boot}=O(1),\qquad mT\varepsilon_{\rm false}\le m^{3/2}\sqrt{KT}\,.
\end{equation*}
Also $mTN^{-2}=O(1)$ because $mT\le N$.
By~\eqref{sq:eq:dyadic-rate},
\[
 \frac{\mathcal J}{\eta}+W\eta T
 \le C\sqrt{\mathcal J W T}\le CV.
\]
Since $D\ge N=2Km(T+1)$ and $\mathcal J\ge m$,
\begin{equation*}
 \frac{mT}{\eta D}\le\frac{\mathcal J}{\eta}\le CV.
\end{equation*}
Also $mKb_N\le CQ_{\rm boot}$, so~\eqref{sq:eq:assembled} gives
\begin{equation}\label{sq:eq:sharper-startup}
 R_T(r)\le C Q_{\rm boot}
       =Cm^{3/2}\sqrt{KT}\,b_N^2
\end{equation}
for every rounded matrix on this output.

For an arbitrary original matrix $r^{\rm orig}$, the protocol uses rounded observations, so the same private randomness produces identical actions under $r^{\rm orig}$ and its rounded matrix $r$. Rounding down changes each comparator reward by at most $m/D$ per round, while the actual unrounded reward is at least its rounded reward. Hence
\[
 R_T(r^{\rm orig})\le R_T(r)+\frac{mT}{D}
 \le Cm^{3/2}\sqrt{KT}\,b_N^2\,.
\]
This pathwise comparison covers all oblivious reward matrices on the same output in $\mathcal G_{\rm pub}$.

On the complementary branch use a fixed assignment to $m$ distinct arms. For $m$-sets $S,S'$, $|S\setminus S'|\le\min\{m,K-m\}$, so $r_t(S)-r_t(S')\le\min\{m,K-m\}$ and regret is at most
\[
 \min\{m,K-m\} T\le mT\le\sqrt{C_b}\,m^{3/2}\sqrt{KT}b_N=O(V)
\]
when $T\le C_bmKb_N^2$. If $N<N_0$, the regret is bounded by an
absolute constant and is absorbed by the same claimed bound.
This proves~\eqref{sq:eq:main}; Section~\ref{sq:sec:complexity} gives the computational bounds.
\end{proof}

\section{Computational and storage complexity}\label{sq:sec:discussion}
\label{sq:sec:complexity}

We count exact arithmetic operations and storage registers, with unit-cost generation of independent $\Unif([0,1))$ draws. Throughout this section, $\widetilde O$ suppresses factors polynomial in $\log N$.

\begin{proposition}[Computation and storage]\label{sq:prop:complexity}
On the construction branch, the following bounds hold.
\begin{enumerate}[label=(\roman*),leftmargin=*]
\item Preprocessing uses $O\Big(\frac{T^{9/2}}{m^{3/2}\sqrt K\,(\log N)^2}\Big)$ operations and memory cells to build and store the public table.
\item After preprocessing, the total work done by all players over the $T$ rounds is $\widetilde O(T^3)$ operations. Together, the players use at most $O(T^2/\log N)$ cells of private memory at any time. These bounds hold even if a player accepts an incorrect row index or message.
\end{enumerate}
\end{proposition}

\begin{proof}
\emph{(i) Preprocessing.} Recall $b_N=\lceil\log_2(N+1)\rceil$, $B_{\rm boot}=\lceil\sqrt{T/(mK)}\rceil$, $M=\lceil8T^2/(mb_N)\rceil$, and $A_{\rm id}=\lceil T(B_{\rm boot}+b_N)/b_N\rceil$. On the construction branch, $T>C_bmKb_N^2$ and $b_N=\Theta(\log N)$, so
\begin{equation*}
 M=O\!\left(\frac{T^2}{m\log N}\right),\qquad
 A_{\rm id}=O\!\left(\frac{T^{3/2}}{\sqrt{mK}\log N}\right)\,.
\end{equation*}
In particular, $A_{\rm id}\ge T\ge Kb_N$ and $M\ge T$.

The ID and forward arrays contain $A_{\rm id}T$ and $MA_{\rm id}T$ bits. The scalar controls contain $O(MT(K+m))$ entries, and the affine binary maps contain $O(MTmB)$ coefficients, where the message width $B$ is $O(b_N)$. Each entry takes one cell and a constant number of operations to generate and store. Thus preprocessing uses
\begin{equation*}
 O\!\left(A_{\rm id}T+MA_{\rm id}T+MT(K+mB)\right)
 =O(MA_{\rm id}T)
\end{equation*}
operations and memory cells. Substituting the bounds on $M$ and $A_{\rm id}$ proves~(i). 

\emph{(ii) During play.} The initialization in Algorithm~\ref{sq:alg:activation} takes place during the $T$ rounds. At each block boundary, computing a follower's $K$ held-arm probabilities in~\eqref{comm:boot-mixture} and their cumulative sums takes $O(K)$ operations. Selecting an arm takes $O(K)$ further operations, and the weight update~\eqref{sq:eq:boot-weight} takes $O(1)$. Over all followers and blocks, this costs $O(mKn_{\rm boot})\le O(mKT)$ operations. The other private draws add $O(mT)$.

An uninformed follower keeps one count and one failure flag for each candidate row. Updating these tests, checking for a unique passing row, and resetting the counts at block boundaries cost $O(mMT)$ operations altogether. The coordinator's discovery-list scans cost $O(A_{\rm id}T)$; its multicast tests cost $O(TL_0)$; and sorting its $K$ counters at block boundaries costs $O(n_{\rm boot}K\log K)$. Thus initialization takes $\widetilde O(mMT+A_{\rm id}T+KmT)$ operations, even if a player remains in one state for all $T$ rounds.

\paragraph{Reading the table and decoding messages.}
To read an injective arm tuple, use a balanced binary tree whose leaves are the ordered arm labels and whose nodes count the unused labels below them. Building the tree costs $O(K)$ operations. Finding and removing each next arm by its position in the ordered list costs $O(\log K)$, so reading a tuple costs $O(K+m\log K)$ per player. Reading the pilots, slot, selected follower, recorder source, and proposal and thinning tests takes $O(m)$ more operations once their probabilities are known. We count this work for every player, since incorrect row declarations can make them read different rows. The total is $\widetilde O(KmT)$.

Computing an affine map and encoding or reading a message takes $O(B)$ operations. A decoder decision solves at most $Q_{\rm dec}$ binary equations in $B$ unknowns in $O(B^2Q_{\rm dec})$ operations. Since $B,Q_{\rm dec}=O(b_N)$ and each player makes at most one such decision per round, message decoding costs $\widetilde O(mT)$ operations.

\paragraph{Choosing assignments.}
The required $m$-set law is either $Q_{w_t}$ or, before the old grid boundary,
\begin{equation*}
 f(S)=Q_{w_t}(S)+\kappa_t\bigl(Q_{w_u}(S)-Q_{w_{u-1}}(S)\bigr),
 \qquad \kappa_t=(1-h)^{t-u+1}/h\,.
\end{equation*}
Here $u$ is the collection cutoff and $h$ is the proposal probability. Choosing the fast arm by~\eqref{sq:eq:orientation} then gives $p_{w_t}$ or the law in~\eqref{sq:eq:deferred-law}. For a product law, the ID-Checking Sampling method~\cite[Section~4]{chenliu1997} chooses the arms in order using $O(Km)$ operations. For $f$, each probability sum used by this rule is the same signed combination of three product-law sums. The elementary symmetric sums of the remaining weights and the product of the weights already chosen compute all these sums in $O(Km)$ operations. Appendix~\ref{protocol:fallback} proves that these sums are nonnegative and their denominators positive on valid executions, and gives the checks used in other cases. Choosing the fast arm takes $O(m)$ more operations. Thus one assignment uses $O(Km)$ operations and at most $K+1$ uniform coordinates, without listing all $\binom Km$ subsets.

\paragraph{Updating scores and saved observations.}
The probabilities $\theta_w(i)$ that each arm is included can also be found in $O(Km)$ operations. First compute $e_0(w)=1,\ldots,e_m(w)$, then use
\begin{equation*}
 e_r(w_{-i})=e_r(w)-w_i e_{r-1}(w_{-i}),
 \qquad
 \theta_w(i)=\frac{w_i e_{m-1}(w_{-i})}{e_m(w)}\,.
\end{equation*}
For each $i$, start with $e_0(w_{-i})=1$ and compute through degree $m-1$. Find each recording probability $\xi_i$ in~\eqref{sq:eq:rates} by comparing $\max\{\beta,f_i\}$ with $2^{-k}$ for $0\le k\le b_N$. The rates, fast-arm probabilities in~\eqref{sq:eq:q}, and queue updates take $\widetilde O(K+m)$ further operations per round. The checks in Appendix~\ref{protocol:fallback} keep the integer scores in $[0,TD]$, so repeated squaring computes each weight $(1+1/D)^{z_j}$ in $O(\log(TD+1))$ operations and $\kappa_t$ in $O(\log(T+1))$. There are at most $T$ assignment starts, giving $\widetilde O(KmT)$ operations in total for these calculations and assignment choices.

Recall $d=m-1$. Precompute $\zeta=((K-1)/K)^d$ and $2^{h'}$, $2^{-h'}$ for $0\le h'\le b_N$. Computing one saved contribution in~\eqref{protocol:eq:saved-contribution} then takes $O(1)$ operations. After each DOWN acceptance, a follower reads at most $T$ calendar rounds to find the cutoff and recover its queue and served sum. At most $T$ acceptances per follower give $O(mT^2)$ operations. The coordinator stores arrivals by arm and time, so recovering one returned arm's queue requires at most $T$ rounds, including rounds with no arrival. Each reply is accepted at most once per collection, and there are at most $T$ assignments. This costs $O(KT^2)$ operations. The other counters, lists, and queue updates fit these bounds.

The loops have these limits even when a row index or message is wrong: Appendix~\ref{protocol:fallback} checks the allowed ranges before using a value. Using the bounds on $M,A_{\rm id}$ and $mKb_N^2\le T$, the total work during play is
\begin{equation*}
 \widetilde O\!\left(mMT+A_{\rm id}T+KmT+(m+K)T^2\right)
 =\widetilde O(T^3)\,.
\end{equation*}

\paragraph{Private memory.}
An uninformed follower keeps $M$ candidate counts and flags, $K$ weights, and a constant number of block variables. After activation, it saves at most one four-field observation record per round. It also keeps one DOWN equation batch, its current assignment and reply, and $O(K)$ temporary cells for arm choices. A batch has at most $Q_{\rm dec}$ equations with $B+1$ bits each, so one follower uses $O(M+K+T+BQ_{\rm dec})$ cells.

The coordinator keeps at most $A_{\rm id}$ discovery indices, $K$ counters, the assignment and pending-follower lists, and one UP equation batch per follower. Scores, queues, rates, anchors, saved values at the cutoff, running sums during collection, and weight vectors use a constant number of $K$-vectors. Choosing assignments needs $O(Km)$ temporary cells. Saved arrivals use at most $KT$ entries, one per arm and round. The coordinator therefore uses $O(A_{\rm id}+KT+Km+mBQ_{\rm dec})$ cells.

When reading saved observations, a player needs only its current queue, served sum, and last service. These cells are reused each time the player goes through its saved observations. Proposal values are read from the public table. The checks in Appendix~\ref{protocol:fallback} keep time indices and list sizes within their allowed ranges after an incorrect message. Adding up all players' memory gives
\begin{equation*}
 O\!\left(mM+A_{\rm id}+KT+Km+mBQ_{\rm dec}\right)
 =O(T^2/\log N)\,,
\end{equation*}
by the bounds on $M,A_{\rm id}$, $m<K$, $B,Q_{\rm dec}=O(b_N)$, and $T>C_bmKb_N^2$. This is the maximum number of private cells used at one time, proving~(ii).
\end{proof}

\appendix
\clearpage
\section{Exceptional executions}\label{protocol:spec}
We define local checks that keep the protocol well-defined after errors and give conditions under which none of these checks fails.
Recall $N=2Km(T+1)$, $b_N=\lceil\log_2(N+1)\rceil$, and the score denominator $D=2^{b_N}$.

\subsection{A total local algorithm}\label{protocol:fallback}

Give each player an absorbing error state $\dagger$. Before each local operation, check its input domains: arm labels in $[K]$; scores, integer queues, initial and terminal queue fields and accumulated service sums in $\mathbb Z\cap[0,TD]$; raw arrivals and served increments in $\mathbb Z\cap[0,D]$; time fields in $[0,T+1]$. Raw buffers hold at most one arrival record per arm and calendar round, with timestamps in $[T]$. Every required denominator must be positive, and every Bernoulli parameter, including thinning and mixture parameters, must lie in $[0,1]$; for the slow-set sampler, check each step's scalar probabilities and normalizers. The local record must specify exactly one applicable action branch. If an input check fails, enter $\dagger$ without evaluating the operation; if an output check fails, enter $\dagger$ before storing or using the output. In $\dagger$, the player performs no learning or transport update and player $i$ plays arm $i+1$ on active rounds; bootstrap updates and activation follow Algorithm~\ref{protocol:alg:order}. A decoder rejection does not enter $\dagger$.

\begin{lemma}[Local implementability and absence of domain errors]
\label{protocol:lem:total}
Use the construction-branch rate choice of Section~\ref{sq:sec:tuning}, and fix any table whose coordinates lie in the prescribed domains (every scalar coordinate lies in $[0,1)$). Suppose that, through round $t$, bootstrap states and activation times agree with the reference execution, every accepted ordinary message is the message currently sent on its channel, and the feasibility and stability bounds of Section~\ref{sq:sec:learner} hold. Then every learning and transport operation through round $t$ uses only information available to the player performing it, has inputs and outputs in its prescribed domains, and does not enter $\dagger$. At every assignment start $v\le t$, the coordinator knows the exact calendar-round scores $Z(v)$ and queues $\mathsf B_v$, and every follower is unoccupied before the new assignment is issued.
\end{lemma}
\begin{proof}
The statement concerns a fixed execution and is deterministic: it uses neither the conditional law of Lemma~\ref{sq:lem:lazy} nor any probabilistic bound. We distinguish a quantity being mathematically defined from its being known to a player.

\paragraph{1. Calendar-round truth and the induction.}
Attribute every raw arrival to the round of its observation, even if its owner has not yet decoded the assignment. For an occupied follower use its online recording rule; before its DOWN acceptance use the true assignment in~\eqref{protocol:eq:sparse-pathwise} and its saved observation. During collection all arms belong to the coordinator, whose recording rule specifies their arrivals. On disabled rounds all arrivals and services are zero. This defines $\Delta_r(j)$ without any logical slow set. Starting with zero queues, the recurrence~\eqref{sq:eq:queue} defines $\Gamma_r(j)$ and $\mathsf B_{r+1}(j)$, and $Z_j(v)=\sum_{r<v}\Gamma_r(j)$. A late calculation is correct if it recovers these quantities; it creates no second arrival or service.

Order operations by round, by their order in Algorithms~\ref{sq:alg:clock} and~\ref{protocol:alg:order}, and then by player label. We show that a first failing operation cannot exist. Besides the bounds below, the induction maintains three information invariants.
\begin{enumerate}[label=(\roman*),itemsep=1pt]
\item At an assignment start $\sigma$, all followers are unoccupied, and the coordinator has $z(\sigma)=Z(\sigma)$ and $\mathsf B_\sigma$.
\item During a live assignment, the coordinator maintains exact queues and scores on $F=[K]\setminus U$ and retains the base scores $Z_{a_i}(\sigma)$ of the follower arms. A follower that has decoded maintains the exact queue and cumulative service of its arm; an undecoded follower keeps the required candidate observations.
\item During collection, each unrecovered follower arm has a retained base score, a buffer of its collection arrivals, and no online service credited to the accumulator $c$. Each recovered arm has an exact current queue, and its service is credited exactly once.
\end{enumerate}
The domain bounds, maintained whenever the quantities are defined, are
\begin{equation}\label{protocol:eq:valid-prefix-bounds}
\begin{aligned}
 \Delta_r(j),\Gamma_r(j)&\in\mathbb Z\cap[0,D],&
 Z_j(v),\mathsf B_v(j)&\in\mathbb Z\cap[0,TD],\\
 Z_j(v)+\mathsf B_v(j)&=\sum_{r<v}\Delta_r(j),&
 S_i,c_j,\mathsf b_i&\in\mathbb Z\cap[0,TD],\\
 \delta_i&\in\mathbb Z\cap[0,D],&
 \mathrm{state}_i&\in\{0,1,2\},\qquad \mathcal P\subseteq[d].
\end{aligned}
\end{equation}
Before launch the coordinator does no ordinary learning or UP decoding. An activated follower is unoccupied, logs DATA candidates, offers the placeholder UP record, and accumulates or clears DOWN batches; reference-compatible bootstrap states exclude a prelaunch DOWN acceptance. These operations use public arms, bounded observations, binary equations and at most one candidate record per round, so they pass the checks. At launch the accepted starts, scores, queues and served sums are zero, all followers are unoccupied, and the UP placeholder is $(0,0,0,0)$. Candidate lists and DOWN batches may be nonempty. Candidate observations predate the first assignment and are excluded by its start time; a DOWN batch may straddle launch, but by hypothesis any word accepted from it is the word currently sent. This is the base case.

\paragraph{2. Integral queues and their bounds.}
The rate choice gives $\eta D\in\mathbb Z_{>0}$, $0<\eta\le\xi_j$ and $4\eta\le1$. Every eligible observation lies in $[0,1]$, so a recorded arrival satisfies $0\le\Delta_r(j)=\lfloor D\eta Y_{i,r}/\xi_j\rfloor\le D$; unrecorded arrivals are zero. A service is an integer in $[0,4\eta D]\subseteq[0,D]$ and cannot make a queue negative. Adding the score and queue recurrences gives the conservation identity in~\eqref{protocol:eq:valid-prefix-bounds}. Each arm has at most $T$ arrivals, each at most $D$, so every score, queue and partial service sum is at most $TD$. These bounds hold even after the analytic threshold $B_{\rm cap}$ is exceeded. A replayed queue obeys them once the replay is shown to start from the correct queue and to use the correct arrivals; paragraphs~3 and~4 show this.

\paragraph{3. DOWN acceptance and delayed service.}
Consider a correct DOWN word accepted after round $s$, for the assignment started at $\sigma$ with first cutoff $u$. Its initial-queue field is $\mathsf B_\sigma(a_i)$ by~(i). The follower was unoccupied when this assignment began and has not applied this word before, since acceptance requires a strictly larger start than its last accepted one; it stays unoccupied until acceptance. Hence it has logged every candidate needed on $[\sigma,\min\{u,s+1\})$, and no earlier acceptance has cleared them. The word gives $\sigma$, $a_i$ and $\xi_{a_i}=2^{-h_i}$, the list supplies observations and thinning values, and the selected row supplies the past proposal bits. Scanning until the first proposal or grid boundary gives $\ell=\min\{u,s+1\}$; if the assignment is still live, $B_s$ is read after the round and acts at $s+1$. By~\eqref{protocol:eq:sparse-pathwise} the replay recovers each $\Delta_r(a_i)$ for $\sigma\le r<\ell$, and applying the queue recurrence on \emph{every} such round, including zero-arrival rounds, yields
\[
 S_i=\sum_{r=\sigma}^{\ell-1}\Gamma_r(a_i),\qquad
 \delta_i=\Gamma_{\ell-1}(a_i),\qquad
 \mathsf b_i=\mathsf B_\ell(a_i),
\]
with the bounds of paragraph~2. If $\ell=s+1$, later live rounds use the online queue rule; if $\ell\le s$, the follower stays unoccupied with its fixed endpoint reply. In neither case does it need any other player's state. An occupied follower knows whether the next round is interrupted, from the grid boundary of its start and the live proposal bits; at the first cutoff it becomes unoccupied and keeps the three endpoint fields until a newer DOWN word is accepted. This also establishes, on the valid prefix, the identities of Lemmas~\ref{protocol:lem:local-state} and~\ref{protocol:lem:identities}.

\paragraph{4. UP acceptance and exact collection accounting.}
At cutoff $u$ the coordinator keeps the known cutoff scores, queues and last services on $F$, sets $c=0$, retains $Z_{a_i}(\sigma)$ for each old follower arm, and buffers that arm's nonzero collection arrivals until its cutoff queue is recovered. These arrivals are computable: during collection the coordinator owns every arm, plays its pilot on DATA and uses the old fixed rates, none of which involves the missing queue or score. A missing buffer entry means a zero arrival, not zero service.

Suppose the current UP reply of follower $i$ is accepted after round $s$. The test $M.\sigma=\sigma$ identifies the current assignment, and paragraph~3 gives
\[
 M.S=\sum_{r=\sigma}^{u-1}\Gamma_r(a_i),\qquad
 M.\delta=\Gamma_{u-1}(a_i),\qquad M.q=\mathsf B_u(a_i).
\]
The coordinator recovers $Z_{a_i}(u)$ by adding $M.S$ to the base score, then replays the queue for $r=u,\ldots,s$, using buffered arrivals before $\bar\tau$ and freezing the queue at and after $\bar\tau$. This recovers $\mathsf B_{s+1}(a_i)$ and sets $c_{a_i}=\sum_{r=u}^{s}\Gamma_r(a_i)$. There is no double counting: the queue was unknown before this decision, so $\Gamma_s^{\rm loc}(a_i)=0$ and $\operatorname{FinishRound}$ adds nothing for this arm in round $s$; from round $s+1$ on, its services are added online exactly once. Other arms continue online throughout. Therefore, after the last current reply has been accepted and before the next assignment starts at $v$,
\[
 c=\sum_{r=u}^{v-1}\Gamma_r,\qquad z(u)=Z(u),\qquad z(u)+c=Z(v),
\]
all queues $\mathsf B_v$ are known, and the saved last services give $Z(u-1)=Z(u)-\Gamma_{u-1}$, with nonnegative coordinates, as needed for $w_{u-1}$.

A current UP reply is offered only while its sender is unoccupied. After offering it and before the next start, the sender cannot become occupied again: no newer DOWN word exists, and the old start fails the strict acceptance test. Hence all followers are unoccupied at $v$, which proves~(i) for the next assignment, also when collection crosses a grid boundary. The argument establishes the hypotheses of Lemma~\ref{sq:lem:reconstruct}; it does not assume its conclusion for a later assignment.

\paragraph{5. Every sampling operation has a valid local input.}
All current and anchor weights are positive, so product-set and orientation probabilities have positive normalizers. The coordinator knows all scores at an assignment start, and during live fast play it knows exactly the outside scores, which suffice for~\eqref{sq:eq:q} because the current follower scores cancel. The parameters satisfy $0<h=4m\eta\le1/4$ and $0\le\gamma\le1$.

The only sampler needing a separate check is a same-episode assignment at $v<\bar\tau$ after cutoff $u$. It draws an $m$-set from
\[
 f_v(S)=Q_{w_v}(S)+\frac{(1-h)^{v-u+1}}h
                       \bigl(Q_{w_u}(S)-Q_{w_{u-1}}(S)\bigr),
\]
computable from the three score vectors of paragraph~4, and then applies the fixed-anchor orientation rule. Bounded service gives $1\le w_{r+1,j}/w_{r,j}\le e^{4\eta}$, so every $m$-set satisfies $Q_{w_{r+1}}(S)\ge e^{-4m\eta}Q_{w_r}(S)\ge(1-h)Q_{w_r}(S)$. Consequently
\[
 f_v(S)\ge\frac{(1-h)^{v-u}}h
            \bigl(Q_{w_u}(S)-(1-h)Q_{w_{u-1}}(S)\bigr)\ge0,
 \qquad \sum_{|S|=m}f_v(S)=1.
\]
This check is purely algebraic; that the resulting slow set has the law of a hidden logical variable is the separate assertion of Lemma~\ref{sq:lem:lazy}.

In the sequential sampler, let $W$ be the mass of the $m$-sets consistent with the choices already made, and $W_{\rm in}$ the mass of those that also contain the next arm. Initially $W=1$, and $0\le W_{\rm in}\le W$. With the prescribed coordinate $V\in[0,1)$, include the arm exactly when $W_{\rm in}>VW$. On inclusion the new mass is positive; otherwise it is $W-W_{\rm in}\ge(1-V)W>0$. Thus every reached normalizer is positive, for $f_v$ as for an ordinary product law, and the orientation has positive weights.

Finally, the recording probabilities are locally available. An occupied follower uses its received rate and $\pi_i(a_i)=1$; the coordinator uses $\pi_0(j)=1/K$ for pilot play or its known fast-arm probability; an unoccupied follower only logs a candidate and postpones its thinning test until it receives the rate. The bounds preceding~\eqref{sq:eq:thinning} show that every denominator is positive and every thinning ratio is at most one; they use feasibility and the arm rules, not the coupling. The single recorder source allows at most one physical raw arrival per round, hence at most one buffer entry per arm and timestamp.

\paragraph{6. Message fields, branch selection and closure.}
Every assignment consists of distinct arms in $[K]$, and its rate exponent satisfies $0\le h_i\le b_N$ because $\xi_{a_i}\ge\beta\ge\eta\ge D^{-1}$. By paragraphs~2--4, every outgoing word belongs to its record set and fits the field width of Lemma~\ref{protocol:lem:fields}; an inadmissible decoder output is merely rejected, with no state update. Follower arm rules use assignment fields only in occupied states, which are entered only after acceptance; unoccupied followers use pilots, tuple arms and their existing UP record. Slot and selected-player values are public, and the if--else order selects exactly one branch. The coordinator uses a fast-arm law only with a live assignment and all confirmations; its other branches use public arms or stored current messages. States take only the values $0,1,2$, $\mathcal P$ is initialized to $[d]$ or loses an element, and stored times are calendar indices clipped at $T+1$.

Thus a putative first failing operation has all required local inputs, positive denominators, admissible probabilities and bounded integral outputs, so it cannot fail. Induction proves the lemma through round $t$.
\end{proof}

\subsection{Decoder error bound}\label{protocol:decoder-proof}

Recall the ordinary message width $B=5\lceil\log_2(TD+1)\rceil$ and positive-test quota $Q_{\rm dec}=B+\lceil\log_2(4dTN^2)\rceil$.

\begin{proof}[Proof of Lemma~\ref{sq:lem:decoder}]
Put $Q=Q_{\rm dec}$. For a completed batch on a directed channel, let $r_1<\cdots<r_Q$ be its positive-test rounds and $M_r\in\mathbb F_2^B\cup\{\bot\}$ the word encoded at round $r$. Acceptance tests and fields such as $M.\sigma$ refer to the parsed record in that channel's direction.

\paragraph{An admissible wrong word was not sent in its batch.}
If a candidate $M\ne M_t$ passes the acceptance test at decision time $t$, then
\begin{equation}\label{protocol:eq:stale-batch}
 M_{r_q}\ne M\qquad(q\in[Q]).
\end{equation}
\begin{proof}
Suppose $M_{r_q}=M$ for some $q$. On DOWN, $M$ is a fixed assignment record with start $\sigma=M.\sigma$. Let $\mathcal P_t^-$ be $\mathcal P$ at round $t$'s arm-choice step, and $\sigma_i(t^-)$ follower $i$'s last accepted DOWN start before its decoder decision at $t$.

If the coordinator still retains the assignment starting at $\sigma$, Algorithm~\ref{sq:alg:handover} encodes $M$ on DOWN selecting $i\in\mathcal P_t^-$. Thus $M_t\ne M$ implies $i\notin\mathcal P_t^-$. Before cutoff, this requires an earlier DOWN confirmation. By Lemma~\ref{sq:lem:avoidance}, $i$ was then occupied; it began unoccupied and can become occupied only by accepting $M$.

If this assignment is retained at or after cutoff, the reset $\mathcal P=[d]$ makes $i\notin\mathcal P_t^-$ require prior acceptance of its UP record, which by Lemma~\ref{protocol:lem:identities} is offered only after accepting $M$.

If a successor assignment has started by round $t$, Lemma~\ref{protocol:lem:identities} likewise requires prior acceptance of $i$'s UP record for the $\sigma$-assignment, hence of $M$. Since accepted DOWN starts strictly increase, all cases give $\sigma_i(t^-)\ge M.\sigma$, contradicting $\sigma_i(t^-)<M.\sigma$ in~\eqref{protocol:eq:acceptance}.

On UP, acceptance requires $M.\sigma=\sigma>0$ and $i\in\mathcal P$, excluding the initial placeholder. By Lemma~\ref{protocol:lem:identities}, the UP record stays fixed until a newer DOWN record is applied, and no successor assignment starts before its acceptance. Thus $M_t\ne M$ implies either $i\notin\mathcal P$ or a current start greater than $M.\sigma$, both contradicting UP acceptance.
\end{proof}

\paragraph{One-batch estimate.}
Fix a channel (UP or DOWN), a batch start $\alpha$, and a deterministic candidate $M\in\mathbb F_2^B$. Its round-$r$ map is $H_r(x)=c_r\cdot x\oplus b_r$, with a fresh pair $(c_r,b_r)$ (Section~\ref{sq:sec:message-decoding}). Let $\tau_q$ be the round of the $q$-th positive test before stopping, or $\infty$ if it does not occur. Let $D_0$ be the sure event and, for $q=1,\ldots,Q$, set
\[
 D_q:=\{\tau_q<\infty,\ M_{\tau_j}\ne M,\
H_{\tau_j}(M)=1\text{ for all }j\le q\}.
\]
Conditional on the batch-start history $\mathcal H_\alpha$, we prove
\begin{equation}\label{protocol:eq:batch-tail}
 \Pr(D_Q\mid\mathcal H_\alpha)\le 2^{-Q}.
\end{equation}

\begin{proof}
Let $\mathcal H_r^-$ be the history exposed before reading the current pair $(c_r,b_r)$, and let $A_r$ indicate that round $r$ adds a positive test to this batch. If $M_r\notin\{M,\bot\}$, the values $(H_r(M_r),H_r(M))$ are independent fair bits conditional on $\mathcal H_r^-$. In this case, with the other current controls fixed, Algorithm~\ref{sq:alg:handover} makes $A_r$ depend on $(c_r,b_r)$ only through $H_r(M_r)$. If $M_r=\bot$, the channel action and $A_r$ do not use this map, including before activation. Consequently,
\begin{equation}\label{protocol:eq:hash-half}
 \E\!\left[A_r\ind\{M_r\ne M,\ H_r(M)=1\}
       \mid\mathcal H_r^-\right]
 =\frac12\E\!\left[A_r\ind\{M_r\ne M\}
       \mid\mathcal H_r^-\right].
\end{equation}

Let $N_{r-1}$ count the batch's tests before round $r$, and let $I_{q,r}$ indicate that the batch is active, $N_{r-1}=q-1$, and $D_{q-1}$ holds. This indicator is $\mathcal H_r^-$-measurable. Summing~\eqref{protocol:eq:hash-half} over the possible round of the $q$-th test gives
\[
 \Pr(D_q\mid\mathcal H_\alpha)
 =\frac12\E\!\left[\sum_{r=\alpha}^T
      I_{q,r}A_r\ind\{M_r\ne M\}\,\middle|\,\mathcal H_\alpha\right]
 \le\frac12\Pr(D_{q-1}\mid\mathcal H_\alpha).
\]
The $q$-th test occurs at most once; iteration for $q=1,\ldots,Q$ proves~\eqref{protocol:eq:batch-tail}.
\end{proof}

\paragraph{Union bound and decoding time.}
Let $E_{\rm false}$ be the event that a completed batch has a unique solution $M\ne M_t$ passing the acceptance test. Such a solution satisfies $H_{r_q}(M)=1$ for every retained equation and $M_{r_q}\ne M$ by~\eqref{protocol:eq:stale-batch}.

Union over at most $2d$ directed channels, $T$ batch starts $\alpha\in[T]$ per channel, and $2^B$ candidates per batch. By~\eqref{sq:eq:decode-quota} and~\eqref{protocol:eq:batch-tail},

\[
 \Pr(E_{\rm false})
 \le 2dT2^{B-Q}\le\frac{1}{2N^2}.
\]

Let $E_{\rm amb}$ be the event that a completed batch of positive tests of one fixed awaited word $M^*$ has multiple solutions. The transmitted $M^*$ solves every equation; any other solution $M$ satisfies $M_{r_q}=M^*\ne M$ and $H_{r_q}(M)=1$ for all $q$. The same conditional estimate and union bound yield
\[
 \Pr(E_{\rm amb})\le\frac{1}{2N^2}.
\]
The bound is uniform when $M^*$ is chosen from the batch-start history, so $\Pr(E_{\rm false}\cup E_{\rm amb})\le N^{-2}$.

Suppose $M^*$ becomes awaited with $n\in\{0,\ldots,Q-1\}$ equations in its receiver's batch. After $Q-n$ further positive receiver observations, the decoder decides and clears that batch, possibly with older equations. If $M^*$ remains awaited, the next $Q$ positives all encode it, so it solves every equation. On $E_{\rm amb}^c$ the solution is unique and passes~\eqref{protocol:eq:acceptance}. Thus $(Q-n)+Q\le2Q$ further positive observations ensure acceptance.

Gaussian elimination solves $Q$ binary equations in $B$ unknowns in $O(QB^2)=O(\log^3N)$ elementary operations; parsing costs $O(B)$.
\end{proof}

\section{Public construction and bootstrap communication}
\label{comm:appendix}\label{comm:section}

\subsection{Public table construction}\label{sq:sec:masks}\label{comm:parameters}\label{comm:messages}\label{comm:arrays-construction}\label{comm:control-rows}
The public table consists of a common ID array and independent complete rows, each containing bootstrap forward bits and learning/control coordinates.

\paragraph{Parameters.}
Recall $N=2Km(T+1)$, $b_N=\lceil\log_2(N+1)\rceil$, and the score denominator $D=2^{b_N}$. The bootstrap block length and number of blocks from~\eqref{sq:eq:boot-block} are
\begin{equation*}
 B_{\rm boot}=\lceil\sqrt{T/(mK)}\rceil,\qquad
 n_{\rm boot}=\lceil T/B_{\rm boot}\rceil\,.
\end{equation*}
The last block may be shorter. Recall also $M=\lceil8T^2/(mb_N)\rceil$, $A_{\rm id}=\lceil T(B_{\rm boot}+b_N)/b_N\rceil$ and $L_0=2^{\lceil\log_2(2b_N)\rceil}$ from Section~\ref{sq:sec:boot}. Write $A:=A_{\rm id}$ and $d:=m-1$. Players are labelled $0,\ldots,m-1$, with coordinator $0$ and followers $1,\ldots,d$.

The fixed universal constants $C_b=16$ and $N_0=2^{28}$ in~\eqref{sq:eq:bootbranch} are independent of $K,m,T$; Appendix~\ref{app:explicit-constants} verifies their admissibility. Outside the construction branch, the players use a fixed assignment to distinct arms.

\paragraph{ID and forward arrays.}
The ID array has indices $(u,t)\in[A]\times[T]$, and the forward arrays have indices $(q,u,t)\in[M]\times[A]\times[T]$. Sample their entries independently with laws
\begin{equation*}
 I_u(t)\sim\operatorname{Bernoulli}(1/2),\qquad
 \varphi_q(u,t)\sim\operatorname{Bernoulli}(1-1/L_0)\,.
\end{equation*}
All these bits are independent of the learning/control coordinates.

\paragraph{Learning/control rows.}
The fixed set of scalar-coordinate labels is
\begin{equation*}
 \begin{aligned}
 \mathcal C_{\rm row}={}&\{\mathrm{slot},\mathrm{sel},\mathrm{rec},\mathrm{fast},\mathrm{thin},\mathrm{prop}\}\\
 &{}\cup\{\mathrm{tuple}_r,\mathrm{pilot}_r:0\le r<m\}
 \cup\{\mathrm{assign}_r:1\le r\le K+1\}\,.
 \end{aligned}
\end{equation*}
For every $(q,t,c)\in[M]\times[T]\times\mathcal C_{\rm row}$, independently sample and publish
\begin{equation*}
 V_{q,t,c}\sim\Unif([0,1))\,.
\end{equation*}
Thus each row contains $T(2m+K+7)$ scalar entries. Coordinates are reserved for every potential call, including calls in branches that are not taken.

Categorical draws use inverse cumulative probabilities in a fixed order. The coordinates $\mathrm{slot},\mathrm{sel},\mathrm{rec}$ choose, respectively, the slot with probabilities $(1/2,1/4,1/4)$ for $(\mathrm{DATA},\mathrm{DOWN},\mathrm{UP})$, the uniform selected follower in $[d]$, and the DATA recorder source with probabilities $s_0=1/2$ and $s_i=1/(2d)$ for $i\in[d]$.

In a fixed row $q$, construct the tuple in order $r=0,\ldots,m-1$: $R_{r,t}$ is the arm of rank $1+\lfloor(K-r)V_{q,t,\mathrm{tuple}_r}\rfloor$ in the increasing list $[K]\setminus\{R_{0,t},\ldots,R_{r-1,t}\}$. The tuple is injective, and each ordered injective tuple has probability $1/[K(K-1)\cdots(K-m+1)]$. The DATA pilots are $P_{r,t}=1+\lfloor K V_{q,t,\mathrm{pilot}_r}\rfloor$, so they are independent uniform arms. The disjoint scalar coordinates make the tuple, pilots, slot, selected follower, and recorder-source inputs independent.

The coordinate $\mathrm{fast}$ supplies the fast-arm draw, while $\mathrm{thin}$ and $\mathrm{prop}$ supply their threshold tests. The coordinates $\mathrm{assign}_1,\ldots,\mathrm{assign}_K$ supply the sequential inclusion tests for assignment sampling, and $\mathrm{assign}_{K+1}$ supplies the orientation draw~\eqref{sq:eq:orientation}; Proposition~\ref{sq:prop:complexity} gives the computation of the corresponding probabilities.

Let $B=5\lceil\log_2(TD+1)\rceil$ be the message width from Section~\ref{sq:sec:transport}. For every $(q,t,\sigma,i)\in[M]\times[T]\times\{\mathrm{DOWN},\mathrm{UP}\}\times[d]$, independently sample
\begin{equation*}
 (c_{q,t,\sigma,i},b_{q,t,\sigma,i})
 \sim\Unif(\mathbb F_2^B\times\mathbb F_2)\,,
\end{equation*}
independently of all scalar coordinates. The learning/control component is
\begin{equation*}
 \bigl((V_{q,t,c})_{t,c},(c_{q,t,\sigma,i},b_{q,t,\sigma,i})_{t,\sigma,i}\bigr)\,.
\end{equation*}
These components are independent across $q$ and independent of the forward arrays and the common ID array. On a selected DOWN or UP channel, the affine binary map is $u\mapsto c_{q,t,\sigma,i}\cdot u\oplus b_{q,t,\sigma,i}$.

\paragraph{Construction and publication.}
Write $R_q$ for the complete row consisting of $(\varphi_q(u,t))_{u,t}$ and the learning/control component. Each row has $T(2m+K+7)$ scalar entries and $T[A+2d(B+1)]$ binary entries. After the table is published and the adversary fixes the rewards, the coordinator privately draws $J\sim\Unif([M])$.

\begin{algorithm}[H]
\caption{Construct and publish the public table}\label{sq:alg:setup}\label{comm:setup}
\normalsize
\begin{algorithmic}[1]
\Require $K,m,T$.
\State $N\gets2Km(T+1)$, $b_N\gets\lceil\log_2(N+1)\rceil$, $M\gets\lceil8T^2/(mb_N)\rceil$.
\If{$T\le C_bmKb_N^2$ or $N<N_0$}
 \State Publish the fixed-arm fallback flag; stop.
\EndIf
\State Set $B_{\rm boot},A,L_0$ by~\eqref{sq:eq:boot-block}--\eqref{sq:eq:key-alphabet}.
\State Draw independent $I_u(t)\sim\operatorname{Bernoulli}(1/2)$ for all $u\in[A],t\in[T]$.
\For{$q\in[M]$}
 \State Draw independent $\varphi_q(u,t)\sim\operatorname{Bernoulli}(1-1/L_0)$ for all $u,t$.
 \State Independently sample the learning/control component above.
\EndFor
\State Publish the ID table and all complete rows, including every scalar entry $V_{q,t,c}$.
\end{algorithmic}
\end{algorithm}

\begin{table}[H]
\centering
\begin{tabularx}{\linewidth}{@{}l >{\raggedright\arraybackslash}X >{\raggedright\arraybackslash}X@{}}
\toprule
Component & Stored data & Sampling law\\
\midrule
ID table & $I_u(t)\in\{0,1\}$ for $u\in[A]$, $t\in[T]$. & Independent fair bits.\\
Bootstrap row & Forward bits $\varphi_q(u,t)$ for $u\in[A]$, $t\in[T]$. & Independent bits of mean $1-1/L_0$.\\
Learning/control row & Scalar coordinates and affine binary maps. & Independent $\Unif([0,1))$ scalars and uniform affine binary maps.\\
\bottomrule
\end{tabularx}
\caption{Components of the public table on the construction branch.}
\label{comm:tab:public-table}
\end{table}

\subsection{Bootstrap decoding and outward protocol}\label{comm:outward-protocol}
The protocol reads no row coordinate before its calendar time.

\paragraph{Coordinator before launch.}
Use the launch threshold $Q$ from~\eqref{sq:eq:launch-budget}.
The coordinator follows the $i=0$ branch of Algorithm~\ref{sq:alg:activation}. Its true slot is the slot in $R_J$ at time $t$, and $R_{0,t}$ is its labelled tuple arm. The counters $C_r(a)$ count cumulative coordinator DOWN positives on arm $a$ through block $r$. At each block boundary the coordinator computes $\widehat D_r$ by~\eqref{sq:eq:root-monitor}; if $\widehat D_r>4Q$, it launches the zero-score, zero-queue learner in the following round. Its ordinary UP decoder is disabled before launch. Write $z\in[K]$ for the block target and $\Lambda\subseteq[A]$ for the identification list. Products over an empty list equal one.

\paragraph{Uninformed followers.}
Fix a follower $i\in[d]$ that has not yet declared a row index at the start of block $s$. Abbreviate its positive bootstrap weights by $w_i(a):=w_i^{\rm boot}(a)$. The held-arm distribution from~\eqref{comm:boot-mixture} is
\begin{equation*}
 p_i(a)=\frac12\left(\frac{w_i(a)}{\sum_z w_i(z)}+\frac1K\right),
 \qquad a\in[K]\,.
\end{equation*}
The follower independently becomes active with probability $1/m$. If active, it draws the held arm $A^{\rm hold}_{i,s}\sim p_i$, then a fresh private identification index $\iota\sim\Unif([A])$. These choices precede all access to current-block row coordinates. Write $a=A^{\rm hold}_{i,s}$ for the held arm; $\iota$ is retained throughout the block and discarded at its end. A stay position has $I_\iota(t)=0$, when the follower plays $a$.

For each candidate row $v\in[M]$, let $\tau_{v,\ell}$ denote its slot at time $\ell$. With $Y_{i,\ell}$ denoting the rounded observation, the eligible positive set at time $t$ is
\begin{equation}\label{comm:positive-set}
 P_i(v;t)=\{\ell\in\mathcal B_s:\ell\le t,\ I_\iota(\ell)=0,
       \tau_{v,\ell}=\mathrm{DATA},\ Y_{i,\ell}>0\}\,.
\end{equation}
Set $q_{\rm dec}=8L_0b_N$. Candidate $v$ passes precisely when
\begin{equation}\label{comm:decoder}
 |P_i(v;t)|\ge q_{\rm dec},\qquad
 \varphi_v(\iota,\ell)=1\quad\text{for every }\ell\in P_i(v;t)\,.
\end{equation}
Before a declaration is saved, every eligible positive is included in the test, and a mismatch rejects that candidate for the rest of the block. A unique passing candidate is saved as a pending declaration and is applied only at the end of the block.

For an active block, let
\begin{equation*}
 G_{i,s}=\sum_{t\in\mathcal B_s}\ind\{I_\iota(t)=0,\ Y_{i,t}>0\}\,;
\end{equation*}
set $G_{i,s}=0$ if the follower is inactive. Use the boundary update~\eqref{sq:eq:boot-weight}, written in Algorithm~\ref{comm:alg:raw-decoder} with precision $P=4b_N$.
All other arm weights are unchanged. A pending declaration changes neither arm choices nor gain counting before the boundary, so the gain includes the entire block. Until a declaration is saved, Algorithm~\ref{comm:alg:raw-decoder} maintains $n_v=|P_i(v;t)|$ and a permanent mismatch flag $\mathrm{dead}_v$; its gain counter satisfies $G=G_{i,s}$ at the boundary.

\paragraph{After declaration.}
After that block ends, the follower enters the ordinary unoccupied state in its declared row from the next round. On DATA it plays its independent uniform pilot. On DOWN it plays its labelled tuple arm and runs the ordinary equation-batch decoder. On UP, when selected, it uses the ordinary sender to encode the initial record $(0,0,0,0)$; otherwise it plays its tuple arm. Until an assignment is accepted, DATA records change none of its accepted start, score, last increment, or queue registers, and it reads no proposal coordinates.

Coordinator launch and follower activation start on the round after a block boundary, so a follower declared at the launching boundary is already activated on the launch round. No separate acknowledgement or launch message is used. The same rules process a shortened final block; a launch after time $T$ has no effect on the horizon.

\begin{algorithm}[H]
\caption{Follower before its first row declaration}\label{comm:alg:raw-decoder}
\begin{algorithmic}[1]
\State $w_i(a)\gets1$ for every $a\in[K]$.
\For{each block $\mathcal B_s$ while uninformed}
 \State Draw an independent activity coin of mean $1/m$.
 \State If active, privately draw $a\sim p_i$, then fresh $\iota\sim\Unif([A])$.
 \State $G\gets0$; clear pending; $(n_v,\mathrm{dead}_v)\gets(0,0)$ for all $v\in[M]$.
 \For{each round $t\in\mathcal B_s$}
  \If{inactive}
   \State Play a fresh private $\Unif([K])$ arm.
  \Else
   \State Play $a$ if $I_\iota(t)=0$; else a fresh private $\Unif([K]\setminus\{a\})$ arm.
   \State Observe the rounded reward $Y_{i,t}$.
   \If{$I_\iota(t)=0$ and $Y_{i,t}>0$}
    \State $G\gets G+1$.
    \If{there is no pending declaration}
     \For{each $v\in[M]$ with $\tau_{v,t}=\mathrm{DATA}$}
      \State $n_v\gets n_v+1$; if $\varphi_v(\iota,t)=0$, set $\mathrm{dead}_v\gets1$.
     \EndFor
     \If{exactly one $v$ has $n_v\ge q_{\rm dec}$ and $\mathrm{dead}_v=0$}
      \State Save $v$ as the pending declaration.
     \EndIf
    \EndIf
   \EndIf
  \EndIf
 \EndFor
 \State If active, $w_i(a)\gets2^{-P}\left\lfloor2^Pw_i(a)\left(1+\frac{G}{4KB_{\rm boot}p_i(a)}\right)\right\rfloor$.
 \If{a declaration is pending}
  \State Accept $v$; from next round use row $R_v$ in Algorithms~\ref{sq:alg:handover} and~\ref{protocol:alg:order}.
  \State $(\sigma_i,S_i,\delta_i,\mathsf b_i,\mathrm{state}_i)\gets(0,0,0,0,0)$.
 \EndIf
\EndFor
\end{algorithmic}
\end{algorithm}

\subsection{Technical Lemmas}

\begin{lemma}[Rare-event upper tail for independent bounded variables]\label{init:lem:rare-tail}
Let $X_1,\ldots,X_r$ be independent random variables in~$[0,1]$. If $a>0$, $\sum_\ell\E X_\ell\le a\mu$, and $0<\mu<\varepsilon$, then
\begin{equation}\label{init:eq:rare-tail}
 \Pr\left\{a^{-1}\sum_\ell X_\ell>\varepsilon\right\}\le\exp\left[-a\varepsilon\log\frac{\varepsilon}{e\mu}\right]\,.
\end{equation}
\end{lemma}
\begin{proof}
For $\lambda>0$, convexity gives $\E e^{\lambda X_\ell}\le1+(e^\lambda-1)\E X_\ell$. Independence and $1+x\le e^x$ therefore give
\begin{equation*}
 \E e^{\lambda\sum_{\ell=1}^r X_\ell}
 \le\exp\left[(e^\lambda-1)\sum_{\ell=1}^r\E X_\ell\right]
 \le\exp\bigl[(e^\lambda-1)a\mu\bigr]\,.
\end{equation*}
Markov's inequality with $\lambda=\log(\varepsilon/\mu)>0$ yields
\begin{align*}
 \Pr\left\{a^{-1}\sum_{\ell=1}^r X_\ell>\varepsilon\right\}
 &\le\exp\left[-a\varepsilon\log\frac{\varepsilon}{\mu}+a(\varepsilon-\mu)\right]\\
 &\le\exp\left[-a\varepsilon\log\frac{\varepsilon}{e\mu}\right]\,.\qedhere
\end{align*}
\end{proof}

For $x\in\{0,1\}^{\mathcal B_s}$ and an interval $D\subseteq\mathcal B_s$, put $S_D(x)=\sum_{t\in D}x_t$.

\begin{lemma}[Uniform identification-array guarantees]\label{comm:arrays}
With probability at least $1-3/N$, the identification array satisfies all the following properties:
\begin{enumerate}[label=(\roman*),leftmargin=*]
\item For every time~$t$ and bit $e\in\{0,1\}$, at least $A/3$ indices have $I_u(t)=e$.
\item Every set of $6b_N$ distinct calendar times has fewer than~$L_0$ indices whose ID bit is one at every one of those times.
\item For every block~$s$ and every binary string $x\in\{0,1\}^{\mathcal B_s}$, at least $3A/4$ indices satisfy, simultaneously for every interval $D\subseteq\mathcal B_s$ and $e\in\{0,1\}$,
\begin{equation}\label{init:eq:local-balance}
 \sum_{t\in D}x_t\ind\{I_u(t)=e\}\ge\frac{S_D(x)}3-8b_N\,.
\end{equation}
\end{enumerate}
\end{lemma}
\begin{proof}
On the construction branch, $T\ge36\log N$. For a fixed column and bit, Hoeffding's inequality bounds failure of~(i) by $e^{-A/18}$. Since $A\ge T$, $2T\le N$ and $T\ge36\log N$, the union over columns and bits is at most $2T e^{-A/18}\le N^{-1}$.

Failure of~(ii) supplies an all-one rectangle with~$L_0$ rows and $6b_N$ columns. If $6b_N>T$ or $L_0>A$, the assertion is automatic. Otherwise its probability is at most
\begin{equation*}
 \binom A{L_0}\binom T{6b_N}2^{-6L_0b_N}\le\exp\bigl(L_0\log A+6b_N\log T-6L_0b_N\log2\bigr)\,.
\end{equation*}
Since $A\le N^2$, $T\le N$, $\log N\le b_N\log2$, and $L_0\ge2b_N$, the exponent is at most $(6b_N^2-4L_0b_N)\log2\le-2b_N^2\log2$. The resulting probability is at most $N^{-2}$.

For~(iii), fix $s,x,u,D,e$ as in~\eqref{init:eq:local-balance}. Markov's inequality and independence give
\begin{equation*}
 \Pr\left\{\sum_{t\in D}x_t\ind\{I_u(t)=e\}<\frac{S_D(x)}3-8b_N\right\}
 \le2^{-8b_N}\left(2^{1/3}\frac34\right)^{S_D(x)}
 \le2^{-8b_N}\,,
\end{equation*}
since $(2^{1/3}\cdot3/4)^3=54/64<1$.

For fixed~$s,x$, let $E_u(s,x)$ be the event that~\eqref{init:eq:local-balance} fails for some interval~$D\subseteq\mathcal B_s$ and bit~$e\in\{0,1\}$. There are at most $2B_{\rm boot}^2$ interval/bit pairs. Since $B_{\rm boot}\le N\le2^{b_N}$ and $b_N\ge5$,
\begin{align*}
 \Pr(E_u(s,x))
 &\le2B_{\rm boot}^2 2^{-8b_N}=:p_0
 \le2^{1-6b_N}\le2^{-5b_N}\,,\\
 \log\frac1{4ep_0}
 &\ge5b_N\log2-\log(4e)\ge2b_N\,.
\end{align*}
The events $E_u(s,x)$ depend on separate rows of the ID array and are independent over~$u$. Lemma~\ref{init:lem:rare-tail}, applied to their indicators, gives
\begin{equation*}
 \Pr\left\{\sum_{u=1}^A\ind_{E_u(s,x)}>\frac A4\right\}
 \le\exp\left[-\frac A4\log\frac1{4ep_0}\right]
 \le e^{-A b_N/2}\,.
\end{equation*}
Using $n_{\rm boot}\le N$, $A b_N\ge T(B_{\rm boot}+b_N)$, $T\ge4$, and $\log N\le b_N\log2$, a union bound gives
\begin{align*}
 \Pr(\text{(iii) fails})
 &\le N2^{B_{\rm boot}}e^{-A b_N/2}\\
 &\le\exp\bigl[-(2-\log2)(B_{\rm boot}+b_N)\bigr]
 \le N^{-1}\,.
\end{align*}
Combining~(i)--(iii) gives total failure probability at most $2N^{-1}+N^{-2}\le 3/N$.
\end{proof}

\begin{lemma}[Interval bounds]\label{init:lem:interval-bound}
Let $(\mathscr H_t)_{t=0}^T$ be a filtration, let $\alpha\in[T]$ be deterministic, and let $\ell>0$.
Let $X_t\in[0,1]$ be $\mathscr H_t$-measurable and put $\mu_t=\E[X_t\mid\mathscr H_{t-1}]$. Then
\begin{equation}\label{init:eq:clock-lower}
 \Pr\left\{\sum_{t=\alpha}^n X_t\ge\frac12\sum_{t=\alpha}^n\mu_t-\ell
       \quad\text{for all }n\in\{\alpha,\ldots,T\}
       \ \middle|\ \mathscr H_{\alpha-1}\right\}
 \ge1-e^{-\ell}\quad\text{a.s.}\,,
\end{equation}
and
\begin{equation*}
 \Pr\left\{\sum_{t=\alpha}^n X_t\le2\sum_{t=\alpha}^n\mu_t+\ell
       \quad\text{for all }n\in\{\alpha,\ldots,T\}
       \ \middle|\ \mathscr H_{\alpha-1}\right\}
 \ge1-e^{-\ell}\quad\text{a.s.}\,.
\end{equation*}
\end{lemma}
\begin{proof}
Set $\kappa=1-e^{-1}\ge1/2$. Convexity on~$[0,1]$ gives
\begin{align*}
 \E[e^{-X_t}\mid\mathscr H_{t-1}]
 &\le1-\kappa\mu_t\le e^{-\kappa\mu_t}\,,\\
 \E[e^{X_t}\mid\mathscr H_{t-1}]
 &\le1+(e-1)\mu_t\le e^{(e-1)\mu_t}\,.
\end{align*}
Consequently, for $n=\alpha-1,\ldots,T$, the processes
\begin{equation}\label{init:eq:clocks}
 \begin{aligned}
 L_n^-&=\exp\left(\kappa\sum_{t=\alpha}^n\mu_t-\sum_{t=\alpha}^n X_t\right),\\
 L_n^+&=\exp\left(\sum_{t=\alpha}^n X_t-(e-1)\sum_{t=\alpha}^n\mu_t\right)
 \end{aligned}
\end{equation}
are nonnegative supermartingales, both initially one, by the preceding conditional-moment bounds. Since $\kappa\ge1/2$ and $e-1<2$,
\begin{align*}
 \left\{\exists n:\sum_{t=\alpha}^n X_t<\frac12\sum_{t=\alpha}^n\mu_t-\ell\right\}
 &\subseteq\left\{\max_{\alpha-1\le n\le T}L_n^->e^\ell\right\},\\
 \left\{\exists n:\sum_{t=\alpha}^n X_t>2\sum_{t=\alpha}^n\mu_t+\ell\right\}
 &\subseteq\left\{\max_{\alpha-1\le n\le T}L_n^+>e^\ell\right\},
\end{align*}
where $n$ ranges over $\{\alpha,\ldots,T\}$. For every $A\in\mathscr H_{\alpha-1}$ and $\nu\in\{-,+\}$, Doob's maximal inequality~\cite[Lemma~5.18]{durrett2012} applied to $(\ind_A L_n^\nu)_{n=\alpha-1}^T$ gives
\begin{equation*}
 \Pr\left(A\cap\left\{\max_{\alpha-1\le n\le T}L_n^\nu>e^\ell\right\}\right)
 \le e^{-\ell}\Pr(A)\,.
\end{equation*}
Since $A$ is arbitrary, each conditional crossing probability is at most $e^{-\ell}$.
\end{proof}

\subsection{Excluding incorrect rows after publication}

A \emph{forward failure} occurs when some $v\ne J$ passes a bootstrap prefix test on the reference observations of Algorithm~\ref{sq:alg:reference}.

\begin{lemma}[Uniform exclusion of false bootstrap rows]\label{comm:acceptance}
For $A=\lceil T(B_{\rm boot}+b_N)/b_N\rceil$ and $q_{\rm dec}=8L_0b_N$, there is a public event of probability at least $1-N^{-1}$ on which the following holds simultaneously for every reward support and every true index~$J$. Conditional on the complete numerical public output and on that fixed index, the private probability of any forward failure is at most $B_{\rm boot}^{-1}$.
\end{lemma}
\begin{proof}
\emph{Conditional block executions.} Fix~$J$. Let $\tau_{v,t}$ be row~$v$'s slot at time~$t$, and condition on
\begin{equation*}
 \mathscr C_J=\sigma\bigl(R_J,(I_u(t))_{u,t},(\tau_{v,t})_{v\ne J,t}\bigr)\,.
\end{equation*}
The false forward bits remain independent with mean $p=1-1/L_0$. Fix a prelaunch block~$s$ and an uninformed follower~$i$, active with held arm~$a$, and condition on the coordinator's target~$z$.

Let $A_{k,t}$ denote player~$k$'s reference arm. Condition further on the preceding private history and on $(A_{k,t})_{k\notin\{0,i\},\,t\in\mathcal B_s}$. These players' blockwise arm rules use no new observations: uninformed players retain their activity and arm rules, while informed players use pilots, tuple arms, and initial UP records. Their private draws are independent of~$\iota$ and the remaining draws of players~$0,i$, so~$\iota$ remains uniform on~$[A]$.

Write $\chi_t(c)=\ind\{r_t(c)>0\}$ for the rounded reward support. For $c\in\{a,z\}$, define
\begin{equation*}
 x_c(t)=\chi_t(c)\prod_{k\notin\{0,i\}}\ind\{A_{k,t}\ne c\},\qquad t\in\mathcal B_s\,.
\end{equation*}
The coordinator's DISCOVER updates use~$x_z$; follower~$i$'s positive tests use~$x_a$. Both also depend on collisions between these two players. The coordinator does not end MULTICAST in response to observations, and declarations and launch take effect only at block boundaries. Thus, given~$\mathscr C_J$, the tuple $(s,i,a,z,x_a,x_z)$ and the remaining private draws determine every prefix test. Continuing tests after a declaration is saved changes no arm choice and only enlarges the failure event.

\emph{Independent false-row slices.} Call each choice $(s,J,i,a,z,x_a,x_z)$ a context. Fix one context as deterministic input to a virtual block. We do not condition on the context realized by a table-dependent choice of rewards: doing so could reveal false-row bits. Instead, we prove a bound for each deterministic context and then take a union over every context. The resulting public event therefore covers whichever context is realized after the table is published. For $v\ne J$ and $u\in[A]$, let
\begin{equation*}
 X_{v,u}=(\varphi_v(u,t))_{t\in\mathcal B_s}.
\end{equation*}
For a numerical slice~$\xi$, define
\begin{equation*}
 g_{v,u}(\xi)=\Pr_\omega\left\{\exists t\in\mathcal B_s:\ |P_i(v;t)|\ge q_{\rm dec},\ \prod_{\ell\in P_i(v;t)}\xi_\ell=1\ \middle|\ \iota=u\right\}\,.
\end{equation*}
Here $\omega$ denotes the remaining local private draws. Thus $g_{v,u}:\{0,1\}^{\mathcal B_s}\to[0,1]$ is a private passing probability, and $g_{v,u}(X_{v,u})$ is random under public sampling. The eligible sets~\eqref{comm:positive-set} inspect no false forward bits, so these random variables are independent across~$(v,u)$, conditional on~$\mathscr C_J$ and the fixed context.

Keep this conditioning and fix $\iota=u$ and~$\omega$. The eligible sets~$P_i(v;t)$ are nested and independent of~$X_{v,u}$: reference observations use only the true row, and~$\mathscr C_J$ fixes all candidate slots. List the block's eligible times as $\ell_1<\cdots<\ell_L$. By the permanent mismatch rule, some prefix passes precisely when $L\ge q_{\rm dec}$ and the first~$q_{\rm dec}$ entries $\varphi_v(u,\ell_j)$ are all one. These are distinct independent Bernoulli entries of mean~$p$. Averaging over~$\omega$ conditional on $\iota=u$ gives
\begin{equation*}
 \E_{X_{v,u}}[g_{v,u}(X_{v,u})]
 =\E_{\omega\mid \iota=u}\bigl[\ind\{L\ge q_{\rm dec}\}p^{q_{\rm dec}}\bigr]
 \le p^{q_{\rm dec}}\,.
\end{equation*}
Thus
\begin{equation}\label{init:eq:local-false-mean}
 \frac1A\sum_{\substack{v\in[M]\setminus\{J\}\\u\in[A]}}
       \E_{X_{v,u}}[g_{v,u}(X_{v,u})]
 \le Mp^{q_{\rm dec}}\le N^3e^{-8b_N}=: \mu\,.
\end{equation}
Here $M\le N^2$ and $p^{q_{\rm dec}}\le e^{-q_{\rm dec}/L_0}=e^{-8b_N}$. For a fixed table and context, the probability that any incorrect row passes in this block is at most $A^{-1}\sum_{v\ne J,u}g_{v,u}(X_{v,u})$.

Set $\varepsilon_{\rm loc}=1/(B_{\rm boot}n_{\rm boot})$. Lemma~\ref{init:lem:rare-tail} gives
\begin{equation}\label{init:eq:local-false-tail}
 \Pr\left\{\frac1A\sum_{v\ne J,u}g_{v,u}(X_{v,u})>\varepsilon_{\rm loc}
       \ \middle|\ \mathscr C_J\right\}
 \le\exp[-\tfrac52(B_{\rm boot}+b_N)]\,,
\end{equation}
since $B_{\rm boot}n_{\rm boot}\le2T\le N$ implies $\varepsilon_{\rm loc}\ge N^{-1}$ and, for the specified~$N_0$,
\begin{equation*}
 \log\frac{\varepsilon_{\rm loc}}{e\mu}\ge8b_N-4\log N-1\ge5b_N,\qquad
 A\varepsilon_{\rm loc}\ge\frac{B_{\rm boot}+b_N}{2b_N}\,.
\end{equation*}

\emph{A common public event.} For fixed~$J$, let $\mathcal G_J$ be the event that $A^{-1}\sum_{v\ne J,u}g_{v,u}(X_{v,u})\le\varepsilon_{\rm loc}$ in every local context. Set $\mathcal G_{\rm test}=\bigcap_{J=1}^M\mathcal G_J$. Conditional on~$\mathscr C_J$, there are at most $n_{\rm boot}dK^2 2^{2B_{\rm boot}}$ contexts. Taking expectations of their union bound and a union over~$J$, with $n_{\rm boot}MdK^2\le N^6$, gives
\begin{equation*}
 \Pr(\mathcal G_{\rm test}^c)
 \le N^6 2^{2B_{\rm boot}}e^{-\frac52(B_{\rm boot}+b_N)}\le N^{-1}\,.
\end{equation*}
Here $B_{\rm boot}\ge4b_N$ and $\log N\le b_N\log2$ imply $\tfrac52(B_{\rm boot}+b_N)\ge2B_{\rm boot}\log2+7\log N$. The union over~$x_a,x_z$ includes every reward support and every choice of the other players' arms.

\emph{Private executions on the fixed table.} Fix a table in~$\mathcal G_{\rm test}$, an index~$J$, and any reward support; write $\Pr_{\rm priv}$ for private probability. Condition on a block's past and the other players' arm choices. For an active follower holding~$a$ and coordinator target~$z$, this fixes~$x_a,x_z$. Its false-pass probability is at most~$\varepsilon_{\rm loc}$. Averaging these choices and using the activity probability~$1/m$ gives
\begin{align*}
 \Pr_{\rm priv}\{\text{any forward failure}\}
 &\le\varepsilon_{\rm loc}\sum_{s=1}^{n_{\rm boot}}\sum_{i=1}^d
       \Pr_{\rm priv}\{i\text{ is uninformed and active in block }s\}\\
 &\le\varepsilon_{\rm loc}\,n_{\rm boot}\frac dm
 \le B_{\rm boot}^{-1}\,.
\end{align*}
After launch no reference follower is uninformed.
\end{proof}

On this public event, the expected regret cost of forward failures is at most
\begin{equation}\label{init:eq:authentication-charge}
 \frac{2mT}{B_{\rm boot}}\le2m^{3/2}\sqrt{KT}\,.
\end{equation}
Intersecting with Lemma~\ref{comm:arrays} gives public probability at least $1-C/N$.

\subsection{Service and gain}\label{comm:product}
For the rest of this appendix, fix a rounded reward matrix~$r$, and condition on a realization of the ID array satisfying properties~(i)--(iii) of Lemma~\ref{comm:arrays} and on the selected index~$J$. We use the reference execution of Algorithm~\ref{sq:alg:reference}, with probabilities over the independent coordinates of~$R_J$ and the private draws, independently of~$r$.

Let $t_s=\max\mathcal B_s$ and let $\mathscr F_s$ be generated by the fixed data, selected-row coordinates through~$t_s$, and private draws exposed by that boundary; $\mathscr F_0$ contains only the fixed data. All block sums stop after the block that triggers launch. At a boundary, a follower counts as uninformed after pending declarations are applied and before any launch correction.

Write $\chi_t(a)=\ind\{r_t(a)>0\}$, $S_D(a)=\sum_{t\in D}\chi_t(a)$, and $S_s(a)=S_{\mathcal B_s}(a)$. Recall $h_{\rm boot}=\lceil C_hL_0b_N\rceil$, $q_{\rm dec}=8L_0b_N$, and that $G_{i,s}$ counts positive stay observations over the entire block. An active uninformed follower's block is \emph{productive} if its held arm~$a$ satisfies $S_s(a)\ge h_{\rm boot}$.

\begin{lemma}[One-block bounds]\label{common:clocks}
Assume that the ID array satisfies properties~(i)--(iii) of Lemma~\ref{comm:arrays}. Fix a block~$s$ and a follower~$i\in[d]$ that is uninformed before the block. Condition on~$\mathscr F_{s-1}$, on follower~$i$ being active, and on its held arm being~$a$. If $S_s(a)\ge h_{\rm boot}$, follower~$i$ learns $J$ by the block boundary with probability at least $1/(16K)$. The expected gain satisfies
\begin{equation}\label{comm:boot-gain}
 \E[G_{i,s}\mid\mathscr F_{s-1},\ i\text{ active},\ A^{\rm hold}_{i,s}=a]
 \ge S_s(a)/512\,.
\end{equation}
\end{lemma}
\begin{proof}
\emph{Isolation.} Let $\mathcal I$ be the event that every other uninformed follower is inactive in block~$s$. The fresh activity coins give
\begin{equation*}
 \Pr(\mathcal I\mid\mathscr F_{s-1},\ i\text{ active},\ A^{\rm hold}_{i,s}=a)
 \ge(1-1/m)^{m-2}\ge1/4\,.
\end{equation*}
Write $A_{k,t}$ for player~$k$'s arm in the reference execution, $\tau_t$ for the slot in~$R_J$, and $Y_{k,t}$ for player~$k$'s rounded observation.

\emph{Discovery and multicast.} Let $\mathcal U_s(a)$ consist of the indices $u\in[A]$ such that, for every interval $D\subseteq\mathcal B_s$ and $e\in\{0,1\}$,
\begin{equation*}
 \sum_{t\in D}\chi_t(a)\ind\{I_u(t)=e\}
 \ge S_D(a)/3-8b_N\,.
\end{equation*}
Lemma~\ref{comm:arrays}(iii) gives $|\mathcal U_s(a)|\ge3A/4$. Condition further on~$\mathcal I$, on $z=a$, and on $\iota=u\in\mathcal U_s(a)$. Let $b_s=\min\mathcal B_s$ and recall $t_s=\max\mathcal B_s$. Let $\mathscr H_t$ be generated by~$\mathscr F_{s-1}$, the block-initial private choices, selected-row coordinates through round~$t$, and private draws used through that round. The time-$t$ row coordinates and fresh private draws retain their independent laws conditional on~$\mathscr H_{t-1}$. The other players' DATA arms are independent uniform draws, so for every $t\in\mathcal B_s$ and arm~$c$,
\begin{equation*}
 \Pr(A_{k,t}\ne c\text{ for every }k\notin\{0,i\}\mid
       \mathscr H_{t-1},\tau_t=\mathrm{DATA})
 \ge(1-1/K)^{m-2}\ge1/4\,.
\end{equation*}

Let $\theta$ be the first round of this block in which the coordinator uses MULTICAST, with $\theta=t_s+1$ if it uses DISCOVER throughout the block. Put
\begin{equation*}
 \mathcal B_s^{\rm dis}=\{t\in\mathcal B_s:t<\theta\},\qquad
 \mathcal B_s^{\rm mul}=\mathcal B_s\setminus\mathcal B_s^{\rm dis}\,.
\end{equation*}
The mode at round~$t$ is determined before its current draws, so $\{t\in\mathcal B_s^{\rm dis}\},\{t\in\mathcal B_s^{\rm mul}\}\in\mathscr H_{t-1}$. Define
\begin{align*}
 X_t^{\rm D}
 &=\ind\{t\in\mathcal B_s^{\rm dis},\ \tau_t=\mathrm{DATA},\ Y_{0,t}>0\}\,,\\
 X_t^{\rm M}
 &=\ind\{t\in\mathcal B_s^{\rm mul},\ \tau_t=\mathrm{DATA},\ I_u(t)=0,\ Y_{i,t}>0\}\,.
\end{align*}

A coordinator positive during DISCOVER requires $I_u(t)=1$, because the follower plays~$a$ when $I_u(t)=0$. Thus the coordinator's list retains~$u$. By Lemma~\ref{comm:arrays}(ii), after $6b_N$ coordinator positives the list has fewer than~$L_0$ indices, so the coordinator switches to MULTICAST by the next round. Hence $\sum_{t\in\mathcal B_s}X_t^{\rm D}\le6b_N$. The DATA probability~$1/2$ and the avoidance bound give
\begin{equation*}
 \E[X_t^{\rm D}\mid\mathscr H_{t-1}]
 \ge\frac{\chi_t(a)}8\ind\{t\in\mathcal B_s^{\rm dis},\ I_u(t)=1\}\,.
\end{equation*}

At the transition, the retained list~$\Lambda$ satisfies $u\in\Lambda$ and $|\Lambda|\le L_0$, and stays fixed. For $t\in\mathcal B_s^{\rm mul}$, the current forward bits are independent of~$\mathscr H_{t-1}$, so
\begin{equation}\label{init:eq:release}
 \Pr(A_{0,t}\ne a\mid\mathscr H_{t-1},\tau_t=\mathrm{DATA})
 =(1-1/L_0)^{|\Lambda|}\ge(1-1/L_0)^{L_0}\ge1/4\,.
\end{equation}
The DATA probability and avoidance by the other players give
\begin{equation*}
 \E[X_t^{\rm M}\mid\mathscr H_{t-1}]
 \ge\frac{\chi_t(a)}{64}\ind\{t\in\mathcal B_s^{\rm mul},\ I_u(t)=0\}\,.
\end{equation*}

\emph{One clock.} The two modes are disjoint, so $X_t:=X_t^{\rm D}/8+X_t^{\rm M}$ lies in~$[0,1]$. Its conditional mean $\mu_t:=\E[X_t\mid\mathscr H_{t-1}]$ satisfies
\begin{equation*}
 \mu_t\ge\frac{\chi_t(a)}{64}
 \bigl(\ind\{t\in\mathcal B_s^{\rm dis}\}I_u(t)
       +\ind\{t\in\mathcal B_s^{\rm mul}\}(1-I_u(t))\bigr)\,.
\end{equation*}
Apply the defining balance of~$\mathcal U_s(a)$ to the discovery prefix and multicast suffix, with zero sums for an empty phase. Since their support counts sum to~$S_s(a)$,
\begin{equation*}
 \sum_{t\in\mathcal B_s}\mu_t
 \ge\frac1{64}\left(\frac{S_s(a)}3-16b_N\right)
 =\frac{S_s(a)}{192}-\frac{b_N}4\,.
\end{equation*}
Set $\ell=4b_N$. Apply Lemma~\ref{init:lem:interval-bound} from the deterministic block start~$b_s$ and use its block-end consequence. On an event~$\mathcal E_{\rm clk}$ with
\begin{equation*}
 \Pr(\mathcal E_{\rm clk}^c\mid\mathscr H_{b_s-1})
 \le e^{-\ell}\le CN^{-2}<1/8\,,
\end{equation*}
we have
\begin{equation}\label{init:eq:id-clock}
 \sum_{t\in\mathcal B_s}X_t^{\rm M}
 =\sum_{t\in\mathcal B_s}X_t-\frac18\sum_{t\in\mathcal B_s}X_t^{\rm D}
 \ge\frac{S_s(a)}{384}-\frac{7b_N}8-\ell\,.
\end{equation}
The prescribed~$C_h$ ensures $h_{\rm boot}\ge384(q_{\rm dec}+7b_N/8+\ell)$. Thus $S_s(a)\ge h_{\rm boot}$ implies $\sum_{t\in\mathcal B_s}X_t^{\rm M}\ge q_{\rm dec}$.

On a DISCOVER DATA round with $I_u(t)=0$, the coordinator and follower both play~$a$, so $Y_{i,t}=0$. On a positive MULTICAST stay round,
\begin{equation*}
 \{X_t^{\rm M}=1\}
 \subseteq\{\tau_t=\mathrm{DATA},\ A_{0,t}\ne a\}
 \subseteq\{\varphi_J(u,t)=1\}\,.
\end{equation*}
Thus the eligible positive set from~\eqref{comm:positive-set} satisfies
\begin{equation*}
 |P_i(J;t_s)|=\sum_{t\in\mathcal B_s}X_t^{\rm M}\ge q_{\rm dec},\qquad
 \prod_{t\in P_i(J;t_s)}\varphi_J(u,t)=1\,.
\end{equation*}
The true-$J$ test therefore accepts by the block boundary. Averaging the block choices gives
\begin{align*}
 &\Pr(i\text{ learns }J\text{ by the block boundary}\mid
       \mathscr F_{s-1},\ i\text{ active},\ A^{\rm hold}_{i,s}=a)\\
\ge&\frac14\cdot\frac1K\cdot\frac{|\mathcal U_s(a)|}{A}
       (1-e^{-\ell})
 \ge\frac1{16K}\,.
\end{align*}

\emph{Gain.} Condition on~$\mathscr F_{s-1}$, on follower~$i$ being active, and on $A^{\rm hold}_{i,s}=a$. Further condition on~$\mathcal I$ and $z\ne a$, the latter having probability at least~$1/2$. For every fixed~$\iota$ and within-block history, the coordinator on DATA plays~$z$ or draws uniformly outside~$\{z\}$. It avoids~$a$ with probability at least $(K-2)/(K-1)\ge1/2$, since $K>m\ge2$. DATA probability and other-player avoidance give
\begin{equation*}
 \E\bigl[\ind\{I_\iota(t)=0,\ Y_{i,t}>0,\ \tau_t=\mathrm{DATA}\}\mid
 \mathscr F_{s-1},\ i\text{ active},\ A^{\rm hold}_{i,s}=a,\mathcal I,z\ne a,\iota\bigr]
 \ge\frac{\chi_t(a)\ind\{I_\iota(t)=0\}}{16}\,.
\end{equation*}
Conditional on~$\mathscr F_{s-1}$, follower~$i$ being active, $A^{\rm hold}_{i,s}=a$, $\mathcal I$, and $z\ne a$, the index~$\iota$ is uniform on~$[A]$. Since $G_{i,s}$ includes all DATA stay positives, averaging the preceding bound over~$\iota$ and summing over~$t$, then using Lemma~\ref{comm:arrays}(i), gives
\begin{align*}
 &\E[G_{i,s}\mid\mathscr F_{s-1},\ i\text{ active},\ A^{\rm hold}_{i,s}=a,\mathcal I,z\ne a]\\
 =&\frac1A\sum_{u=1}^A
   \E\!\left[G_{i,s}\,\middle|\,
       {\mathscr F_{s-1},\ i\text{ active},\ A^{\rm hold}_{i,s}=a,
                 \mathcal I,z\ne a,\iota=u}\right]\\
 \ge&\frac1{16}\sum_{t\in\mathcal B_s}\chi_t(a)
       \left(\frac1A\sum_{u=1}^A\ind\{I_u(t)=0\}\right)\\
 \ge&\frac1{48}\sum_{t\in\mathcal B_s}\chi_t(a)
   =\frac{S_s(a)}{48}\,.
\end{align*}
Since $G_{i,s}\ge0$, including the probabilities of~$\mathcal I$ and $z\ne a$ yields
\begin{align*}
 \E[G_{i,s}\mid\mathscr F_{s-1},\ i\text{ active},\ A^{\rm hold}_{i,s}=a]\ge\frac14\cdot\frac12\cdot\frac{S_s(a)}{48}
   \ge\frac{S_s(a)}{512}\,.
\end{align*}
\end{proof}

\subsection{Support while information is incomplete}

\begin{lemma}[Uninformed-prefix bounds]\label{comm:potential}
With probability at least $1-N^{-2}$, simultaneously at every prelaunch block boundary $r\in[n_{\rm boot}]$ where some follower remains uninformed,
\begin{equation}\label{comm:unfinished-support}
 D_r:=\max_{S\in\cS}\sum_{s\le r}\sum_{a\in S}S_s(a)
 \le C\bigl[m^2KB_{\rm boot}b_N+mh_{\rm boot}r+m^2h_{\rm boot}b_N\bigr]\le Q\,,
\end{equation}
where $C$ is some universal constant.
The bound also holds for a shortened final block.
\end{lemma}
\begin{proof}
\emph{Weight comparison.} Use the block-boundary filtration~$\mathscr F_s$ of Subsection~\ref{comm:product} and fix a follower~$i$. Let $V_{i,s}$ indicate that it is uninformed and the coordinator has not launched before block~$s$; define $V_{i,n_{\rm boot}+1}$ by the same conditions at the final boundary. These indicators are $\mathscr F_{s-1}$-measurable. Stop updates after the first block in which~$i$ declares~$J$ or the coordinator launches; set later gains to zero and retain the weights. Write $p_{i,s}$ for the held-arm distribution and $\mathsf A_{i,s}$ for activity. For analysis, draw held arms also on inactive blocks and include them in~$\mathscr F_s$, without changing arm choices.

Set $\eta_0=1/(4KB_{\rm boot})$ and $\varepsilon_w=2^{-4b_N}$. Define the stopped importance-weighted gain
\begin{equation*}
 \widehat g_{i,s}(a)=
 \frac{V_{i,s}\mathsf A_{i,s}\ind\{A^{\rm hold}_{i,s}=a\}G_{i,s}}{p_{i,s}(a)}\,.
\end{equation*}
Since $p_{i,s}(a)\ge1/(2K)$, we have $0\le\widehat g_{i,s}(a)\le2KB_{\rm boot}$ and $\eta_0\widehat g_{i,s}(a)\le1/2$. Let $w_{i,s}(a)$ be the weight before block~$s$ and $W_{i,s}=\sum_a w_{i,s}(a)$. The update~\eqref{sq:eq:boot-weight} reads
\begin{equation*}
 w_{i,s+1}(a)
 =\varepsilon_w\left\lfloor
       \frac{w_{i,s}(a)(1+\eta_0\widehat g_{i,s}(a))}{\varepsilon_w}
   \right\rfloor\,.
\end{equation*}
For unselected arms and after stopping, $\widehat g_{i,s}(a)=0$ and $w_{i,s}(a)/\varepsilon_w$ is an integer, so the identity still holds. Nonnegative gains preserve $w_{i,s+1}(a)\ge w_{i,s}(a)\ge1$. Write $x=\eta_0\widehat g_{i,s}(a)\in[0,1/2]$. If $x=0$, the grid update is exact. If $x>0$, the integer gain satisfies $G_{i,s}\ge1$ and $p_{i,s}(a)\le1$, so
\begin{equation*}
 x\ge\eta_0=\frac1{4KB_{\rm boot}},
 \qquad \varepsilon_w\le\frac1{16KB_{\rm boot}}\le\frac x4\,.
\end{equation*}
The gain-weight rounding comparison uses $KB_{\rm boot}\le KT\le N$, $2^{b_N}\ge N$, and $N\ge24$. Since $w_{i,s}(a)\ge1$, the bound $\lfloor y\rfloor\ge y-1$ gives
\begin{equation*}
 \frac{w_{i,s+1}(a)}{w_{i,s}(a)}
 \ge1+x-\frac{\varepsilon_w}{w_{i,s}(a)}
 \ge1+\frac{3x}4\,.
\end{equation*}
Using $\log(1+y)\ge y/(1+y)$ for $y\ge0$ and $x\le1/2$, we obtain in both cases
\begin{equation*}
 \log\frac{w_{i,s+1}(a)}{w_{i,s}(a)}
 \ge\frac{3x/4}{1+x}
 \ge\frac x2=\frac{\eta_0}{2}\widehat g_{i,s}(a)\,.
\end{equation*}
When an active block updates its held arm~$a$, all other weights are unchanged. Downward rounding and $w_{i,s}(a)/W_{i,s}\le2p_{i,s}(a)$ give
\begin{equation*}
 W_{i,s+1}
 \le W_{i,s}+\frac{\eta_0w_{i,s}(a)G_{i,s}}{p_{i,s}(a)}
 \le W_{i,s}(1+2\eta_0G_{i,s})\,.
\end{equation*}
On inactive or stopped blocks, $W_{i,s+1}=W_{i,s}$ gives the same bound. Summing the logarithmic bounds, using $\log(1+x)\le x$, $w_{i,1}(a)=1$, $W_{i,1}=K$, and $w_{i,r+1}(a)\le W_{i,r+1}$, yields
\begin{equation*}
 \frac{\eta_0}{2}\sum_{s\le r}\widehat g_{i,s}(a)
 \le\log w_{i,r+1}(a)
 \le\log K+2\eta_0\sum_{s\le r}G_{i,s}\,.
\end{equation*}
Substituting $\eta_0=1/(4KB_{\rm boot})$ gives
\begin{equation}\label{init:eq:weight-potential}
 \sum_{s\le r}\widehat g_{i,s}(a)
 \le4\sum_{s\le r}G_{i,s}+8KB_{\rm boot}\log K\,.
\end{equation}
The last term is at most~$CKB_{\rm boot}b_N$.

\emph{Support bound.} For $\mu_{i,s}(a)=\E[\widehat g_{i,s}(a)\mid\mathscr F_{s-1}]$, the activity probability~$1/m$ and~\eqref{comm:boot-gain} give
\begin{equation*}
 \mu_{i,s}(a)\ge\frac{V_{i,s}S_s(a)}{512m}\,.
\end{equation*}
Put $U_*=2KB_{\rm boot}$ and $\ell=4b_N$. Apply the lower bound in Lemma~\ref{init:lem:interval-bound} with $X_s=\widehat g_{i,s}(a)/U_*$, filtration~$(\mathscr F_s)$, horizon~$n_{\rm boot}$, and $\alpha=1$. Its conditional mean is $\mu_{i,s}(a)/U_*$, so with probability at least $1-e^{-\ell}$, simultaneously for all~$r$,
\begin{equation*}
 \sum_{s\le r}\mu_{i,s}(a)
 \le2\sum_{s\le r}\widehat g_{i,s}(a)+2U_*\ell\,.
\end{equation*}
For every $r\in[n_{\rm boot}]$ such that $V_{i,s}=1$ for all $s\le r$, substitution of~\eqref{init:eq:weight-potential} and the mean bound gives
\begin{equation}\label{comm:potential-bound}
 \sum_{s\le r}S_s(a)\le Cm\left(\sum_{s\le r}G_{i,s}+KB_{\rm boot}b_N\right)\,.
\end{equation}

\emph{Productive and active block counts.} Let
\begin{equation*}
 \mathsf P_{i,s}=V_{i,s}\mathsf A_{i,s}
       \ind\{S_s(A^{\rm hold}_{i,s})\ge h_{\rm boot}\},\qquad p=1/(16K)\,.
\end{equation*}

Put $\mathscr K_{i,s}=\mathscr F_{s-1}\vee\sigma(\mathsf A_{i,s},A^{\rm hold}_{i,s})$. By Lemma~\ref{common:clocks}, on $\{\mathsf P_{i,s}=1\}$,
\begin{equation*}
 \Pr(V_{i,s+1}=0\mid\mathscr K_{i,s})
 \ge\Pr(i\text{ learns }J\text{ in block }s\mid\mathscr K_{i,s})
 \ge p\,.
\end{equation*}
For analysis, continue the block-initial activity and held-arm draws after stopping according to their prescribed conditional laws, and include them in~$\mathscr F_s$. Put $n=n_{\rm boot}$ and define the filtration
\begin{equation*}
 \mathscr H_{s-1}:=\mathscr K_{i,s}\quad(1\le s\le n),
 \qquad \mathscr H_n:=\mathscr F_n\,.
\end{equation*}
It is increasing because $\mathscr K_{i,s}\subseteq\mathscr F_s\subseteq\mathscr K_{i,s+1}$ for $s<n$, and $\mathscr K_{i,n}\subseteq\mathscr F_n$. The variable $X_s:=\mathsf P_{i,s}(1-V_{i,s+1})$ is $\mathscr F_s$-measurable, hence $\mathscr H_s$-measurable, and the preceding success bound gives
\begin{equation*}
 \E[X_s\mid\mathscr H_{s-1}]
 =\mathsf P_{i,s}\Pr(V_{i,s+1}=0\mid\mathscr K_{i,s})
 \ge p\,\mathsf P_{i,s}\,.
\end{equation*}
The indicators~$V_{i,s}$ decrease from one to zero at most once, and $\mathsf P_{i,s}>0$ requires $V_{i,s}=1$. Thus $\sum_sX_s\le1$. Apply the lower bound in Lemma~\ref{init:lem:interval-bound} to~$X_s\in[0,1]$, with filtration~$(\mathscr H_s)$ and $\alpha=1$. With probability at least $1-e^{-\ell}$, simultaneously for every~$r\le n$,
\begin{equation*}
 p\sum_{s\le r}\mathsf P_{i,s}
 \le\sum_{s\le r}\E[X_s\mid\mathscr H_{s-1}]
 \le2\sum_{s\le r}X_s+2\ell
 \le2(1+\ell)\,.
\end{equation*}
The total number of productive blocks, including the completing block, is therefore at most $2(1+\ell)/p\le CKb_N$ on this event.

 Apply the upper bound in Lemma~\ref{init:lem:interval-bound} with $X_s=\mathsf A_{i,s}$, filtration~$(\mathscr F_s)$, horizon~$n_{\rm boot}$, and $\alpha=1$. Since $\E[\mathsf A_{i,s}\mid\mathscr F_{s-1}]=1/m$, with probability at least $1-e^{-\ell}$,
\begin{equation*}
 \sum_{s\le r}\mathsf A_{i,s}\le\frac{2r}{m}+\ell
 \qquad\text{for every }r\in[n_{\rm boot}]\,.
\end{equation*}
Productive blocks contribute at most~$B_{\rm boot}$ each, and other active blocks less than~$h_{\rm boot}$. Hence, for every $r$ with $V_{i,s}=1$ for all $s\le r$,
\begin{equation}\label{comm:gain-budget}
 \sum_{s\le r}G_{i,s}\le CKB_{\rm boot}b_N+h_{\rm boot}(2r/m+Cb_N)\,.
\end{equation}

A union bound over follower-arm pairs and both count estimates gives total failure probability at most $4mKe^{-4b_N}\le N^{-2}$.

\emph{Comparator support.} At a boundary~$r$ with an uninformed follower~$i$, before any launch correction, sum~\eqref{comm:potential-bound} over the $m$ arms of $S\in\cS$ and apply~\eqref{comm:gain-budget}:
\begin{equation*}
 \sum_{s\le r}\sum_{a\in S}S_s(a)
 \le C\bigl[m^2KB_{\rm boot}b_N+mh_{\rm boot}r+m^2h_{\rm boot}b_N\bigr]\,.
\end{equation*}
The estimates hold for every~$i,a,r$, so we may choose~$i$ after the execution and maximize over~$S$. Since $r\le n_{\rm boot}$, the prescribed~$C_*$ in~\eqref{sq:eq:launch-budget} makes the bound at most~$Q$. Only $|\mathcal B_s|\le B_{\rm boot}$ was used, so the final shortened block is included.
\end{proof}

\subsection{The coordinator's counts and launch}

Recall the coordinator's cumulative DOWN-positive counts~$C_r(a)$ and statistic $\widehat D_r=K\max_{S\in\cS}\sum_{a\in S}C_r(a)$ from~\eqref{sq:eq:root-monitor}.

\begin{lemma}[Coordinator monitor]\label{comm:monitor}
Let $\ell=4b_N$, let $\widehat D_r$ be the coordinator statistic~\eqref{sq:eq:root-monitor}, and let $D_r$ be the comparator support~\eqref{comm:unfinished-support}. With probability at least $1-N^{-2}$, simultaneously at every prelaunch block boundary~$r$,
\begin{equation}\label{comm:monitor-bounds}
 D_r\le128(\widehat D_r+mKB_{\rm boot}\ell),\qquad
 \widehat D_r\le2D_r+mKB_{\rm boot}\ell\,.
\end{equation}
A block increases $\widehat D_r$ by at most~$KB_{\rm boot}$.
\end{lemma}
\begin{proof}
\emph{Conditional means.} Let $Z_s(a)$ count coordinator DOWN positives on~$a$ in block~$s$. Conditional on~$\mathscr F_{s-1}$ before launch, all uninformed followers are inactive with probability at least $(1-1/m)^{m-1}\ge1/4$. Their uniform arms then avoid any coordinator arm with probability at least $(1-1/K)^{m-1}\ge1/4$. Informed followers remain in state~$0$ and use distinct DOWN tuple arms. Since DOWN has probability~$1/4$ and $R_{0,t}$ is uniform,
\begin{equation}\label{init:eq:monitor-means}
 \frac{S_s(a)}{64K}\le\E[Z_s(a)\mid\mathscr F_{s-1}]
 \le\frac{S_s(a)}K,\qquad 0\le Z_s(a)\le B_{\rm boot}\,.
\end{equation}
The upper bound uses only $R_{0,t}=a$ and $\chi_t(a)=1$; the lower bound averages the activity coins. Declarations leave arm rules unchanged within the block.

\emph{Stopped comparisons.} Let $L_s=\ind\{\text{coordinator has not launched before block }s\}$, which is $\mathscr F_{s-1}$-measurable, and put
\begin{equation*}
 X_s=\frac{L_sZ_s(a)}{B_{\rm boot}},\qquad
 v_s=\frac{L_sS_s(a)}{KB_{\rm boot}},\qquad
 \mu_s=\E[X_s\mid\mathscr F_{s-1}]\,.
\end{equation*}
Then $0\le X_s\le1$, and~\eqref{init:eq:monitor-means} gives $v_s/64\le\mu_s\le v_s$. Apply both bounds in Lemma~\ref{init:lem:interval-bound} with this~$X_s$, filtration~$(\mathscr F_s)$, horizon~$n_{\rm boot}$, and $\alpha=1$. A union over arms and the two bounds gives probability at least $1-2Ke^{-\ell}\ge1-N^{-2}$ that, simultaneously for all arms and block indices~$r$,
\begin{equation*}
 \frac1{128}\sum_{s\le r}v_s-\ell
 \le\sum_{s\le r}X_s
 \le2\sum_{s\le r}v_s+\ell\,.
\end{equation*}
Rescaling by~$KB_{\rm boot}$ yields, through the launching block,
\begin{equation}\label{init:eq:arm-monitor}
 \sum_{s\le r}S_s(a)\le128\bigl(KC_r(a)+KB_{\rm boot}\ell\bigr),\qquad
 KC_r(a)\le2\sum_{s\le r}S_s(a)+KB_{\rm boot}\ell\,.
\end{equation}
Here $L_s=1$ through the launching block. Summing over $m$ arms and maximizing proves~\eqref{comm:monitor-bounds}.

\emph{Overshoot.} Since $C_r(a)-C_{r-1}(a)=Z_r(a)\ge0$, and a largest-$m$ sum increases by at most the total coordinate increase,
\begin{equation*}
 0\le \widehat D_r-\widehat D_{r-1}\le K\sum_a Z_r(a)
 \le K|\mathcal B_r|\le KB_{\rm boot}\,.
\end{equation*}
The coordinator has at most one positive per round. The same bounds cover a shortened final block.
\end{proof}

\subsection{Cost scale and the transition at launch}\label{comm:handover}

Recall $Q$ and $Q_{\rm boot}$ from~\eqref{sq:eq:launch-budget} and~\eqref{sq:eq:bootQ}:
\begin{equation*}
 \begin{aligned}
 Q&=\left\lceil C_*\bigl[m^2KB_{\rm boot}b_N+mh_{\rm boot}n_{\rm boot}
       +m^2h_{\rm boot}b_N+mB_{\rm boot}\bigr]\right\rceil,\\
 Q_{\rm boot}&=m^{3/2}\sqrt{KT}\,b_N^2\,.
 \end{aligned}
\end{equation*}

\begin{lemma}[Cost scale]\label{init:cor:cost}
For the parameters above, assume~\eqref{sq:eq:bootbranch}. Then
\begin{equation}\label{init:eq:cost-scale}
 Q\le C Q_{\rm boot}\,,
\end{equation}
where $C$ is some universal constant.
\end{lemma}
\begin{proof}
Since $2b_N\le L_0<4b_N$, $h_{\rm boot}\le Cb_N^2$, $\sqrt{T/(mK)}\le B_{\rm boot}\le\sqrt{T/(mK)}+1$, and $n_{\rm boot}\le T/B_{\rm boot}+1$,
\begin{align*}
 m^2KB_{\rm boot}b_N
 &\le m^{3/2}\sqrt{KT}\,b_N+m^2Kb_N,\\
 mh_{\rm boot}n_{\rm boot}
 &\le C\bigl[m^{3/2}\sqrt{KT}\,b_N^2+mb_N^2\bigr],\\
 m^2h_{\rm boot}b_N+mB_{\rm boot}
 &\le C\bigl[m^2b_N^3+\sqrt{mT/K}+m\bigr]\,.
\end{align*}
Including the ceiling gives
\begin{equation*}
 Q\le C\bigl[Q_{\rm boot}+m^2Kb_N^3\bigr]\,.
\end{equation*}
Since $T>C_bmKb_N^2$,
\begin{equation*}
 m^2Kb_N^3\le C_b^{-1/2}Q_{\rm boot}\,.\qedhere
\end{equation*}
\end{proof}

\paragraph{Prelaunch assignment tests.}
Let $E_{\rm down}$ mean that some DOWN output~$M$ satisfies $\operatorname{Acc}_{i,t}^{\rm DOWN}(M)=1$ before launch. The reference forces that decision to zero. Recall from~\eqref{sq:eq:decode-quota} that
\begin{equation*}
 B=5\lceil\log_2(TD+1)\rceil,\qquad
 Q_{\rm dec}=B+\left\lceil\log_2\bigl(4(m-1)TN^2\bigr)\right\rceil\,.
\end{equation*}
Before launch, all DOWN arm choices are independent of the current affine DOWN map. Its independence of the slot, tuple, forward bits, and private draws permits exposing receiver positivity first. Conditional on this information and preceding equations, a fixed word satisfies the fresh equation with probability~$1/2$.

The stopped-equation argument of~\eqref{protocol:eq:batch-tail} bounds the probability of~$Q_{\rm dec}$ positive tests before launch all retaining a fixed word by~$2^{-Q_{\rm dec}}$, conditional on each batch-start history. A union over at most $(m-1)T$ starts and $2^B$ words gives
\begin{equation*}
 \Pr(E_{\rm down})\le(m-1)T2^{B-Q_{\rm dec}}\le\frac1{4N^2}\,.
\end{equation*}
The coordinator's decoder is disabled before launch.

\paragraph{Buffers at launch.}
On declaring~$J$, a follower initializes
\begin{equation*}
 (\sigma_i,S_i,\delta_i,\mathsf b_i,\mathrm{state}_i)=(0,0,0,0,0)\,.
\end{equation*}
Rejecting every prelaunch assignment preserves these registers while DATA and DOWN update only their buffers.

The DATA observation list and DOWN equation batch may be nonempty at launch. The first accepted assignment has $\sigma\ge\tau_{\rm L}$, so~\eqref{protocol:eq:sparse-sum} excludes observations before~$\sigma$. Lemma~\ref{sq:lem:decoder} allows a DOWN batch to straddle launch.

\subsection{Proof of Proposition~\ref{sq:prop:startup}}\label{comm:startup-proof}

\begin{proof}[Proof of Proposition~\ref{sq:prop:startup}]
\emph{Reference execution.} Fix an identification array $\mathsf I^{\circ}$ in the set of probability at least $1-3/N$ from Lemma~\ref{comm:arrays}, and fix $r,j$ as in the statement. Apply Algorithm~\ref{sq:alg:reference} and couple it to Algorithm~\ref{sq:alg:activation} using the same complete rows and private randomness. Before conditioning on $J=j$, the selected row index $J$ is uniform on $[M]$.

Let $E_{\rm init}$ mean that a missing follower is corrected at launch or an ordinary DOWN acceptance is suppressed before launch; let $E_{\rm fwd}$ mean that some $v\ne j$ would pass the bootstrap prefix test~\eqref{comm:decoder} on the true-only reference observations. These virtual tests do not affect the reference execution. Only $E_{\rm init}$ concerns selected-row corrections.

\emph{Support and launch.} Set $\ell=4b_N$ and use the comparator support~$D_r$ from~\eqref{comm:unfinished-support}.
Lemmas~\ref{comm:potential} and~\ref{comm:monitor} give an event $\mathcal E_{\rm sup}$ with $\Pr(\mathcal E_{\rm sup}^c)\le 2N^{-2}$ on which, simultaneously at every prelaunch block boundary,
\begin{align*}
 &\text{some follower is uninformed after block }r\\
 \Longrightarrow&D_r\le Q,\quad
 D_r\le128(\widehat D_r+mKB_{\rm boot}\ell),\quad
 \widehat D_r\le2D_r+mKB_{\rm boot}\ell\,.
\end{align*}
The choice of $C_*$ gives $mKB_{\rm boot}\ell\le Q$, so an uninformed follower forces $\widehat D_r\le3Q$ and cannot remain at launch, which requires $\widehat D_r>4Q$. Each block increases $\widehat D_r$ by at most $KB_{\rm boot}$, giving first-crossing value at most $4Q+KB_{\rm boot}$. The monitor bounds also cover the final shortened block when no launch occurs. Since $r_t(a)\le\ind\{r_t(a)>0\}$, on $\mathcal E_{\rm sup}$,
\begin{equation*}
 \max_{S\in\cS}\sum_{t<\tau_{\rm L}}r_t(S)
 \le128(4Q+KB_{\rm boot}+mKB_{\rm boot}\ell)
 \le C Q\le C Q_{\rm boot}\,.
\end{equation*}
The last inequality is Lemma~\ref{init:cor:cost}.

\emph{Coupling events.} The ordinary DOWN suppression bound in Appendix~\ref{comm:handover} and the launch implication above give
\begin{equation*}
 \Pr(E_{\rm init})\le\Pr(\mathcal E_{\rm sup}^c)
                   +(m-1)T2^{B-Q_{\rm dec}}\le 3N^{-2}\,.
\end{equation*}

Condition on $R_j$, the private bits, and all candidates' slot coordinates. Every reference eligible positive set is independent of the false forward bits. With $p=1-1/L_0$, a union over at most $mMT$ candidate-prefix tests gives
\begin{equation*}
 \Pr(E_{\rm fwd})\le mMTp^{q_{\rm dec}}
 \le mMT e^{-8b_N}\le N^{-2}\,.
\end{equation*}
On $E_{\rm init}^c\cap E_{\rm fwd}^c$, every true-row bootstrap pass is the unique passing candidate, and no reference correction is made. Induction over rounds and block boundaries gives identical actions, observations, gain weights, effective declarations, activation times, and ordinary registers. A union bound gives
\begin{equation*}
 \Pr(E_{\rm init}\cup E_{\rm fwd})\le4N^{-2}\le\Delta_{\rm boot}\,,
\end{equation*}
proving~(i).

\emph{Reference state and dependence.} Suppressing prelaunch ordinary DOWN acceptances keeps the accepted start, score, last increment, and queue registers zero despite DATA buffers and partial DOWN batches (Appendix~\ref{comm:handover}). Algorithm~\ref{sq:alg:reference} puts missing followers in this true-$j$ unoccupied state at launch; the coordinator initializes zero score and queue vectors. This holds on every reference launch path.

For each $S\in\cS$, the deterministic bound $\sum_{t<\tau_{\rm L}}r_t(S)\le mT$ gives
\begin{equation*}
 \E\sum_{t<\tau_{\rm L}}r_t(S)
 \le C Q_{\rm boot}+mT\Pr(\mathcal E_{\rm sup}^c)
 \le C Q_{\rm boot}+CmT\Delta_{\rm boot}\,.
\end{equation*}
All complete-row reads of the reference, including the tests defining $E_{\rm init}$, use $R_j$. Thus $E_{\rm init}\in\sigma(R_j,\omega_{\rm priv})$ for fixed $\mathsf I^{\circ},r,j$, where $\omega_{\rm priv}$ denotes private bits. Let $\mathscr H_t$ be generated by $\mathsf I^{\circ},j$, the selected-row coordinates through round $t$, and private bits exposed by its end. The launch rule satisfies
\begin{equation*}
 \{\tau_{\rm L}\le t\}\in\mathscr H_{t-1},\qquad 1\le t\le T+1\,.
\end{equation*}
Future selected-row calendar slices are independent of preceding slices and private bits. Conditioning on each possible launch round therefore preserves the product law of unexposed selected-row coordinates, proving~(ii) and~(iii).
\end{proof}

\subsection{Explicit numerical choices}
\label{app:explicit-constants}
The following conservative choices remove all unspecified algorithmic constants:
\[
 C_h=4096,\qquad C_*=2^{20},\qquad C_1=2^{20},\qquad
 C_b=16,\qquad N_0=2^{28}.
\]
They are chosen for transparent inequalities, not practical efficiency.
Below write $b=b_N$, $B_0=B_{\rm boot}$, and $h_0=h_{\rm boot}$.
The model gives $N\ge24$ and $b\ge5$.

\paragraph{Bootstrap quotas and launch threshold.}
Since $q_{\rm dec}=8L_0b$ and $L_0\ge2$,
\[
 4096L_0b
 \ge384\bigl(8L_0b+7b/8+4b\bigr).
\]
Thus the one-block proof's required inequality for $h_0$ holds.
In the proof of Lemma~\ref{comm:potential}, keeping numerical constants gives
\[
 \sum_{s\le r} S_s(a)
 \le4096m\sum_{s\le r}G_{i,s}+16384mKB_0b.
\]
Indeed the lower-tail estimate gives
$\sum\mu_{i,s}(a)\le2\sum\widehat g_{i,s}(a)+16KB_0b$,
and~\eqref{init:eq:weight-potential} and $\log K\le b$ give this display.
The productive-block count is at most
$32K(1+4b)\le160Kb$, while the active-block count is at most $2r/m+4b$.
Consequently
\[
 \sum_{s\le r}G_{i,s}
 \le160KB_0b+h_0(2r/m+4b).
\]
Summing the preceding support estimate over the $m$ comparator arms gives
\[
 D_r\le671744m^2KB_0b+8192mh_0r+16384m^2h_0b\le Q.
\]
Also $mKB_0(4b)\le Q$. These establish every use of $C_*$.

\paragraph{Tail estimates for public arrays.}
On the construction branch, $B_0\ge4b$ and $T>16mKb^2\ge36\log N$.
Thus the column-balance failure probability is at most
$2T e^{-T/18}\le N^{-1}$. The all-one rectangle bound is at most
$2^{-2b^2}\le N^{-2}$. For the interval-balance bound,
$p_0\le2^{-5b}$ and
$5b\log2-\log(4e)\ge2b$ for $b\ge5$; its union bound is at most
$\exp[-(2-\log2)(B_0+b)]\le N^{-1}$.
For the incorrect-row context bound, $\mu=N^3e^{-8b}$ and
$\varepsilon_{\rm loc}\ge N^{-1}$ give
\[
 \log\frac{\varepsilon_{\rm loc}}{e\mu}
 \ge8b-4\log N-1\ge5b.
\]
The final inequality follows from $\log N\le b\log2$ and $b\ge5$.
The context union is valid because
\[
 \tfrac52(B_0+b)\ge2B_0\log2+7\log N,
 \qquad B_0\ge4b.
\]
Thus none of these public-array bounds requires an additional increase in $N_0$.
The within-execution unions are also covered:
$4mKe^{-4b}\le N^{-2}$, $2Ke^{-4b}\le N^{-2}$, and
$mMT e^{-8b}\le N^{-2}$, using $4mK\le N$, $M\le N^2$, and $T,m\le N$.

\paragraph{Learning and transfer coefficients.}
Since $TD+1\le D^2$ and $4(m-1)TN^2\le N^3$, the ordinary
message width and decoder quota satisfy
\[
 B\le10b,\qquad Q_{\rm dec}\le13b.
\]
Also $\log(m-1)+1\le b$. With $\kappa=1-e^{-1}>1/2$,
\eqref{sq:eq:session} therefore gives
\[
 C_{\rm tr}
 =(m-1)K\left[\frac{48Q_{\rm dec}+28\log(m-1)+20}{\kappa}+2\right]
 \le2048mKb.
\]
As $4m+128\le68m$ for $m\ge2$,
\[
 (4m+128)C_{\rm tr}\le139264m^2Kb.
\]
Furthermore $L_{\rm stab}\le3b$, $R\le128m^2K$, and the explicit
coefficient in Proposition~\ref{sq:prop:learning} obeys
\[
 16m^2+\frac{16}{11}R+32m+1024(K-m+1)L_{\rm stab}
 <808m^2Kb.
\]
Their sum is smaller than $C_1m^2Kb$ for $C_1=2^{20}$.

\paragraph{Feasible dyadic tuning.}
For $\mathcal J=\lceil\log_2(\binom Km+1)\rceil+m$,
\[
 m\le\mathcal J\le m(b+1)\le\tfrac65mb,
 \qquad L_{\rm stab}\le3b.
\]
Putting $W=2^{20}m^2Kb$ gives
\[
 \frac{(1024mKL_{\rm stab})^2\mathcal J}{W}
 \le\frac{54}{5}mKb^2<16mKb^2<T.
\]
Hence the unconstrained target $\sqrt{\mathcal J/(WT)}$ is admissible,
and the largest dyadic $\eta$ below it is larger than half this target.
If $N\ge2^{28}$, then $N\ge2^{21}b=2C_1b$.
To see this, put $k=\lfloor\log_2N\rfloor\ge28$; then
$b\le k+2\le2^{k-21}$, so $2^{21}b\le2^k\le N$.
Since $mKT\le N/2$,
\[
 \eta>\frac1{2\sqrt{C_1mKbT}}
 \ge\frac1{\sqrt{2C_1bN}}\ge\frac1N.
\]
As $\eta$ and $D$ are powers of two and $D\ge N$, this implies
$\eta D\in\mathbb Z$. The specified $N_0$ therefore supplies every
remaining numerical condition.

\clearpage
\section*{AI methodology}
The author used GPT-6-astra to assist in developing the public-table construction and designing the initialization protocol, influencing the constructions presented in Section~\ref{sq:sec:boot} and Appendix~\ref{comm:appendix}. Opus-5.5 and GPT-6.1-sol were used to assist with drafting and revising the manuscript's exposition. The author verified the correctness and originality of all content, including the mathematical arguments and references, and takes responsibility for the submission.%
\medskip
\bibliographystyle{alpha}
\bibliography{references}
\end{document}